\documentclass[11pt]{article}
\usepackage[margin=0.9in]{geometry}
\usepackage[T1]{fontenc}
\usepackage{lmodern}
\usepackage{amsmath,amssymb,amsthm,mathtools,booktabs}
\usepackage[expansion=false]{microtype}
\usepackage[numbers,square,sort&compress]{natbib}
\usepackage{graphicx}
\usepackage{caption}
\usepackage{tikz}
\usetikzlibrary{arrows.meta,positioning}
\usepackage[hidelinks]{hyperref}
\newtheorem{theorem}{Theorem}[section]
\newtheorem{lemma}[theorem]{Lemma}
\newtheorem{proposition}[theorem]{Proposition}
\newtheorem{corollary}[theorem]{Corollary}
\theoremstyle{definition}
\newtheorem{definition}[theorem]{Definition}
\newtheorem{assumption}[theorem]{Assumption}
\newtheorem{example}[theorem]{Example}
\newtheorem{algo}{Algorithm}
\theoremstyle{remark}

\newcommand{\E}{\mathbb E}
\newcommand{\OPT}{\operatorname{OPT}}
\newcommand{\ALG}{\operatorname{ALG}}
\newcommand{\LP}{\mathrm{LP}}
\newcommand{\Q}{\mathcal Q}
\newcommand{\Pb}{\bar P}

\newcommand{\Reg}{\mathrm{Reg}}
\newcommand{\Lbar}{\bar L}

\newcommand{\Work}{\mathsf{t}}
\newcommand{\Pflat}{P^{\mathrm{flat}}}
\newcommand{\Pspread}{P^{\mathrm{spread}}}

\begin{document}
\title{Planned Restlessness: Delayed Meta Bandits with Budgets}
\author{Sudipto Guha\thanks{Email:{\tt contactsudiptoguha@gmail.com}}}
\date{}
\maketitle

\begin{abstract}
We study episodic finite-horizon Bayesian multi-armed bandits, also known as meta-bandits. Each arm’s quality is drawn afresh in every episode from an unknown recurring prior. Outcomes arrive after known delays, and plays are capped per episode and in total. A play earns reward and provides information about both the current quality and the recurring prior. The learner must therefore explore, use, and generalize. Existing posterior-sampling methods establish learning guarantees for specified sampling and exploration rules.

We formulate a common planning optimization over learning trajectories. Its LP relaxation values an experiment through its current reward, resource consumption and continuation: the value of future decisions from the learning state it produces. Dual prices separate planning into arm-local programs and certify every admissible policy, including hierarchical Thompson sampling. Execution turns the relaxed plan into a feasible policy by preserving the promised learning transitions or accounting for their loss. The framework can represent risk-sensitive objectives and price purchased observations and buyout options.

Delayed observations create planned restlessness: an arm's information changes while it is not being played, through reports of its own previous plays arriving on a known calendar. Planning and execution must track when information arrives and which decisions can still use it. For stationary priors, independent arms, finite reward alphabets, bounded report lag and nonbinding budgets, posterior sampling over the recurring priors bounds the price of not knowing them by $O(\sqrt{N})$ over $N$ episodes, with the other parameters fixed. For a single undelayed episode with independent arms capped at $T/X$ plays each, we obtain an LP-relative approximation factor $1+(2\pi X)^{-1/2}+O(X^{-1})$, improving previous factors, with a matching lower bound for the leading term.
\end{abstract}

\section{Introduction}\label{sec:intro}
We develop a common planning framework for episodic bandits over unknown recurring priors, incorporating delayed observations, hard per-arm budgets and per-episode caps, and risk-sensitive objectives beyond expected reward.
In an episodic bandit or a meta-bandit problem
there are $n$ unknown arms and a learner uses these arms over $N$
episodes, each with at most $T$ decision steps.
Each episode draws fresh arm qualities from the recurring, initially unknown prior distribution.
A single episode is a finite-horizon multi-armed bandit
problem.
At each step, the learner plays an arm, earns reward and generates an observation. The observation provides information about the current episode’s arm quality and, across episodes, its recurring prior. Therefore
the typical multi-armed bandit tradeoff
of explore/use now expands to \emph{explore, use, and generalize}.
Meta-bandits
formalize this transfer/generalization by learning shared priors
\citep{kveton2021meta,basu2021,hong2022} using variants of Thompson sampling.
Other examples include
cold-start dynamic pricing \citep{bastani2022meta}, cloud management \citep{qiu2024flash}, database tuning
across hardware environments \citep{zhang2021restune}, program
synthesis \citep{ellis2021dreamcoder}, and LLM routing
\citep{zeng2026metarouter}. 
The analyzed MetaTS policy supplements Thompson sampling with a
forced play of every arm in each task
\citep[Section~4.1]{kveton2021meta}.
HierTS avoids forced exploration in its sequential Gaussian analysis,
but its concurrent guarantee again prescribes initial exploratory
actions \citep[Sections~6.1 and~6.3]{hong2022}.
These methods share information through a hierarchical posterior,
while their guarantees establish the learning produced by specified
sampling and exploration rules.
They do not explicitly optimize which observations are worth acquiring
through their contribution to later feasible decisions.
Even with independent arms, optimizing each episode separately can lose linearly in $N$ by
never collecting observations valuable to later episodes
(Examples~\ref{ex:price-test} and~\ref{ex:spread}).
We formulate a common planning objective that accounts for current
reward, resource consumption and the continuation of learning.
Its execution guarantee then accounts for the information and reward
actually realized under the hard constraints.

Budgets are natural in distributed applications of episodic bandits.
Budgeted learning is NP-hard even for a single episode with independent
arms \citep{GM07}.
An unexpectedly long episode, or a favorable draw of an arm’s episode quality, can make immediate use attractive enough to exhaust the arm’s remaining budget. Those plays then become unavailable for learning and reward in later episodes. Per-episode caps limit this drawdown.
Routing, coding, tuning and
pricing also motivate risk-sensitive objectives beyond expected reward.
Risk-sensitive objectives introduce a further structural difficulty.
Mean--variance criteria need not satisfy Bellman's principle of
optimality, and maximizing expected return subject to a variance bound
is NP-hard in finite-horizon Markov decision processes
\citep{mannor2011meanvariance}.
The objective depends on the return distribution across possible
histories. Evaluating risk conditional on a sampled episode quality
can discard variation that determines which action is preferable
(Example~\ref{ex:risk}).
Planning over outcome distributions retains this variation:
fixed outcome scores enter as linear coefficients, while risk
functionals act on the resulting law
(Subsection~\ref{sec:beyond}).

Delays couple decisions within and across episodes. Concurrent Bayesian
meta-bandits \citep{hong2022} analyze feedback withheld within a batch
of decisions, and distributed optimization also motivates delayed
feedback \citep{agarwal2011distributed}. Delayed observation in learning
dates back to \citet{anderson1964}.
Surprisingly, it was shown by \citet{eick1988} that
in the presence of delays an optimal policy can use a known arm and later return to an
unknown arm. Thus the optimal policy of a delayed finite horizon multi-armed bandit problem 
is not a familiar stopping time problem \citep{gittins1974dynamic}, creating a foundational divergence 
from standard multi-armed bandits.  A late-arriving observation changes an arm's information even when the arm is not being played.
This violates the standard rested-bandit assumption, and such arms are restless \citep{whittle1988}.
\citet{caro2010} study no-crossover delays.
General restless bandits have no known tractable algorithms, and their decision
versions are PSPACE-hard \citep{papadimitriou1999}. Here the delayed observations result from our own earlier plays
and arrive on a calendar fixed in advance. We call this \emph{planned restlessness}. 

Budgets constrain resource use on every realized path; delays constrain
the information available at each decision; risk criteria specify how
the reward distribution is evaluated.
Under this common model, we distinguish the losses due to unavailable information from the loss incurred by the algorithm.
Let $V_e$ be a policy's reward in episode $e$ and $V_A=\E\sum_eV_e$ its expected total.
Write $V_{\mathrm{state}}$ for the optimum when the recurring prior is
revealed and each episode's qualities are revealed at its start,
$V_{\mathrm{prior}}$ for the
optimum when only the recurring prior is revealed, and
$V_{\mathrm{learn}}$ for the Bayes-optimal learner with our information.
Then the overall loss decomposes as:
\begin{equation}
 V_{\mathrm{state}}-V_A
 =(V_{\mathrm{state}}-V_{\mathrm{prior}})
 +(V_{\mathrm{prior}}-V_{\mathrm{learn}})
 +(V_{\mathrm{learn}}-V_A).
 \label{eq:core}
\end{equation}
The gap $V_{\mathrm{state}}-V_{\mathrm{prior}}$ is the cost of uncertainty about each episode's fresh qualities. It can remain linear in $N$ even when the recurring law is known (Example~\ref{ex:paid-knowledge}). The gap $V_{\mathrm{prior}}-V_{\mathrm{learn}}$ is the price of learning that law, and the largest fee worth paying to have it revealed before learning begins (Section~\ref{sec:cost-content}). Under stationarity, independent arms, finite reward alphabets, bounded report lag and nonbinding budgets, this price is $O(\sqrt N)$, and hence vanishes per episode (Proposition~\ref{prop:ps-price}).
The gap $V_{\mathrm{learn}}-V_A$ is the algorithm's loss under the
same information and constraints. It measures how well the algorithm
uses its information; Section~\ref{sec:unroll} controls this loss for
the learning-LP learner.
Information is valuable only while the decisions it can improve remain available, and it is necessary that
all three benchmarks face the same budgets, caps and report calendar. 

We formulate a linear programming (LP) relaxation over learning trajectories.
We assume independence across arms in this paper which permits separate arm programs coupled by expected episode capacities.
The LP is written \emph{ex ante}, over all possible observation records.
Given the incoming belief, exchangeability and the fixed report calendar
determine how each stopping record leads to the arm's next learning
state: its belief about the recurring prior, remaining budget and
pending reports. Histories with the same sufficient state share these
transitions and can be represented by one node. The transitions are
specified before any particular outcome is observed; the policy
determines the probability mass flowing through them. Current reward,
resource use and continuation can therefore be accounted for linearly.
Pricing the shared episode capacities separates the arms into local
programs on their state DAGs
(Figure~\ref{fig:experiment-continuation}). Their values are computed
by backward induction, but their policies must be executed forward,
echoing Kierkegaard's distinction between understanding backward and
living forward.\footnote{S{\o}ren Kierkegaard, journal entry JJ:167
(1843), also numbered \emph{Pap.}~IV~A~164:
``Det er ganske sandt, hvad Philosophien siger, at Livet maa
forstaaes bagl{\ae}nds. Men derover glemmer man den anden
S{\ae}tning, at det maa leves forl{\ae}nds.''
The connection to Nietzsche's ``Modern Restlessness,''
\emph{Human, All Too Human}~I, \S~285 (1878),
does not escape our attention. The observer must act within the process it seeks to understand. Its actions determine both the evidence from which it learns and the possibilities from which it continues.}

Execution proceeds \emph{ex post}. The learner is an observer within
the system it controls: its belief about the arms is itself part of
the state being planned. Each play earns reward, spends budget and
creates a report on the known calendar. The policy determines which
evidence is acquired, when it becomes available, and which decisions
remain feasible when it arrives. The belief from which the next
episode is planned is therefore an output of the current policy.
Budgets, delays and risk criteria constrain or evaluate this same
process. The optimization trajectory is also a learning trajectory.

This coupling already drives the feedback restless bandits of
\citet{GMS10}: the recurrent state process has a known transition law
and evolves independently of the actions, while plays determine
observations and future beliefs.
Their approximation guarantees exploit further structure in the
belief transitions.
\citet{liu2010indexability} establish indexability
and derive closed-form indices for a related multichannel-access model.
Meta-bandits extend this information problem
across episodes: a play informs both the current quality and the
recurring prior, and its value depends on the future decisions that
its report can improve. Here \emph{planned restlessness} supplies the
structure for planning this evolution: reports from chosen plays
change an arm's learning state while other arms are played, on a
known calendar. We plan the evolution of beliefs, budgets and pending
reports jointly; execution must preserve the promised learning
transitions or account for their loss.
This action--information coupling is also central to dual control
\citep{klenske2016dual}.

An observed outcome selects a branch already represented in the plan.
If a local experiment is
completed under the specified report calendar, it produces the
promised law of next learning states. The next local program is
valued under that premise. But the LP fits the arms' workloads into an
episode only in expectation, and the hard horizon may prevent
completion. An unfinished experiment leaves the learner without
some of the evidence used to value future decisions. It changes the
decision problem from which the learner must continue.
Execution must therefore preserve the learning transitions on which
continuation was valued, or charge for their loss and continue from
the information and resources actually available.\footnote{This is 
the \emph{Leap of LP}, echoing Kierkegaard's distinction between
understanding backward and living forward.
For C.~S.~Peirce, testing a conditional prediction includes producing
or finding the conditions under which it applies
(Lowell Lectures, 1903, Lecture~VII).
John Dewey treats inquiry as the directed transformation of the
situation under investigation
(\emph{Logic: The Theory of Inquiry}, 1938, Chapter~VI).
Here execution must establish the informational conditions under
which continuation was valued.}
Planning and execution thus have a predictor--corrector structure.
The LP supplies the value and transition certificate; the executor
enforces pathwise feasibility. In the classic Mizuno--Todd--Ye interior-point method,
the predictor reduces the duality gap and the corrector restores
centrality while retaining that reduction \citep{mizuno1993adaptive}.
Here the quantity to protect is the value of future learning.
Preserving continuation while approximating current reward lets the
episode guarantees compose without multiplying their reward losses
across episodes (Equation~\ref{eq:coordinate-composition}).

\paragraph{The challenge of a proxy.}
The Bayes-optimal learner solves a joint program over the product of
the arms' state spaces. An algorithm therefore targets
$V_{\mathrm{prior}}$ through a proxy $V_{\mathrm{proxy}}$, the value of
the decision problem it actually solves:
\begin{equation}\label{eq:proxy-core}
 V_{\mathrm{state}}-V_A
 =(V_{\mathrm{state}}-V_{\mathrm{prior}})
 +(V_{\mathrm{prior}}-V_{\mathrm{proxy}})
 +(V_{\mathrm{proxy}}-V_A).
\end{equation}
The second term is the choice of proxy and the third is how much of it
the algorithm realizes; $V_{\mathrm{proxy}}=V_{\mathrm{learn}}$ recovers
\eqref{eq:core}. The learning LP has value $V_{\LP}\ge V_{\mathrm{learn}}$:
it keeps the learner's information but relaxes shared capacity and fast
feedback. Its second term,
$(V_{\mathrm{prior}}-V_{\mathrm{learn}})-(V_{\LP}-V_{\mathrm{learn}})$,
is signed, so the LP can value a learning program before we prove how
much of it can be realized. Its third term is the leap from the LP
certificate to a feasible learning policy.
Section~\ref{sec:unroll} bounds this term; a guarantee
$V_A\ge\alpha V_{\LP}-\varepsilon$ also gives
$V_A\ge\alpha V_{\mathrm{learn}}-\varepsilon$, with no ordering between
$V_{\LP}$ and $V_{\mathrm{prior}}$.

For posterior sampling, consider instead
$V_{\mathrm{proxy}}=V_{\mathrm{TS}}$, the reward of hierarchical
Thompson sampling (HierTS) told the recurring prior.
Meta-regret against a prior-informed Thompson-sampling rule measures
$V_{\mathrm{TS}}-V_A$ \citep{bastani2022meta};
the allocation loss $V_{\mathrm{prior}}-V_{\mathrm{TS}}$ remains.
HierTS shares information across episodes \citep{hong2022}.
Its sampled-mean rule chooses the feasible arm with largest sampled
episode mean, using the reports received.
Even with the recurring priors known, this rule loses at least
$NT/32$ under delays and $NT/12$ under budgets and caps
(Examples~\ref{ex:delay-motivation} and~\ref{ex:budget-motivation}).
Thus $V_{\mathrm{prior}}=V_{\mathrm{learn}}$ can coexist with a
linear allocation loss.
Evaluating risk conditional on a sampled episode quality also fails
for mean--variance, CVaR and downside objectives
(Example~\ref{ex:risk}). Appendix~\ref{app:examples} gives these
comparisons, their calculations and the corresponding planning solutions.
Bayes-regret analyses bound the complete loss against the best action
for each realized task in their settings
\citep{kveton2021meta,basu2021,hong2022}.
The LP's prices and continuation values give a common certificate:
every admissible policy's shortfall splits into unused capacity and
Bellman slacks. Section~\ref{sec:certificates} uses it to compare
planning and execution, including HierTS.

\paragraph{Contributions.}
\begin{itemize}
\item \emph{The Learning LP and risk generalizations.}
For general bounded reward distributions, we formulate a relaxation
over learning trajectories. It decomposes into a backward cascade of
the single-arm LPs of \citet{GM13} (GM13 hereafter), linked by
continuation values (Theorem~\ref{lem:perarm}). For explicit finite
state graphs, its size is polynomial in their total size. Episode
weights set the tradeoff between current reward and the later
decisions improved by learning (Section~\ref{sec:weights}).
Bounded outcome scores enter linearly. For risk functionals, we identify the
conditions for arm decomposition and execution.

\item \emph{Feasible execution.}
Full completion uses a time reserve whose leading lower bound
matches as the caps grow; protecting selected reports avoids that reserve
(Section~\ref{sec:unroll}). For a single undelayed episode, caps at
most $T/X$ improve GM13's factor $2$ to
$1+(2\pi X)^{-1/2}+O(X^{-1})$, with a matching leading lower bound.
Fixed-order execution also permits up to $K$ distinct arms per step
(Theorem~\ref{thm:k-lanes}).
With arbitrary caps and delays at most $\delta$, the factor tends to
$2$ whenever $\delta\log\delta=o(T)$.

\item \emph{Learning-value stability and verified execution.}
The episode dual bounds planning-value error by the weighted
total-variation error in observable experiment laws, with sharp
constant $1$. Protected reports transfer this comparison through
execution, paying the single-episode factor once. Short probes give
an executable $O(N^{2/3})$ stationary comparison; the same argument
permits drift (Section~\ref{sec:rates}, Appendix~\ref{app:moving}).

\item \emph{Information prices and predictive checks.}
Capacity prices bound a report's value in the learning LP; repricing
gives an exact nonnegative correction to its fixed-price valuation.
Purchased reports and prior buyout enter the local programs.
The price $V_{\mathrm{prior}}-V_{\mathrm{learn}}$ is $O(\sqrt N)$ under
Proposition~\ref{prop:ps-price}'s assumptions, with other parameters
fixed. A common dual certificate evaluates every admissible policy,
including HierTS. Its reported innovations separate allocation loss
from predictive error and support sequential model checks
(Section~\ref{sec:cost}).
\end{itemize}

Section~\ref{sec:model} gives the model and learning LP;
Section~\ref{sec:unroll} develops execution and its extension to
outcome scores. Section~\ref{sec:rates} proves stability and learning
comparisons. Section~\ref{sec:cost} prices information and develops
the policy certificates and predictive checks.
Appendix~\ref{app:examples} collects the illustrative examples and
their planning solutions. The remaining appendices give execution
proofs, statistical calculations and information-price bounds.

\subsection{Other related work}\label{sec:related}

\paragraph{Learning priors.}
\citet{simchowitz2021} bound the sensitivity of Bayesian decision
rules to prior misspecification and derive meta-learning guarantees
from prior estimation.
\citet{azizi2023meta} study meta-learning for best-arm identification,
using simple regret as the task objective.

\paragraph{Restless bandits and relaxations.}
Lagrangian relaxation separates weakly coupled dynamic programs
by pricing their shared resources
\citep{whittle1988,adelman2008}.
\citet{liu2010indexability} establish indexability for two-state
Markov channels, and \citet{caro2010} study no-crossover response delays.
\citet{jung2019} analyze posterior sampling for unknown restless
dynamics in resetting episodes, relative to a chosen model-to-policy
map. \citet{jiang2023restless} study unknown transition and reward
laws with unobserved states, using a planning oracle.
Paid state probes also appear in \citet[Section~7]{GMS10}.

\paragraph{Delayed and partial feedback.}
\citet{joulani2013} give reductions from delayed-feedback learning
to algorithms for immediate feedback.
\citet{kalkanli2021} analyze Thompson sampling with updates at
batch boundaries.
Impatient Bandits combines partial observations and delayed rewards
through a predictive model
\citep{mcdonald2023,zhang2025}.

\section{Model and the learning LP}\label{sec:model}\label{sec:promise}\label{sec:learn}

\subsection{Model}

There are $n$ arms, indexed by $i\in\{1,\ldots,n\}$, and $N$
episodes, indexed by $e\in\{1,\ldots,N\}$. The planning horizon $N$
and positive integer episode lengths $T_1,\ldots,T_N$ are fixed and
known before planning. Write $\mathbf T=(T_1,\ldots,T_N)$ and
$T=\max_eT_e$.
At each step the learner plays one arm or waits. Waiting earns zero,
uses no arm budget and allows scheduled reports to arrive.
The LP below allows unequal known lengths. The execution and rate
statements use the uniform case $T_e=T$ unless stated otherwise.

Arm $i$ has an episode quality $\theta_i^{(e)}$ in a standard Borel
parameter space $\Theta_i$. Given that quality, its successive rewards
have a known law $K_i(\theta_i^{(e)},dy)$ on $[0,1]$, with mean
$\mu_i(\theta)=\int y\,K_i(\theta,dy)$.
The quality specifies a reward distribution: taking $\Theta_i$ to
be the probability measures on $[0,1]$ and
$K_i(\theta,\cdot)=\theta$ includes every bounded reward law.
Successive plays consume an outcome tape $(Y_{i,j}^{(e)})_{j\ge1}$,
which may be sampled before execution. A play earns its outcome
immediately, but the learner observes it only when its report arrives.
The policy's random coins are independent of the arms.

\begin{assumption}[Independent prior sequences across arms]\label{ass:drift}
Arm $i$ has a random sequence of priors
$\mathcal P_i=(P_i^{(e)})_{e\ge1}$ on $\Theta_i$, with known law
$\Q_{1,i}$. These sequences are independent across arms.
Conditional on them, the qualities are independent with
$\theta_i^{(e)}\sim P_i^{(e)}$. Conditional on the qualities, tape
entries are independent with law $K_i(\theta_i^{(e)},\cdot)$.
\end{assumption}

The learner knows the meta-prior $\Q_1=\bigotimes_i\Q_{1,i}$, but not
the realized priors. In the stationary case $P_i^{(e)}=P_i$ for every
episode. There are then two levels of exchangeability: episodes are
i.i.d.\ given the recurring prior, and plays within an episode are
i.i.d.\ given its quality. Conversely, invariance of each arm's
infinite outcome array under permutations of episodes and of plays
within each episode gives this hierarchy by the hierarchical
de~Finetti representation \citep{austin2014hierarchical}.
The model requires no parametric or conjugate family.

Arm $i$ has a total play budget $B_i$ and a per-episode cap $L_i$,
both nonnegative integers.
Write $B=(B_i)_i$ and $\Lbar=\max_iL_i\le T$.
If its remaining budget is $b$, it can make at most
$\min\{L_i,T_e,b\}$ plays in episode $e$. Both limits hold on every
realized path.

Arm $i$ has known delay $\delta_i\ge1$. A play at step $t$ reports
before the decision at $t+\delta_i$ if that step lies in the same
episode; $\delta_i=1$ is undelayed feedback. In episode $e$, the arm
is \emph{fast} if $\delta_i<T_e$ and \emph{slow} otherwise.
A fast arm's remaining outcomes report at the next episode start.
A slow arm's episode-$e$ batch reports at a known boundary
$\mathrm{rep}_i(e)\in\{e+1,\ldots,e+g\}$, where the integer $g\ge1$
bounds the report lag in episodes. Reports from earlier episodes
arrive only at boundaries. The calendar is fixed in advance and
independent of outcomes and within-episode scheduling.
For $g=1$, all reports arrive before the next episode; otherwise
queues can remain pending. Time between episodes is outside the
reward horizon. Write $\delta=\max_i\delta_i$ when a common delay
bound is used.

\subsection{Policies and states}

A policy is \emph{delay-respecting} if it uses only reported outcomes
and independent randomness. It is \emph{admissible} if it also plays
at most one arm per step and respects all caps and budgets. Let $H_{e-1}$
be the history available at episode $e$'s start, including reports
delivered there. Write $V_e$ for its reward and $\OPT_N(B)$ for the
optimal expected reward over the $N$ episodes.

A \emph{local policy} $\pi$ specifies one arm's plays within one
episode; we also call it an experiment. An \emph{arm program} $\sigma_i$ chooses its local policies
over all episodes from that arm's own reports and independent coins.
At the start of episode $e$, write
$\Q_{e,i}=\mathcal L(\mathcal P_i\mid H_{e-1})$ for arm $i$'s
posterior over its prior sequence, $b_{e,i}$ for its remaining budget,
and $\mathcal D_{e,i}$ for its pending reports with their episode
labels and delivery dates. The queue records the plays awaiting
reports, not their unobserved outcomes. These objects determine the
conditional joint law of the arm's future outcomes and reports. The arm's boundary state is
$S_{e,i}=(\Q_{e,i},b_{e,i},\mathcal D_{e,i})$.
Let $\mathcal S_{e,i}$ be the
possible boundary states of this arm in episode $e$, so
$S_{e,i}\in\mathcal S_{e,i}$;
initially $S_{1,i}=(\Q_{1,i},B_i,\varnothing)$. In a calculation for a
fixed arm and episode, we write $S=(\Q,b,\mathcal D)$ for a possible
state and suppress the arm index on its rewards and transition kernels.

\begin{definition}[Planned restlessness]\label{def:planned}
An arm is \emph{planned-restless} if, while other arms are played, its
information changes only through reports of its own earlier plays, by a
delivery rule fixed in advance. A report arrives a known offset after
its play if that step lies in the same episode. Otherwise it arrives at
a boundary that depends only on the play's episode. Neither the rule nor
the boundary depends on outcomes or on how arms are interleaved.
\end{definition}

Every arm in this model is planned-restless. Its conditional reward
law stays fixed within an episode. Completion preserves the law of
its next boundary state under interleaving
(Lemma~\ref{lem:placement}).

For planning, fast arms get immediate access to their current
outcomes; this is the \emph{delay relaxation}. Slow arms use
nonadaptive local policies. Each arm program still respects its cap,
budget and boundary report calendar. Execution enforces the physical
delays.

\subsection{Single-arm policies}\label{sec:local-policies}

Use the posterior states of \citet[Sections~2.1 and~3.1]{GM13}.
At a state $u$, let $r(u)$ be the expected reward of a play and
$p(u,v)$ its transition probabilities. The state records the plays
made, remaining budget and reports. It may be a complete observable
history or a sufficient statistic preserving conditional future laws
and feasible actions. The play count makes the local graph acyclic.
The root has no predecessors. At each playable node $u$,
$p(u,v)\ge0$ and $\sum_vp(u,v)=1$.
For a slow arm it has no outcome-dependent branches before handoff.
Bayesian updating gives $r(u)=\sum_vp(u,v)r(v)$, including the
predictive mean at terminal states.

For a finite graph rooted at $u_0$, let $f(u)$, $w(u)$ and $z(u)$
be the probability assigned to reaching, playing at and stopping at $u$.
Then
\begin{equation}\label{eq:local-flows}
 f(u) \leq \mathbf1_{\{u=u_0\}}+\sum_vp(v,u)w(v), \qquad f(u)\leq 1, \qquad
 w(u)+z(u)=f(u),\qquad f,w,z\ge0.
\end{equation}
The reach inequality permits unused probability. In particular,
$0\le w(u),z(u)\le f(u)\le1$.
Set $w(u)=0$ where the cap or remaining budget forbids a play.

\begin{proposition}[Local policy representation]\label{prop:lawfree}
Every allowed local policy gives a feasible solution of
\eqref{eq:local-flows}. Conversely, every feasible solution has a
completion $(\bar f,w,\bar z)$ that preserves the play probabilities
and satisfies equality in the reach constraints. Playing with
probability $w(u)/\bar f(u)$ realizes this completion. Its expected
reward, workload and stopping probabilities are
\begin{equation}\label{eq:flow-linear}
 R(\pi)=\sum_u r(u)w(u),\qquad
 \Work(\pi)=\sum_uw(u),\qquad
 \Pr\{\text{stop at }u\}=\bar z(u).
\end{equation}
\end{proposition}
\begin{proof}
A policy's probabilities satisfy the reach constraints with equality.
For a feasible LP solution, put
\[
 \bar f(u)=\mathbf1_{\{u=u_0\}}+\sum_vp(v,u)w(v),\qquad
 \bar z(u)=\bar f(u)-w(u).
\]
Then $\bar f\ge f$ and $\bar z\ge z$. Induction over the DAG shows
that playing with probability $w/\bar f$ gives arrival probabilities
$\bar f$, play probabilities $w$ and stopping probabilities $\bar z$.
Thus $\bar f\le1$ and $\sum_u\bar z(u)=1$.
Summing play contributions proves \eqref{eq:flow-linear}.
Choices at zero-probability states are arbitrary.
\end{proof}

Policy and handoff calculations below use completed solutions; we
write their probabilities as $f,w,z$.

For a finite reward alphabet $\mathcal Y_i$, the likelihood of an
ordered outcome history $h=(y_1,\ldots,y_j)$ under an episode prior
$P$ is
\begin{equation}\label{eq:likelihood}
 \ell_i(P;h)=\int_{\Theta_i}\prod_{t=1}^j
 K_i(\theta,\{y_t\})\,P(d\theta).
\end{equation}
It is linear in $P$ and satisfies
$\ell_i(P;h)=\sum_y\ell_i(P;hy)$.
Suppress the arm index. Let $A(h)$ be the probability, over a local
policy's coins, of making the plays in $h$ when supplied those
outcomes, and $B(h)$ the probability of also making the next play.
Then $A(\varnothing)=1$, $A(hy)=B(h)$ and $0\le B(h)\le A(h)$.
The reach, play and terminal probabilities are respectively
\begin{equation}\label{eq:history-weights}
 A(h)\ell(P;h),\qquad B(h)\ell(P;h),\qquad
 (A(h)-B(h))\ell(P;h).
\end{equation}
These identities hold for delayed policies because each decision uses
only reported outcomes and independent coins. In particular,
\[
 R^P(\pi)=\sum_h B(h)\sum_{y\in\mathcal Y}y\ell(P;hy),\qquad
 \Work^P(\pi)=\sum_hB(h)\ell(P;h).
\]

A depth-$k$ commit policy makes at most $k$ adaptive plays and then
fixes its remaining play count.
\begin{corollary}[Information in an experiment]\label{cor:tree}
\textup{(i)} The stopped outcome law, reward and workload use only
likelihoods at histories reached with positive weight.
\textup{(ii)} Given the policy weights, the stopped law determines
these likelihoods.
\textup{(iii)} Fixed recurring Bernoulli priors with the same first $m$ moments are
indistinguishable by local policies capped at $m$ plays per episode.
\textup{(iv)} A depth-$k$ commit policy with cap $L$ has reward and
workload determined by the law of $\min\{L,k+1\}$ observations;
its continuation can require all $L$ observations.
\end{corollary}
\begin{proof}
Equation~\eqref{eq:history-weights} gives (i). The stopped law gives
each prefix probability; division by $A(h)>0$ gives (ii).
For Bernoulli rewards a length-$j$ likelihood is
$\int\theta^s(1-\theta)^{j-s}P(d\theta)$, a polynomial moment
of degree $j$, proving (iii). For (iv), commitment fixes the tail
length from at most $k$ outcomes. Conditional on that prefix and
the quality, its expected reward is unchanged by replacing the tail
by its length times one further outcome. Its full outcome law still
uses all $L$ observations.
\end{proof}

\begin{proposition}[History likelihood]\label{prop:history}
In a global observable history $h$, let $h_{i,u}$ list the outcomes
reported from arm $i$ and episode $u$. For finite alphabets,
\begin{equation}\label{eq:history-law}
 \Pr_\pi(H=h)=w_\pi(h)\prod_i\ell_i(h),\qquad
 \ell_i(h)=\E_{\Q_{1,i}}\prod_u\ell_i(P_i^{(u)};h_{i,u}),
\end{equation}
where $w_\pi(h)$ depends on the policy and reported history but is
independent of the latent prior sequences. Unreported outcomes are
marginalized.
\end{proposition}
\begin{proof}
Fix the reported history. The policy's choices contribute a factor
obtained by averaging over its independent coins. Conditional on the
prior sequences, integrating each episode quality gives its
likelihood in \eqref{eq:likelihood}. Independence across arms gives
the product in \eqref{eq:history-law}. The fixed report calendar
contributes no outcome-dependent factor.
\end{proof}

\begin{lemma}[The conditional episode]\label{lem:posterior}
Given $H_{e-1}$, the posterior factorizes as
$\Q_e=\bigotimes_i\Q_{e,i}$, and the episode has independent
predictive priors $\Pb_{e,i}=\E[P_i^{(e)}\mid H_{e-1}]$.
Writing $\Pb_e=(\Pb_{e,i})_i$, every episode policy satisfies
$\E[V_e\mid H_{e-1}]=V^{\Pb_e}(\pi_e)$.
\end{lemma}
\begin{proof}
The policy factor in \eqref{eq:history-law} cancels in Bayes' rule,
leaving independent updates of the arm meta-priors. The new tape law
is linear in $P_i^{(e)}$, so averaging gives $\Pb_{e,i}$.
Reports from earlier episodes arrive only at boundaries.
\end{proof}

The predictive priors suffice for current-episode reward. Continuation
requires the full beliefs $\Q_{e,i}$, which determine how new reports
change later decisions (Example~\ref{ex:spread}). For general
observation spaces these statements use conditional kernels and
integrals. Finite LPs require a supplied finite state representation.

\subsection{The continuation}\label{sec:continuation}

Fix arm $i$, an episode $e$ and its incoming boundary state $S$.
We suppress the arm index throughout this calculation. Each arm has
$N$ episode layers. In its DAG, the episode-$e$ local policy moves
probability mass to possible
starting nodes of episode $e+1$. Different reports can give different
posterior beliefs about the recurring prior; the plays also change
the remaining budget and the queue. These are the states from which
the arm's program continues.

\begin{definition}[Planned handoff]\label{def:handoff}
For an allowed local policy $\pi$, the \emph{planned handoff}
$\mathcal T_e^\pi(S,\cdot)$ is the law of the arm's state at the start
of episode $e+1$ when $\pi$ runs alone and to completion (undelayed if
fast, nonadaptively if slow), and the reports due at
that boundary are delivered.
\end{definition}

On a finite state graph, write $H_e(S,u,v)$ for the probability that
stopping at local node $u$ leads to next-episode starting node $v$.
This transition deducts the plays, updates the queue and delivers the
reports due at the boundary. The incoming state and local node
determine these probabilities; the policy determines its stopping
probabilities $z(u)$. Thus the mass arriving at $v$ is
\begin{equation}\label{eq:boundary-flow}
 q(v)=\sum_u z(u)H_e(S,u,v)=\mathcal T_e^\pi(S,v).
\end{equation}
This is the boundary flow joining successive episode DAGs.

Assign a future value $F(v)$ to each next-episode starting node.
The current policy's \emph{continuation contribution} is the expected
value of the state it produces:
\begin{equation}\label{eq:continuation-flow}
  C_F(\pi)=\sum_v q(v)F(v)=\sum_u z(u)c_F(u),\qquad c_F(u)=\sum_v H_e(S,u,v)F(v).
\end{equation}
At a stopping history, the policy weight cancels in Bayes' rule.
Its posterior, remaining budget, queue and conditional law of reports
due at the boundary therefore fix $H_e$ independently of the policy.
Thus both the handoff and the continuation in
\eqref{eq:continuation-flow} are linear in terminal probabilities,
even when $F$ is nonlinear in the posterior.
Figure~\ref{fig:experiment-continuation} shows these two operations.
For general observation spaces the sums are integrals against the
handoff kernels.

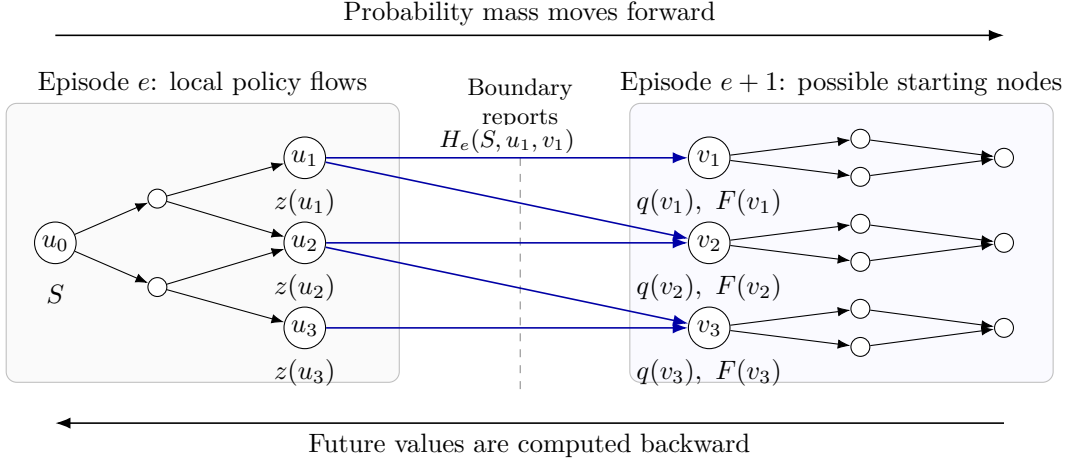
\begin{figure}[htbp]
\centering
\begin{tikzpicture}[>=Latex,font=\small,x=1cm,y=0.9cm,
 dot/.style={circle,draw,fill=white,minimum size=5.5mm,inner sep=0pt},
 small dot/.style={circle,draw,fill=white,minimum size=2.5mm,inner sep=0pt},
 handoff/.style={->,draw=blue!65!black,semithick}]
\filldraw[fill=black!2,draw=black!25,rounded corners]
 (-0.65,-2.05) rectangle (4.55,2.05);
\filldraw[fill=blue!2,draw=black!25,rounded corners]
 (7.60,-2.05) rectangle (13.20,2.05);
\node at (1.95,2.35) {Episode $e$: local policy flows};
\node at (10.40,2.35) {Episode $e+1$: possible starting nodes};
\draw[dashed,black!45] (6.15,-2.15)--(6.15,2.15);
\node[align=center,font=\footnotesize,fill=white] at (6.15,2.05)
 {Boundary\\reports};
\node[dot] (root) at (0,0) {$u_0$};
\node[below=4pt of root] {$S$};
\node[small dot] (b1) at (1.35,0.65) {};
\node[small dot] (b2) at (1.35,-0.65) {};
\node[dot,label=below:{$z(u_1)$}]
 (u1) at (3.30,1.25) {$u_1$};
\node[dot,label=below:{$z(u_2)$}]
 (u2) at (3.30,0) {$u_2$};
\node[dot,label=below:{$z(u_3)$}]
 (u3) at (3.30,-1.25) {$u_3$};
\draw[->] (root)--(b1);
\draw[->] (root)--(b2);
\draw[->] (b1)--(u1);
\draw[->] (b1)--(u2);
\draw[->] (b2)--(u2);
\draw[->] (b2)--(u3);
\foreach \j/\y in {1/1.25,2/0,3/-1.25}{
 \node[dot,label=below:{$q(v_\j),\ F(v_\j)$}]
   (v\j) at (8.65,\y) {$v_\j$};
 \node[small dot] (a\j) at (10.65,\y+0.28) {};
 \node[small dot] (b\j) at (10.65,\y-0.28) {};
 \node[small dot] (c\j) at (12.55,\y) {};
 \draw[->] (v\j)--(a\j);
 \draw[->] (v\j)--(b\j);
 \draw[->] (a\j)--(c\j);
 \draw[->] (b\j)--(c\j);
}
\draw[handoff] (u1)--node[above,font=\footnotesize,fill=white,inner sep=1.5pt]
 {$H_e(S,u_1,v_1)$} (v1);
\draw[handoff] (u1)--(v2);
\draw[handoff] (u2)--(v2);
\draw[handoff] (u2)--(v3);
\draw[handoff] (u3)--(v3);
\draw[->,semithick] (0,3.05)--node[above]
 {Probability mass moves forward} (12.55,3.05);
\draw[->,semithick] (12.55,-2.65)--node[below]
 {Future values are computed backward} (0,-2.65);
\end{tikzpicture}
\caption{One arm across an episode boundary. Shown are selected stopping
nodes $u$ of the current local DAG and starting nodes $v$ of the next.
The handoff sends stopping mass $z(u)$ through the fixed probabilities
$H_e(S,u,v)$, giving arrival masses $q(v)$. The continuation is
$\sum_v q(v)F(v)$. Each next starting node carries its posterior,
remaining budget and pending reports.}
\label{fig:experiment-continuation}
\end{figure}

For a fast arm with a finite reward alphabet and no earlier report
pending, the handoff from a terminal history $h$ is deterministic:
\begin{equation}\label{eq:handoff}
 S_h^+=(\Q_h^+,b-|h|,\varnothing),\qquad
 \Q_h^+(d\mathcal P_i)\ \propto\
 \ell_i(P_i^{(e)};h)\,\Q(d\mathcal P_i).
\end{equation}
The policy's action probabilities cancel from Bayes' rule, so the
state attached to $h$ is fixed independently of its flow weight.
Pending reports can make the handoff random, as represented by the
kernel in \eqref{eq:boundary-flow}. The next subsection derives the
arm's future values by pricing the shared episode capacities.

\subsection{The learning LP and its prices}\label{sec:across}

The learning LP has $N$ episode layers for each arm. A layer contains
a local program for every possible incoming boundary state. Stopping
in layer $e$ sends probability mass to starting states in layer
$e+1$. At fixed capacity prices, their values are computed backward
from episode $N$. The prices are chosen jointly across episodes.
Thus the cascade has depth $N$; its size also counts the states and
transitions in each layer. Each local block is the single-arm LP of
GM13, with its terminal values supplied by the next episode.

Let $R_e(\sigma_i)$ and $\Work_e(\sigma_i)$ be arm $i$'s expected
reward and number of plays in episode $e$. Each arm program respects
its cap and total budget on every path. For known capacities
$\mathbf C=(C_1,\ldots,C_N)$ with $0<C_e\le T_e$, define
\begin{equation}\label{eq:lpmeta}
 \LP_N(\mathbf C,B)=\sup_{(\sigma_i)}\sum_{e=1}^N\sum_iR_e(\sigma_i)
 \quad\text{subject to}\quad
 \sum_i\Work_e(\sigma_i)\le C_e\quad(e=1,\ldots,N).
\end{equation}
Only these shared capacity constraints hold in expectation.
For uniform lengths and capacities, write $\LP_N(C,B)$ for
$\mathbf C=(C,\ldots,C)$. The subscript $N$ records the number of
episode layers; $\LP(P)$ below denotes a single-episode relaxation.

\begin{proposition}[Relaxation across episodes]\label{prop:relax}
\textup{(i)} $\OPT_N(B)\le\LP_N(\mathbf T,B)$.
\textup{(ii)} For $0<s\le1$,
$\LP_N(s\mathbf C,B)\ge s\LP_N(\mathbf C,B)$; in particular,
$\LP_N(C,B)\ge(C/T)\LP_N(T,B)$ in the uniform case.
\textup{(iii)} Pricing the episode capacities separates the arms;
fixing the future values makes each local step an LP.
\end{proposition}
\begin{proof}
(i) Project an admissible policy onto each arm. The other arms' latent
variables and tapes, together with the policy coins, supply randomness
independent of that arm. Its projected program preserves reward,
workload, caps, budget and boundary reports under relaxed fast
feedback. The expected workloads sum to at most $T_e$ in episode $e$.
(ii) Retain each entire arm program independently with probability
$s$ and otherwise make no plays of that arm. Reward and every
episode's expected workload are multiplied by $s$.
(iii) Price the capacity constraints as below.
\end{proof}

The conditional reward law is fixed within an episode. The
relaxation therefore uses one capacity constraint per episode,
without a within-episode time index \citep[Claim~1]{GM13}.

\paragraph{The LP on states.}
Join each arm's local graphs through their handoff kernels. Let
$\mathcal V_i$ be its acyclic graph over episodes $1,\ldots,N$,
$o_i$ its initial node, and $e(u)$ the episode of node $u$.
Write $p_i(u,v)$ for play transitions and $h_i(u,v)$ for handoffs to
the next episode. Each allowed play and each handoff has total
outgoing probability one. Handoffs after episode $N$ enter a terminal state
of value zero. Its incoming mass need not be included as a variable.
The finite LP is
\begin{equation}\label{eq:lp-state}
\begin{aligned}
 \LP_N(\mathbf C,B)=\max\ &\sum_i\sum_{u\in\mathcal V_i}r_i(u)w_i(u)\\
 \text{s.t.}\quad
 f_i(u)& \leq \mathbf1_{\{u=o_i\}}+
 \sum_v\bigl[p_i(v,u)w_i(v)+h_i(v,u)z_i(v)\bigr],\\
 f_i(u) & \leq 1, \\
 w_i(u)+z_i(u)&=f_i(u),\qquad f_i,w_i,z_i\ge0,\\
 \sum_i\sum_{u:e(u)=e}w_i(u)&\le C_e\qquad(e=1,\ldots,N).
\end{aligned}
\end{equation}
There is one unit of initial mass per arm; later episode roots
receive handoff mass from the preceding layer.
Complete a feasible solution in topological order: keep $w_i$,
set $\bar f_i$ to the full initial and incoming probability, using
the completed handoffs, and set $\bar z_i=\bar f_i-w_i$.
Induction gives $\bar f_i\ge f_i$ and $\bar z_i\ge z_i$.
The completed arm program has the same play probabilities, reward
and workload. All its boundary-state distributions have total mass one.
Forbidden plays have $w_i(u)=0$. The state records the remaining
budget and current play count, so these restrictions enforce both
pathwise limits. Only the capacity constraints couple the arms.

Let $\lambda_e\ge0$ be the dual multiplier of episode $e$'s capacity
constraint, and write $\lambda=(\lambda_1,\ldots,\lambda_N)$.
For fixed prices, $U^\lambda_{e,i}(S)$ is the maximum expected
remaining reward of arm $i$ from state $S$, less $\lambda_j$ for
each play in episode $j\ge e$. Put $U^\lambda_{N+1,i}=0$.

\begin{theorem}[Dual and episode decomposition]\label{prop:dual-meta}\label{lem:perarm}
For finite state graphs, the learning LP \eqref{eq:lp-state} has dual
\begin{equation}\label{eq:dual-meta}
 \LP_N(\mathbf C,B)=\min_{\lambda\ge0}\left\{
 \sum_{e=1}^N\lambda_eC_e+
 \sum_iU^\lambda_{1,i}(\Q_{1,i},B_i,\varnothing)\right\}.
\end{equation}
At fixed prices, the nonnegative node values $v_i(u)$ satisfy
\begin{equation}\label{eq:state-dual}
 v_i(u)\ge r_i(u)-\lambda_{e(u)}+\sum_vp_i(u,v)v_i(v),\qquad
 v_i(u)\ge\sum_vh_i(u,v)v_i(v),
\end{equation}
with the play inequality omitted where playing is forbidden.
Their smallest feasible values are computed by a backward cascade
of $N$ single-episode LP layers, one local LP per incoming state.

Fix arm $i$ and $\lambda$, suppress these indices in $U_e$, and put
$U_{N+1}=0$. Root the episode-$e$ local graph at $S$. For
$f,w,z$ satisfying \eqref{eq:local-flows} and next-state
probabilities $q$ from \eqref{eq:boundary-flow},
\begin{equation}\label{eq:bellman-meta}
 U_e(S)=\max_{f,w,z}\left\{
 \sum_u(r(u)-\lambda_e)w(u)+\sum_{S'}q(S')U_{e+1}(S')\right\}.
\end{equation}
The first sum is priced current reward; the second is continuation.
With future values fixed, put
$c_e(S,u)=\sum_{S'}H_e(S,u,S')U_{e+1}(S')$.
This is GM13's priced single-arm LP \citep[Section~3.4]{GM13},
with stopping value $c_e(S,u)$ and recursion
\begin{equation}\label{eq:play-or-stop}
 v(u)=\max\left\{c_e(S,u),\quad
 r(u)-\lambda_e+\sum_{u'}p(u,u')v(u')\right\}.
\end{equation}
Only stopping is allowed at forbidden plays, and $U_e(S)=v(u_0)$.
A slow arm chooses a nonadaptive play count. Each local maximizer
can be completed without loss, so $q$ is a probability distribution.
\end{theorem}
\begin{proof}
The acyclic graph and unit initial mass imply $f_i(u)\le1$, so
these bounds are redundant. Eliminate $f_i=w_i+z_i$ from the reach
inequalities. Their nonnegative multipliers $v_i(u)$ and the
capacity multipliers $\lambda_e$ give \eqref{eq:state-dual} and
objective $\sum_e\lambda_eC_e+\sum_iv_i(o_i)$.

At fixed future values, \eqref{eq:flow-linear} and
\eqref{eq:continuation-flow} give \eqref{eq:bellman-meta}.
Stopping at $u$ gives $c_e(S,u)$; playing gives the other term in
\eqref{eq:play-or-stop}. Backward induction on each local DAG and
then on episodes gives the smallest node values and identifies
$v_i(o_i)$ with $U^\lambda_{1,i}$ at its initial state.
Since $c_e(S,u)\ge0$, stopping dominates discarding probability,
which proves completion. Finite LP duality gives \eqref{eq:dual-meta};
complementary slackness supplies mixtures of maximizing arm programs
that satisfy the capacities and attain the primal value.
\end{proof}

The prices are optimized jointly across episodes. At fixed prices,
backward induction assigns values to next-episode starting states;
the local LP chooses their probability masses. A zero-play policy
still receives reports due at the boundary and retains its future
options. Every nonnegative price vector gives an upper bound through
\eqref{eq:dual-meta}.

\begin{proposition}[Size]\label{prop:size}
Let $V_*=\sum_i|\mathcal V_i|$ and let $E_*$ count the nonzero play
and handoff transitions. Given their probabilities and reward
coefficients, the LP has $O(V_*+N)$ variables and constraints and
$O(V_*+E_*+N)$ nonzero coefficients. Fixed-price backward induction
takes $O(V_*+E_*)$ operations.
For stationary priors, fixed finite reward alphabets, fixed caps and
fixed report lag, a sufficient-state graph has size polynomial in
$N$ and the budgets.
\end{proposition}
\begin{proof}
There are three variables and a constant number of constraints per vertex, one capacity
row per episode and one coefficient per transition. Backward induction
visits each edge once.
For an alphabet of size $d_i$, there are
$M_i=\binom{\Lbar+d_i}{d_i}$ outcome histograms of at most $\Lbar$
plays. Write $\ell_i(P_i;\boldsymbol s)$ for the likelihood of any
ordering of histogram $\boldsymbol s$. If $n_{\boldsymbol s}$ completed episodes have histogram
$\boldsymbol s$, their likelihood is
$\prod_{\boldsymbol s}\ell_i(P_i;\boldsymbol s)^{n_{\boldsymbol s}}$.
Thus at most $\binom{N+M_i}{M_i}$ count vectors determine the
stationary posterior. Record also the budget and the counts and
delivery dates of at most $g-1$ pending batches. These contribute
polynomially many states at fixed lag.
\end{proof}

The bound concerns the supplied state graphs and their coefficients.
Continuous observations may require an infinite graph or a specified
finite approximation.

With one arm and $C_e\ge\min\{L_1,T_e\}$ in every episode, the
capacity constraints are redundant and $\lambda_e=0$ is optimal.
The program gives the Bayes-optimal learner over the allowed local
policies. Slow-arm policies are physically admissible; fast-arm
policies use the delay relaxation unless feedback is undelayed.

\subsection{Weighting use and generalization}\label{sec:weights}

The planner can value current use and later learning differently.
Fix episode weights $\eta_e\ge0$ and bounded terminal values
$\Psi_i\ge0$ on arm $i$'s boundary states after episode $N$.
The terminal value may represent deployment after learning.
The weighted LP $\LP_N^\eta(\mathbf C,B)$ retains the final states
$t$ and maximizes, under the same constraints,
\begin{equation}\label{eq:weighted-objective}
 \sum_i\left[\sum_{u\in\mathcal V_i}
 \eta_{e(u)}r_i(u)w_i(u)+\sum_t\Psi_i(t)f_i(t)\right].
\end{equation}
Unit weights and zero terminal values recover \eqref{eq:lp-state}.

\begin{proposition}[Weighted learning LP]\label{prop:weights}
For finite state graphs the dual has the form \eqref{eq:dual-meta}, with
reward coefficient $\eta_er(u)$ in episode $e$ and terminal value
$U_{N+1}=\Psi_i$ for arm $i$.
\end{proposition}
\begin{proof}
Apply Theorem~\ref{lem:perarm} with these coefficients and terminal
values.
\end{proof}

To separate early use from later learning, fix a boundary $e_0<N$.
Let $V_{\le e_0}$ and $V_{>e_0}$ be the early and late values,
including the terminal payoff in the latter. Write $\LP_N^\eta$
for the scalar-weight objective $V_{\le e_0}+\eta V_{>e_0}$.
Its value is convex and piecewise linear in $\eta$; between
breakpoints its slope is the selected program's late value.
Equivalently, $\eta$ prices a feasible lower target $v_{\rm late}$:
\begin{equation}\label{eq:late-target-price}
 \max\{V_{\le e_0}:V_{>e_0}\ge v_{\rm late}\}
 =\min_{\eta\ge0}
 \{\LP_N^\eta(\mathbf C,B)-\eta v_{\rm late}\}.
\end{equation}
Both sides use the original capacity and arm constraints.

For $\eta>0$, the episode-$e_0$ local objective is
$R_\pi-\lambda_{e_0}W_\pi+\eta C_F(\pi)$, with the late program
valued at unit weight and prices $\lambda_e/\eta$ for $e>e_0$.
Thus one experiment improves on another exactly when
$\Delta R+\eta\Delta C>\lambda_{e_0}\Delta W$.
The weight values the later decisions improved by the experiment;
the capacity price charges for the plays it requires.
At $\eta=0$ and $e_0=1$, the objective values only the current episode.

\paragraph{A specified safe alternative.}
Some examples add an always-available action with deterministic
reward $a\in(0,1]$ per step and no binding play budget. It competes
with reward plays; ordinary waiting still earns zero.
In that extension, eliminating the safe action adds
$a\sum_eT_e$ to the objective and replaces each unknown-arm reward
coefficient $r_i(u)$ by $r_i(u)-a$. The total play price is then
$a$ plus the nonnegative capacity multiplier, or $a\eta_e$ plus
that multiplier for weighted rewards. We use $\LP_N$ and $\OPT_N$
for the specified instance, including this action when present.
The general execution and program statements below use zero-reward
waiting; examples specifying $a>0$ include this safe alternative.
Thinning the arm programs leaves the safe action available, so
Proposition~\ref{prop:relax}(ii) still holds.

Example~\ref{ex:price-test} makes the distinction exact: the learning
LP pays for a test that an episode planner never makes, avoiding a
$\Theta(N)$ loss. The information price is constant in that instance.
Episode weights also price multi-play experiments and purchased
reports through their later value (Examples~\ref{ex:spread}
and~\ref{ex:paid-knowledge}).

\paragraph{Execution of weighted plans.}
For weights in $[0,1]$, Theorems~\ref{thm:meta}
and~\ref{thm:promise-less} use $\E\sum_e\eta_eV_e$, with the
planned value weighted alike. A terminal value depending on the
retained state has its planned expectation whenever execution
preserves that state's law. If the sum of terminal values is bounded
by $\bar\Psi$, a completion failure adds at most
$\bar\Psi q_{\rm fail}$ to the failure charge.
The policy certificate of Theorem~\ref{thm:common-certificate}
uses the same weighted coefficients. The rate statements use unit
weights.

The relaxation fits workloads only in expectation. The next section
enforces the uniform hard horizon while preserving the learning
transitions used to value future decisions.

\section{From the LP to a policy}\label{sec:unroll}\label{sec:execmeta}

Planned restlessness lets us interleave completed arm programs
without changing the learning states they pass to the next episode.
Execution must fit those plays into the hard horizon. We separate
current reward, continuation and resource use, then give two ways to
realize the planned learning trajectory. One completes all prescribed
plays using a reserve and charges failed completion. The other
protects only the reports used by later programs and truncates
reward-only tails. The single-episode bound quantifies that reward
loss.
All input arm programs are extracted from completed LP solutions
as in Proposition~\ref{prop:lawfree}. Waiting earns zero throughout
the execution statements.

\paragraph{Reward, continuation and resource use.}
Fix an episode $e$, the incoming states of all arms, a nonnegative
future value function $F=\sum_iF_i$, where $F_i$ values arm $i$'s next
state, and the planned local policies $\pi=(\pi_{e,i})_i$. Write
\[
 R_\pi=\E[\text{current reward}],\qquad
 W_\pi=\E[\text{plays}],\qquad
 C_F(\pi)=\E[F(S^\pi)]
\]
for the episode as a whole: the reward of all arms, their plays,
and the value of their next states. 
Let $\widehat\pi$ be the feasible joint policy produced by executing
the planned local policies. Its expected reward and workload are
defined as above. Its continuation is evaluated at the state retained
by each arm's program: the data it conditions on, its budget ledger
and its report queue.
Physical plays
and budgets enter only the hard limits. An $(\alpha_e,\beta_e)$
realization preserves reward and continuation if
\begin{equation}\label{eq:coordinate-bound}
 R_{\widehat\pi}\ge\alpha_e R_\pi,\qquad
 C_F(\widehat\pi)\ge\beta_e C_F(\pi),
 \qquad 0\le\alpha_e,\beta_e\le1.
\end{equation}

If $W_{\widehat\pi}\le\rho W_\pi$ with $\rho\ge0$, the coordinate bounds give,
for $\eta,\lambda\ge0$,
\begin{equation}\label{eq:coordinate-priced}
 R_{\widehat\pi}+\eta C_F(\widehat\pi)-\lambda W_{\widehat\pi}
 \ge\alpha_eR_\pi+\eta\beta_eC_F(\pi)-\rho\lambda W_\pi.
\end{equation}
For $\alpha_e>0$, factoring out $\alpha_e$ reprices continuation
at $\eta\beta_e/\alpha_e$ and plays at $\rho\lambda/\alpha_e$.
Thus reward and workload retention compare priced values at
different prices. GM13's exchangeability argument gives the
following reward--workload bound before subtracting any price.

\begin{lemma}[Reward truncation under any true prior]\label{lem:truncation}
Let $\pi$ be a nonanticipating single-arm policy, possibly designed
using a prior $\widehat P$. Couple it to a nonanticipating truncation
$\pi^-$ on the same episode quality, outcome tape and random coins.
Write $\Work^P(\pi)=\E_P N_\pi$ for expected play consumption.
Suppose that, on every path,
$bN_\pi\le N_{\pi^-}\le N_\pi$ for a fixed $b\in[0,1]$.
Then every true prior $P$ on the quality space $\Theta_i$ satisfies
\[
 R^P(\pi^-)\ge bR^P(\pi),\qquad
 \Work^P(\pi^-)\le\Work^P(\pi).
\]
The conclusion also holds conditional on external randomness
independent of the arm, with $b$ fixed by that randomness.
\end{lemma}
\begin{proof}
Fixing the quality, outcome tape and policy coins fixes both play
counts. The hypothesis says that the truncated policy retains at
least a fraction $b$ of the original policy's plays on every such run.

For either policy $\sigma$, the decision to make play $j$ uses only
earlier outcomes and independent coins. Conditional on $\theta$,
its reward $Y_j$ therefore has mean $\mu_i(\theta)$ independently
of that decision. Summing over the possible plays gives
\[
 R^P(\sigma)
 =\sum_{j=1}^{L_i}\E_P[
   \mathbf1_{\{N_\sigma\ge j\}}Y_j]
 =\E_P[\mu_i(\theta)N_\sigma].
\]
Now multiply $N_{\pi^-}\ge bN_\pi$ by $\mu_i(\theta)\ge0$
and take expectations. This proves the reward bound;
$N_{\pi^-}\le N_\pi$ gives the workload bound.
The argument also holds conditional on independent external
randomness and does not depend on the design prior.
\end{proof}

The next-state law determines how episode reward guarantees
compose. Let $R_e$ and $\widehat R_e$ be the conditional expected
rewards of the planned and realized episodes, and put
$J_{\rm plan}=\sum_e\E R_e$. If execution reproduces the planned next-state
law at every boundary, reward guarantees with fixed factors
$\alpha_e$ add:
\begin{equation}\label{eq:coordinate-composition}
 \E\sum_e\widehat R_e\ge\sum_e\alpha_e\E R_e
 \ge(\min_e\alpha_e)J_{\rm plan},
\end{equation}
Matching the incoming-state laws lets us sum the episode expectations.
The reward factor is
therefore paid once. The constructions below preserve the planned
states, or charge the event on which completion fails.

\subsection{Delayed execution and the next learning state}\label{sec:executors}\label{sec:execution}\label{sec:handoff}

\begin{definition}[Executor]\label{def:executor}
A map $\pi\mapsto\pi^{\mathrm x}$ is an \emph{executor} with parameters
$(\epsilon,w)$ if, for every allowed local policy $\pi$ making $m$ plays:
\begin{enumerate}
\renewcommand{\theenumi}{\roman{enumi}}
\renewcommand{\labelenumi}{(\theenumi)}
\item\label{item:executor-tape} It supplies $\pi$ its outcomes in tape
order, preserving the law of its decisions and planned data.
\item\label{item:executor-work} Its number $X$ of plays satisfies
$m\le X\le\min\{\text{cap},(1+\epsilon)m\}$ on every path.
\item\label{item:executor-prefix} After at most $w$ of its own steps,
each remaining own step is a play.
\item\label{item:executor-interleaving} Its actions are determined by
the arm's tape, its coins and $\pi$. Interleaving changes only when
those actions occur.
\end{enumerate}
\end{definition}

Direct execution is an executor with $(\epsilon,w)=(0,0)$ for
undelayed arms and for slow arms' nonadaptive local policies.
For delayed fast arms, buffered execution makes nonadaptive blocks, waits for their
reports, and supplies outcomes to the planned policy in tape order.
Extra outcomes are excluded from its planned data.

\begin{definition}[Buffered execution]\label{def:buffered-execution}
Let $\pi$ be an undelayed policy for arm $i$ and
$\epsilon\in(0,1]$. The buffered executor $A_\epsilon(\pi)$
makes actual plays and supplies their reported outcomes to $\pi$,
at most once each and in play order.

Suppose $\pi$ has consumed all $q$ outcomes generated so far and
requests another play. Make a nonadaptive block of
\[
 b=\min\{\max\{1,\lceil\epsilon q\rceil\},L_i-q\}
\]
plays, one per step. Wait for their reports. Supply the outcomes to
$\pi$ one at a time until it stops, commits, or requests another
block. If it commits inside a block, credit unused block plays
toward its committed tail. Make any remaining tail plays without
waiting for further reports. Extra plays beyond the original policy
are excluded from its planned data.

For a general policy, switch modes once
$q\ge\lceil\delta_i/\epsilon\rceil$. Make up to $\delta_i$
further plays and enter the delay-free mode of
\citet[Lemma~19]{GM13}. At each subsequent request, supply the
oldest unused outcome. Also make a new play on that step if the
cap has not been reached. Stop making new plays when $\pi$ stops
or the cap is exhausted.
\end{definition}

\begin{lemma}[Buffered executor, prefix bounds proved in Appendix~\ref{app:buffered-prefix}]\label{lem:buffered-executor}
$A_\epsilon(\pi)$ reproduces the law of $\pi$'s decision path and
satisfies $R(A_\epsilon(\pi))\ge R(\pi)$. On every path it makes
at most $\min\{L_i,(1+\epsilon)m\}$ plays, where $m$ is the number
of plays made by $\pi$.

For a commit policy, the prefix consists of its own play and wait
steps before the committed plays. For a general policy, it also
includes the plays that start delay-free mode. Its length is at most
\[
 W_i=\begin{cases}
 \delta_iB_\epsilon(k)+(1+\epsilon)k+1,
     &\text{depth-}k\text{ commit},\\
 5(\delta_i+1)(1+\ln(\delta_i+1))/\epsilon,
     &\text{general},
 \end{cases}
\]
where $B_\epsilon(k)=
\min\{k,(3+1.45\ln(1+\epsilon k))/\epsilon\}$.
After the prefix, each remaining own step is a play.
Hence $A_\epsilon$ is an executor with parameters $(\epsilon,\max_iW_i)$.
\end{lemma}

\begin{proof}
Couple the two policies on the same quality, tape and coins. Supplying
outcomes in tape order reproduces $\pi$'s decisions. Every supplied
outcome comes from a physical play already made, so the executor makes
at least the $m$ plays requested by $\pi$. If $\pi$ stops in a block
started after $q$ outcomes, it has requested at least $q+1$ plays and
leaves at most $b-1<\epsilon q$ unused plays, unless $b=1$, when
there is no excess. At the switch to delay-free mode, the unused buffer
is at most $\delta_i\le\epsilon q$. Unused block plays count toward a
committed tail, and every block respects the cap. Thus
$m\le X\le\min\{L_i,(1+\epsilon)m\}$. The latent-mean identity in
Lemma~\ref{lem:truncation} gives $R(A_\epsilon(\pi))\ge R(\pi)$.
The actions use only this arm's tape, coins and own progress.

After the singleton blocks, each completed block grows the number of
available outcomes by a factor of at least $1+\epsilon$. Counting
these blocks and their waits gives the stated prefix bounds; the
calculation is in Appendix~\ref{app:buffered-prefix}. After commitment
or the switch to delay-free mode, each remaining own step is a play.
\end{proof}

\begin{lemma}[Placement invariance]\label{lem:placement}
Fix arm $i$'s incoming state in episode $e$. Couple its standalone
and interleaved executions on the same qualities, tapes and coins,
using the same executor. If the interleaved execution completes
within the episode, both executions earn the same reward and retain
the same state at the start of episode $e+1$.
\end{lemma}

\begin{proof}
Condition~\ref{item:executor-interleaving} preserves the executor's
ordered plays and supplied observations under interleaving. The same
tape therefore gives the same rewards and retained data.
Under the report calendar of Definition~\ref{def:planned}, reports
from earlier episodes arrive only at boundaries, and any reports
still pending at the next boundary have delivery dates independent
of placement. The retained belief, remaining budget and report
queue consequently agree. The fixed within-episode reward law is
essential (Example~\ref{rem:within}).
\end{proof}

Fix an arm order for the episode. At each step, the
\emph{priority scheduler} plays the first executor requesting a
play, and waits if none does. Report delays elapse while other
arms are played.

\begin{lemma}[Handoff under completion]\label{lem:handoff}
Run arm programs $(\sigma_i)_i$ through an executor with parameters
$(\epsilon,w)$ under the priority scheduler, in any fixed order. Couple
this run with the $n$ independent standalone runs on the same qualities,
tapes and coins. Let $X_{e,i}$ be arm $i$'s standalone plays in episode
$e$ and $S_e=\sum_iX_{e,i}$. On the event
\[
 \mathcal E_{\le e}=\{S_u\le T-w\text{ for all }u\le e\},
\]
the two runs make the same plays of unknown arms, earn the same rewards
from them, and reach the same planned data, budgets and queues at the
start of episode $e+1$.
\end{lemma}

\begin{proof}
Induct on $e$. Suppose the runs agree at the start of episode $e$. Then
they choose the same local policies. Index the arms in priority order
and let $M_{e,i}$ be the plays the scheduler makes of arm $i$. Plays of
higher-priority arms do not depend on arm $i$, and there are at most
$\sum_{j<i}X_{e,j}$ of them. Arm $i$'s waiting steps advance in calendar
time, and a higher-priority play blocks it only when it requests a play. By
condition~\ref{item:executor-prefix}, after at most $w$ own steps, every unblocked step is a play until
the policy ends. Hence
\begin{equation}\label{eq:priority-plays}
 M_{e,i}\ge\min\Bigl\{X_{e,i},\Bigl(T-w-\sum_{j<i}X_{e,j}\Bigr)^+\Bigr\}.
\end{equation}
On $S_e\le T-w$ the right side equals $X_{e,i}$ for every $i$, whatever
the order, so every executed policy completes.
Lemma~\ref{lem:placement} then gives the same rewards and the same
boundary states.
\end{proof}

Completion preserves the standalone arm programs' data, budgets
and queues. In particular, extra executor plays earn reward, but their outcomes
are ignored for learning: using them would change the transition law
valued by the LP. The boundary-aligned calendar is essential to this
handoff: a report from an earlier episode arriving partway through a
local policy would insert a posterior update into its decision graph,
which Lemma~\ref{lem:placement} does not cover.

\subsection{The learner and its completion guarantee}\label{sec:alg1}\label{sec:selfverify}

\begin{algo}[The learning-LP learner]\label{alg:learner}
Choose an executor for each arm: buffered execution $A_\epsilon$ for delayed fast
arms, and direct execution otherwise. Let $w<T$ bound every prefix
($\max_iW_i\le w$), with $\epsilon=w=0$ when no fast arm is delayed.
Choose a reserve $0\le s<T-w$ and plan with
\[
 C=\frac{T-w-s}{1+\epsilon},\qquad
 B_i'=\begin{cases}
 B_i,&B_i\le1/\epsilon,\\
 \lfloor B_i/(1+\epsilon)\rfloor,&B_i>1/\epsilon,
 \end{cases}
\]
with $B'=B$ when $\epsilon=0$.
\begin{enumerate}
\item Complete a feasible solution of $\LP_N(C,B')$ and extract
its arm programs $(\sigma_i)_i$. Their
objective $F$ is the \emph{value certificate}.
\item In episode $e$, let $\pi_{e,i}$ be the local policy of $\sigma_i$
given arm $i$'s planned history. Execute it with the chosen executor and
cap $\min\{L_i,\text{remaining budget}\}$. Interleave all arms with the
priority scheduler in any fixed order. No admission coins are needed:
the reduced capacity already reserves time.
\item At each boundary, update $\sigma_i$ using only the reported
outcomes of $\pi_{e,i}$'s own plays. Pending outcomes stay in the queue.
Extra executor plays earn reward, but their outcomes are excluded from
the planned data.
\item If some local policy does not complete, the prediction fails.
Switch to any admissible recovery for later episodes, built from the
actual reported data, budgets and queues. Idling is
admissible; so is re-solving the learning LP from the actual state.
\end{enumerate}
\end{algo}

\begin{theorem}[Self-verification]\label{thm:meta}
Under Assumption~\ref{ass:drift}, Algorithm~\ref{alg:learner} respects
every cap and budget on every path. Let its arm programs have objective
$F$. Run them standalone through the executor, let $X_{e,i}$ be arm
$i$'s workload in episode $e$, and let $\mathcal E=\mathcal E_{\le N}$ be
the completion event of Lemma~\ref{lem:handoff}. Any certified bound
$q_{\mathrm{fail}}\ge\Pr(\mathcal E^c)$ on a failed prediction gives
\[
 \E\sum_eV_e\ge F
 -n\Lbar Nq_{\mathrm{fail}}.
\]
This holds for every admissible recovery after a failed prediction.
If no arm has positive cap, the learner waits and $F=0$.

For a uniform completion certificate, put
$\Lambda_{\mathrm{exec}}=\log(n\Lbar N^2)$ for $n\Lbar>0$ and choose
\[
 s=\sqrt{2\Lbar T\Lambda_{\mathrm{exec}}}
   +\tfrac23\Lbar\Lambda_{\mathrm{exec}}<T-w.
\]
Then the failure charge is at most $1$, and
\[
 \E\sum_eV_e\ge F-1.
\]
For an optimal solution, put
$\omega=1-C/T=((w+s)/T+\epsilon)/(1+\epsilon)$. Then
\[
 \E\sum_eV_e\ge(1-\omega)\OPT_N(B')-1.
\]
\end{theorem}

\begin{proof}
Direct execution uses the planned plays. Buffered execution inflates their
count by at most $1+\epsilon$. For $B_i>1/\epsilon$,
$(1+\epsilon)B_i'\le B_i$; for $B_i\le1/\epsilon$, it uses singleton
blocks and makes exactly the planned plays. These budget bounds
hold on every path. The executor respects each local cap, and recovery
respects the actual remaining resources. Run each arm program
standalone on its own qualities, tape and coins. These runs are
independent across arms and retain the planned data law. Write
$V_{e,i}^\circ$ for their rewards and
$S_e=\sum_iX_{e,i}$ for their physical plays. Then
\[
 \E\sum_iV_{e,i}^\circ\ge\sum_iR_e(\sigma_i),\qquad
 \E S_e\le(1+\epsilon)\sum_i\Work_e(\sigma_i)\le(1+\epsilon)C.
\]
Lemma~\ref{lem:handoff} couples the scheduled run to these runs on
$\mathcal E$. Put
$V_e^\circ=\sum_iV_{e,i}^\circ$. Its expectation sums to at least
$F$, and $V_e^\circ\le S_e\le n\Lbar$.
The scheduled reward agrees with it on $\mathcal E$ and is
nonnegative off it. Hence
\[
 \E\sum_eV_e\ge F
 -n\Lbar N\Pr(\mathcal E^c).
\]
This proves the certificate bound for every admissible recovery.

For each episode
$\E S_e\le(1+\epsilon)C=T-w-s$. 
Put $\mu_e=\E S_e\le T-w-s$. The standalone workloads are
independent across arms, with
$\sum_i\operatorname{Var}(X_{e,i})\le\Lbar\mu_e$.
Bernstein's inequality gives
\[
 \Pr(S_e>T-w)
 \le \exp\left\{-\frac{s^2}
 {2\Lbar\mu_e+(2/3)\Lbar s}\right\}
 \le e^{-\Lambda_{\mathrm{exec}}}.
\]
A union bound yields
$\Pr(\mathcal E^c)\le Ne^{-\Lambda_{\mathrm{exec}}}$ and failure charge
at most $1$. Finally,
$F=\LP_N(C,B')\ge(C/T)\OPT_N(B')$ by
Proposition~\ref{prop:relax},
$C/T=1-\omega\ge1-(w+s)/T-\epsilon$.
These give the optimal-solution bound.
\end{proof}

For $\Lbar=T/X$, the capacity loss tends to zero when
$\Lambda_{\mathrm{exec}}=o(X)$, $w=o(T)$ and $\epsilon=o(1)$.
The following lower bound shows that the leading reserve is
unavoidable in the worst case when every prescribed experiment
must complete.

\paragraph{A matching lower bound for full completion.}\label{sec:reserve-lower}
\begin{theorem}[Necessary reserve for full completion, proved in Appendix~\ref{app:reserve-tail}]
\label{thm:reserve-lower}
Fix integers $m\ge2$, $n\ge m$ and $N\ge1$. There are $n$ independent
undelayed one-shot arms with cap $\Lbar=L=m^2$, known stationary priors and
nonbinding budgets $B_i=NL$.
For planning capacity $C=nm$ and episode horizon $T=C+s$, where
$s\ge0$ is an integer, every optimal solution of $\LP_N(C,B)$
prescribes, almost surely, one probe of each arm and continuation to
its cap after success, stopping after failure.

Let $s_{\min}$ be the smallest integer reserve for which all these
policies can complete in all $N$ episodes with probability at least
$1-\delta$, where $0<\delta\le1/2$. Put
\[
 h=\log(N/\delta),\qquad v=\frac{nm}{(m+1)^2}.
\]
Along any sequence with $h\to\infty$ and $h=o(v)$,
\begin{equation}\label{eq:necessary-reserve}
 s_{\min}=(1+o(1))(m-1)\sqrt{2nm\log(N/\delta)}.
\end{equation}
At $T=C+s_{\min}$ this is
\[
 s_{\min}=(1+o(1))\left(1-\frac1{\sqrt L}\right)
 \sqrt{2LT\log(N/\delta)}.
\]
In particular, if also $m\to\infty$, the leading constant is
$\sqrt2$, matching the uniform reserve of Theorem~\ref{thm:meta}.
\end{theorem}

For the failure budget $\delta=1/(n\Lbar N)$ used above,
$\log(N/\delta)=\Lambda_{\mathrm{exec}}$. Along this construction,
when $\Lambda_{\mathrm{exec}}=o(T/\Lbar)$, the lower bound is
$1-1/\sqrt{\Lbar}$ times the leading uniform reserve, up to a
$1+o(1)$ factor. The leading constants match when also
$\Lbar\to\infty$.

The lower-bound construction also admits a feasible policy that collects every informative
observation and approaches the LP value without a reserve.
Each arm's episode mean is $1$ with probability $p=1/(m+1)$
and $0$ otherwise. The first play reveals it; later plays supply
no further information. If
$Z\sim\operatorname{Binomial}(n,p)$, the prescribed work is
$S=n+(L-1)Z$, with $\E S=C=nm$, and the LP's per-episode value is
exactly $npL$, at both capacities $C$ and $T=C+s$. Probe all $n$ arms
first, then spend the remaining steps on successful arms. This
feasible policy earns
\[
 V_{\mathrm{probe}}=npL-\E(S-T)^+.
\]
Even with no reserve, $\operatorname{Var}(S)=(m-1)^2nm$ gives
\[
 \frac{V_{\mathrm{probe}}}{npL}
 \ge1-\frac{1-m^{-2}}2\sqrt{\frac mn},
\]
because $\E(S-C)^+=\tfrac12\E|S-C|\le\tfrac12\sqrt{\operatorname{Var}(S)}$.
The lower-bound regime forces $n/m\to\infty$, so this ratio tends to
$1$. Protecting the first reports preserves all the information while
allowing reward-only tails to be truncated. We next bound that reward
loss, then extend this separation to learning across episodes.

\subsection{Single-episode execution and the effect of caps}\label{sec:constants}\label{sec:plan}\label{app:known}\label{app:known-execution}

The remaining ingredient bounds reward lost at the horizon when
continuation has already been protected. For known, independent recurring
priors $P=(P_1,\ldots,P_n)$, write $\OPT_\delta(P)$ and $\OPT(P)$ for the delayed and
undelayed single-episode optima.

For an undelayed single-arm policy $\pi_i$ of cap $L_i$, write
$R(\pi_i)$ and $\Work(\pi_i)$ for its expected reward and plays.
GM13's relaxation is
\[
  \LP(P)=\max_{(\pi_i)}\quad\sum_iR(\pi_i)\qquad
 \text{subject to}\quad\sum_i\Work(\pi_i)\le T.
\]
Every local cap holds on every path. Shared time holds only in
expectation. The local policies may require more than $T$
plays on a realized joint path.
Write $\LP^{(k)}(P)$ for this relaxation restricted to
depth-$k$ commit policies.

Projecting a feasible policy onto its arms, as in
Proposition~\ref{prop:relax}(i), and relaxing all within-episode
delays gives
\begin{equation}\label{eq:episode-upper}
 \OPT_\delta(P)\le\LP(P),\qquad \OPT(P)\le\LP(P).
\end{equation}
The projection preserves reward, workload and caps and ignores
outcomes until the original policy would receive them
\citep[Claim~1]{GM13}.

\begin{theorem}[Cap-sensitive undelayed execution, proved in Appendix~\ref{app:cap-execution}]\label{thm:budget}
Fix local policies with independent coins, caps $L_i\le T/X$ for
$X\ge1$, and total expected workload at most $T$. Let their value be
$
 J=\sum_iR(\pi_i)
$.
Execute them in decreasing order of expected reward per expected
play, and idle at unused steps. Truncate at the horizon.
For integer $X$, let $Z_X\sim\operatorname{Poisson}(X)$.
This undelayed policy earns at least $c_XJ$, where 
\[
c_X = \max\Bigl\{\tfrac12,\;1-\Pr(Z_X=X) - \Pr(Z_X\ge X)/(2X)\Bigr\},
\]
and $1/c_X=1+(2\pi X)^{-1/2}+O(X^{-1})$.
\end{theorem}

With the safe action of Section~\ref{sec:weights}, fill unused steps
with it: its expected reward is at least
$a(T-\sum_i\Work(\pi_i))$, so the total is at least
$c_X\sum_iR(\pi_i)+a(T-\sum_i\Work(\pi_i))
\ge c_X[aT+\sum_i(R(\pi_i)-a\Work(\pi_i))]$, since $c_X\le1$.

The order is fixed before execution. For any arm, the time consumed
by higher-priority arms is independent of its tape. The scheduling
bound \eqref{eq:priority-plays} and Lemma~\ref{lem:truncation} turn
the available time into a reward guarantee. The total workload is a
sum of independent variables bounded by their caps. A Poisson
comparison bounds its excess and gives $c_X$; the calculation is in
Appendix~\ref{app:cap-execution}. Thus smaller caps improve retention
of reward that may be cut at the horizon.

\begin{corollary}[Delayed execution, proved in Appendix~\ref{app:cap-execution}]\label{cor:delayed}
Take the feasible local policies of Theorem~\ref{thm:budget}. Choose
$\epsilon\in(0,1]$ and $\gamma\in(0,1)$ with
$\max_iW_i\le\gamma T$. Admit each transformed policy independently
with probability $(1-\gamma)/(1+\epsilon)$. Order admitted policies by
expected reward per expected play. At each step, run the
highest-priority policy requesting a play, otherwise idle.
The policy earns at least $\kappa J$, where
\[
 \kappa=\frac{1-\gamma}{1+\epsilon}\,c_{X,\gamma}.
\]
The finite-cap factor $c_{X,\gamma}$ is given in
\eqref{eq:finite-retention} in Appendix~\ref{app:cap-execution}.
It is at least $(1-\gamma)/2$. For fixed integer $X$,
$c_{X,\gamma}\to c_X$ as $\gamma\to0$.
\end{corollary}

When no report arrives within an episode, local policies are nonadaptive.
Example~\ref{ex:slow} gives the exact planning solution and a
$T/16$ loss for sampled-mean Thompson sampling even with the prior known.

GM13 retains half the LP value for general caps. Small caps
concentrate the independent workloads and improve the factor. The
bound applies to every feasible solution.
For noninteger $X$, use $c_{X,0}$ from
\eqref{eq:finite-retention}. The leading square-root loss is necessary:

\begin{proposition}[The square-root loss is necessary, proved in Appendix~\ref{app:cap-execution}]\label{prop:gap}
For every integer $X\ge1$, there are undelayed one-shot instances
with $L_i=T/X$ such that
\[
 \frac{\OPT(P)}{\LP(P)}
 \le1-\Pr\{\operatorname{Poisson}(X)=X\}+o(1)
\]
as $T\to\infty$ through multiples of $X$.
\end{proposition}

\begin{proposition}[The factor $2$ is tight, proved in Appendix~\ref{app:cap-execution}]\label{prop:gap-two}
For every $T\ge2$, there is an undelayed one-shot instance with
$L_i=T$ such that $\OPT(P)/\LP(P)\le1/2+1/T$.
\end{proposition}

At $X=1$ the factor $1/2$ is therefore exact for this relaxation. In
Proposition~\ref{prop:gap} the loss comes from successes competing for
time. Here it comes from the time spent finding one, which the single
time constraint does not see.

If every unknown arm is fast with delay at most $\delta$, put
$D_\delta=(\delta+1)(1+\ln(\delta+1))$. The additional delay loss is
$O(\sqrt{D_\delta/T})$ at fixed $X$, and the factor for arbitrary caps
tends to $2$ when $\delta\log\delta=o(T)$
(Corollary~\ref{cor:delay-regret}). Appendix~\ref{app:cap-execution}
gives the parameter calculation and the direct comparison with GM13.
These bounds price reward truncation. Full completion in
Theorem~\ref{thm:meta} instead pays the reserve and executor overhead.

\paragraph{Several plays per step.}
With up to $K\ge1$ distinct arms played per step, each arm still makes
at most one play per step. Write $\LP_K(P)$ for the same relaxation
with total expected workload at most $KT$; projecting a feasible
policy gives $\OPT_K(P)\le\LP_K(P)$.

\begin{theorem}[Fixed-order execution with $K$ plays per step, proved in Appendix~\ref{app:kplays}]\label{thm:k-lanes}
Fix undelayed local policies with independent coins, caps
$L_i\le L\le T$, value $J=\sum_iR(\pi_i)$ and expected workload
$W=\sum_i\Work(\pi_i)\le KT$. Order them by decreasing expected
reward per expected play. At each step play the at most $K$
highest-priority policies that have not stopped, and truncate at $T$.
With $\nu=T/L$, the expected reward is at least
\[
 \max\left\{1-\frac{W}{2KT},\ \rho_{P,K}(\nu)\right\}J,
\]
where, for $Z\sim\operatorname{Poisson}(K\nu)$,
\begin{equation}\label{eq:k-retention}
 \rho_{P,K}(\nu)=1-\frac{\E H_\nu(Z/K)}{\nu},\qquad
 H_\nu(x)=\int_0^x\min\{1,(y-\nu+1)^+\}\,dy.
\end{equation}
For $K=1$ and integer $\nu=X$, the uniform factor is $c_X$.
\end{theorem}

The uniform factor $\max\{1/2,\rho_{P,K}(\nu)\}$ is nondecreasing
in $K$ and tends to $1-L/(2T)$. Example~\ref{ex:k-finding-time}
shows that this limit is tight relative to $\LP_K$; at $L=T$,
the worst-case factor is $1/2$ for every $K$.
These are undelayed single-episode bounds.

\subsection{Predicting less}\label{sec:promise-less}

We now preserve the planned learning transitions on every path.
Each arm program learns from a designated prefix of its episode's
plays. Execution completes these plays in a fixed window and may
truncate the remaining reward plays. The program ignores their
outcomes and charges their planned count to its budget ledger,
so truncation does not change the state from which it continues.

\begin{definition}[Prediction-limited local policy]\label{def:prediction-limited}
Fix an integer $r\ge0$. A local policy $\pi$ is
\emph{prediction-limited} at depth $r$ if it commits to its remaining
play count after at most $r$ adaptive plays; for a slow arm, the count
is fixed at the episode's start. If $N_\pi$ is its total play count,
the first $\min\{N_\pi,r\}$ plays are \emph{predicted} and the rest
are \emph{unpredicted}. Its next boundary state uses predicted reports
only, and its budget ledger deducts every planned play.
\end{definition}

Fix depths $r_{e,i}$ in advance. A \emph{prediction-limited arm
program} chooses its local policies from this class using its own
retained reports. Choose executors with parameters $(\epsilon,w)$
and comparison budgets $B'$ as in Algorithm~\ref{alg:learner}.
Choose a window $T_{\mathrm p}<T$ satisfying
$w+(1+\epsilon)\sum_i r_{e,i}\le T_{\mathrm p}$ in every episode,
and put $C_2=T-T_{\mathrm p}$. The restricted learning LP is
\[
 \LP_N^{[r]}(C_2,B)=\sup_{(\sigma_i)}\sum_{e,i}R_e(\sigma_i)
 \quad\text{s.t.}\quad\sum_i\Work^{\mathrm u}_e(\sigma_i)\le C_2
 \quad(e=1,\ldots,N),
\]
over these arm programs with budgets $B$. Here
$\Work^{\mathrm u}_e(\sigma_i)$ is arm $i$'s expected number of
unpredicted plays in episode $e$. The window accommodates the
predicted plays and executor overhead on every path; the constraint
limits expected unpredicted work to the remaining $C_2$ steps.

Each commitment specifies an expected reward, a workload and a
distribution of retained next states. Its probability weights these
quantities linearly. Pricing the workload constraint at $\lambda_e$
separates the arms: each local program maximizes current reward plus
continuation, less $\lambda_e$ times its unpredicted workload.

\begin{algo}[Predict less]\label{alg:promiseless}
Use the depths, window, executors and comparison budgets chosen above.
\begin{enumerate}
\item Complete a feasible solution of $\LP_N^{[r]}(C_2,B')$ and
extract its arm programs $(\sigma_i)_i$. Let $F$ be their objective.
\item \emph{Phase 1}, steps $1$ to $T_{\mathrm p}$. Run the predicted
plays through their executors under the priority scheduler. Wait
when no executor requests a play.
\item \emph{Phase 2}, steps $T_{\mathrm p}+1$ to $T$. Run the remaining
unpredicted plays arm by arm, in an order fixed before the run by
decreasing expected reward per expected play of their standalone blocks.
Wait when all blocks have ended, and truncate at $T$.
\item At each boundary, update each arm program from predicted
reports, retain those still pending, and deduct every planned play
from its budget ledger.
\end{enumerate}
Buffered execution may make unpredicted plays during phase~1;
phase~2 plays the rest.
\end{algo}

\begin{theorem}[Predict less]\label{thm:promise-less}
Under Assumption~\ref{ass:drift}, Algorithm~\ref{alg:promiseless}
respects every cap and budget on every path and preserves the planned
law of the retained learning states. Every predicted play is made in
phase~1. Put $X_2=C_2/\Lbar\ge1$ and
$c=c_{X_2,0}$. Then
\[
 \E\sum_eV_e\ge F-(1-c)\sum_{e,i}R^{\mathrm u}_e(\sigma_i)
 \ge cF,
\]
where $R^{\mathrm u}_e(\sigma_i)$ is the expected reward of arm $i$'s
unpredicted plays in episode $e$.
\end{theorem}

\begin{proof}
Let $X^{\mathrm p}_{e,i}$ be arm $i$'s standalone physical plays in
phase~1. A block is requested only before $r_{e,i}$ outcomes have
been supplied, so $X^{\mathrm p}_{e,i}\le(1+\epsilon)r_{e,i}$.
The window condition and \eqref{eq:priority-plays} therefore make
every predicted prefix complete by $T_{\mathrm p}$ on every path.
Couple to the standalone prefixes on the same qualities, tapes and
coins. Lemma~\ref{lem:placement} gives the same predicted reports,
committed tail counts and pending queue. The ledger deducts every
planned play, including truncated ones, so the retained states have
the planned law in every episode. The physical budget suffices by
the executor allowance in $B'$, and every executor respects its cap.

The remaining tails are nonadaptive capped blocks, independent
across arms in their standalone runs, with total expected work at
most $C_2$. The single-episode execution bound retains at least
$c$ of their expected reward. Unpredicted plays made early by the
executor are all kept. Write
$J_e=\sum_i(R^{\mathrm p}_{e,i}+R^{\mathrm u}_{e,i})$ for the
planned episode reward. Predicted plays retain their reward, hence
\[
 \E V_e\ge J_e-(1-c)\sum_iR^{\mathrm u}_{e,i}.
\]
The retained states have the planned law, so $\sum_eJ_e=F$.
Summing and using $\sum_{e,i}R^{\mathrm u}_{e,i}\le F$ proves the
claim. Continuation is preserved throughout; only unpredicted
reward pays the single-episode factor.
\end{proof}

For one-shot arms, the first report contains all the information
available from that arm in the episode.

\begin{corollary}[One-shot arms]\label{cor:oneshot}
Suppose every episode mean lies in $\{0,1\}$. Take $r_{e,i}=1$ and
$T_{\mathrm p}=w+\lceil(1+\epsilon)n\rceil$ with
$(T-T_{\mathrm p})/\Lbar\ge1$, and use an optimal restricted LP
solution. Then
\[
 \E\sum_eV_e\ge c\Bigl(1-\frac{T_{\mathrm p}}T\Bigr)\OPT_N(B').
\]
Without delayed fast arms, $T_{\mathrm p}=n$, $\epsilon=w=0$ and $B'=B$.
\end{corollary}

\begin{proof}
After its first report, a fast arm's outcomes are determined.
Replacing its local policy by depth-$1$ commit form preserves
reward, workload and the next learning state; slow arms are already
nonadaptive. In particular, $\LP^{(1)}(P)=\LP(P)$ for a single episode.
Retain each arm program independently with probability
$1-T_{\mathrm p}/T$. The total expected unpredicted workload is then
at most $T-T_{\mathrm p}$, so
\[
 \LP_N^{[1]}(T-T_{\mathrm p},B')\ge
 \Bigl(1-\frac{T_{\mathrm p}}T\Bigr)\LP_N(T,B')
 \ge\Bigl(1-\frac{T_{\mathrm p}}T\Bigr)\OPT_N(B').
\]
Apply Theorem~\ref{thm:promise-less}.
\end{proof}

Allocating the remaining plays by conditional mean gives a sharper
factor (Proposition~\ref{prop:greedy-phase2}); it is asymptotically
sharp for one-shot arms whose probes take $o(T)$ steps
(Corollary~\ref{cor:short-probe-gap}).

\subsection{Objectives beyond expected reward}\label{sec:beyond}

A fixed bounded score $\phi$ of a local policy's plays and outcomes
enters as a payoff coefficient. Use equality in the reach constraints
of \eqref{eq:local-flows} and \eqref{eq:lp-state}, so every outcome
receives its score; the node multipliers are then unrestricted in
sign. With the history weights of \eqref{eq:history-weights},
\begin{equation}\label{eq:score-flow}
 \E^P\phi=\sum_h(A(h)-B(h))\ell(P;h)\phi(h).
\end{equation}
For general observation spaces, integrate against the terminal-flow
measure. Downside penalties, failure indicators and squared loss
against a fixed target are examples. Scores additive across arms
preserve the dual decomposition. A score of their joint return
generally requires joint histories. The state must retain any running
reward total, or its conditional law while reports remain pending,
needed to evaluate the score.

A risk functional acts on the resulting distribution.
Mean--variance utility is convex in trajectory probabilities, so its
maximization is not an LP and is hard in general
\citep{mannor2011meanvariance}. CVaR has a threshold representation
whose fixed-threshold payoff is linear \citep{rockafellar2000cvar};
the threshold is an outer optimization parameter. Risk of joint
return can still couple the arms.

Direct completion preserves the scored trajectory law, with failures
charged by a bound on the score difference. Reward truncation also
applies to fixed nonnegative per-play scores; a general terminal
score or risk functional needs a corresponding execution bound.
The reward factors above therefore do not automatically extend to
every distributional objective. Example~\ref{ex:risk} shows why
conditioning on a sampled quality can discard variation decisive
for mean--variance, CVaR and downside decisions. Related
mean--variance and CVaR Thompson-sampling formulations already exist
\citep{zhu2020mvts,baudry2021cvarts}; the example concerns
conditioning on a sampled episode quality.

\section{Learning-value stability and verified execution}\label{sec:rates}\label{sec:rates-discussion}

An error in the experiment law changes both reward and expected
workload. The episode dual bounds these changes together. We then
transfer the comparison through the reports preserved by execution.

\subsection{Stability of the learning value}

The quantitative comparisons in this subsection assume finite reward
alphabets $\mathcal Y_i$, of sizes $d_i$. The model and execution
results allow general bounded reward laws. Fix adaptation depths $k_i$,
with $k_i=0$ for slow arms, and write $\LP^{(k)}$ for these arm-specific
classes, with $k=\max_i k_i$. Put $r_i=\min\{L_i,k_i+1\}$. We need the law of
an $r_i$-play experiment to compare its reward and workload.

For $r$ nonadaptive plays of arm $i$, let
$\boldsymbol s=(s_y)_{y\in\mathcal Y_i}$ be the outcome histogram,
with $|\boldsymbol s|=r$. Its law and distance are
\begin{equation}\label{eq:experiment-distance}
 \mathbf q_r(P)=\left(\binom r{\boldsymbol s} \int\prod_{y\in\mathcal
 Y_i}K_i(\theta,\{y\})^{s_y}P(d\theta) \right)_{|\boldsymbol s|=r},\qquad
 d_r(P,P')=\tfrac12\|\mathbf q_r(P)-\mathbf q_r(P')\|_1.
\end{equation}
There are $D_i=\binom{r_i+d_i-1}{d_i-1}$ possible histograms.
Conditional on a histogram, exchangeability makes the ordered outcomes
a uniform permutation under every prior. The ordered tape has the
same total variation distance. Write
$\|x-x'\|_{\rm TV}=\frac12\|x-x'\|_1$ also for arbitrary vectors.

\begin{lemma}[Reward and workload stability]\label{lem:value}
For a depth-$k_i$ commit policy of cap $L_i$, its reward and workload
each change by at most $L_i d_{r_i}(P,P')$ when $P$ is replaced by $P'$.
\end{lemma}

\begin{proof}
Pre-sample $r_i$ outcomes and the policy's coins. These determine its
adaptive prefix, commitment and workload. If it commits to $b$ further
plays at depth $j$, represent their reward by $b$ times the next unused
outcome. Conditional on the quality and prefix, this preserves expected
reward. The surrogate reward and workload both lie in $[0,L_i]$ and
use at most $r_i$ outcomes. Average over coins and the uniform ordering
given the histogram, and apply the range-times-total-variation bound.
\end{proof}

\begin{theorem}[Stability between fixed experiment laws]\label{prop:lawfree-stability}
For two product priors $P,P'$ and the same caps and adaptation
depths, put $D_{\rm exp}=\sum_iL_i d_{r_i}(P_i,P_i')$. Then
\[
 |\LP^{(k)}(P)-\LP^{(k)}(P')|\le D_{\rm exp}.
\]
The constant $1$ is sharp.
\end{theorem}

\begin{proof}
For a fixed local policy, the surrogate reward in
Lemma~\ref{lem:value} satisfies
$0\le\widetilde R_\pi\le W_\pi\le L_i$ on each sampled tape and
coin sequence. For $0\le\lambda\le1$,
\[
 -\lambda L_i\le\widetilde R_\pi-\lambda W_\pi
 \le(1-\lambda)L_i.
\]
The interval has length $L_i$. The range-times-total-variation bound
therefore controls the change in the policy's priced value by
$L_i d_{r_i}(P_i,P_i')$. Maximizing over policies preserves this
bound. The episode dual is
\[
 \LP^{(k)}(P)=\min_{0\le\lambda\le1}
 \left\{\lambda T+\sum_i\sup_{\pi_i}
 \E_{P_i}[\widetilde R_{\pi_i}-\lambda W_{\pi_i}]\right\}.
\]
Prices above $1$ are unnecessary: all priced rewards are then
nonpositive, and each arm can make no plays. Summing the arm bounds
and minimizing over the common price proves the claim. One arm with
cap $T$, giving identically $0$ under one prior and identically $1$
under the other, attains equality.
\end{proof}

\begin{corollary}[Learning-value stability]\label{thm:learning-stability}
Fix an episode and available history $H$, with remaining budgets at
least the supplied caps. Let $P=(P_i)_i$ be the current recurring
priors and $\Pb$ their conditional predictive priors. Then
\begin{equation}\label{eq:posterior-spread-stability}
 \LP^{(k)}(\Pb)\ge\E[\LP^{(k)}(P)\mid H]
 -\sum_iL_i\E[d_{r_i}(P_i,\Pb_i)\mid H].
\end{equation}
For any $H$-measurable vectors $\hat{\mathbf q}_i\in\mathbb R^{D_i}$,
\begin{equation}\label{eq:learning-stability}
 \LP^{(k)}(\Pb)\ge\E[\LP^{(k)}(P)\mid H]
 -2\sum_iL_i\E\!\left[
 \|\mathbf q_{r_i}(P_i)-\hat{\mathbf q}_i\|_{\rm TV}\mid H\right].
\end{equation}
The vectors need not be realizable experiment laws.
\end{corollary}
\begin{proof}
Apply Theorem~\ref{prop:lawfree-stability} with $P'=\Pb$ and average
over $P$ given $H$. Since
$\mathbf q_{r_i}(\Pb_i)=\E[\mathbf q_{r_i}(P_i)\mid H]$,
the triangle inequality and Jensen's inequality bound the posterior
spread by twice the error about any $H$-measurable center.
\end{proof}

\subsection{Probes and verified learning}

A current-episode solution gives an immediate reward floor.

\begin{theorem}[Per-episode floor]\label{thm:floor}
Fix episode $e$, assume every remaining budget is at least its cap,
and take a feasible solution of $\LP^{(k)}(\Pb_e)$ of value at least
$(1-\varepsilon_{\rm opt})\LP^{(k)}(\Pb_e)$, where $0\le\varepsilon_{\rm opt}<1$.
Under the conditions of Theorem~\ref{thm:budget} or
Corollary~\ref{cor:delayed}, with execution factor $\kappa$,
\[
 \E[V_e\mid H_{e-1}]\ge\kappa(1-\varepsilon_{\rm opt})\LP^{(k)}(\Pb_e).
\]
\end{theorem}

\begin{proof}
Conditional on $H_{e-1}$, Lemma~\ref{lem:posterior} gives the
independent predictive priors $\Pb_e$. Apply the single-episode
execution guarantee to the supplied solution.
\end{proof}

The following policy acquires observations through ordinary plays and
plans one episode at a time. It supplies both an executable learner
and a reward comparison for the cumulative learning LP.

\begin{algo}[Probing with episode LPs]\label{alg:probe}
Partition the arms into $J_{\mathrm{grp}}$ probe groups with
$\sum_{i\in\mathrm{group}}r_i\le T$ in each group. One can take
$J_{\mathrm{grp}}\le1+2\sum_i r_i/T$. If all probes fit, use one
group. Fix $\xi\in(0,1)$. In each episode:
\begin{enumerate}
\item With probability $1-\xi$, execute an optimal
$\LP^{(k)}(\Pb_e)$ solution with the known-prior guarantee of
Section~\ref{sec:plan}.
\item Otherwise, choose a group uniformly. Make $r_i$ nonadaptive plays
of each arm in that group, and idle at the remaining steps.
\item Update the posterior when reports arrive. Use reported
outcomes from both branches.
\end{enumerate}
\end{algo}

Policy weights cancel from the posterior
(Proposition~\ref{prop:history}), so the update uses reports from both
branches. Each arm is probed with probability
$p=\xi/J_{\rm grp}$ per episode, independently of its qualities. With
the histogram dimensions $D_i$ from \eqref{eq:experiment-distance}, put
\begin{equation}\label{eq:estimation-coefficient}
  L_\Sigma=\sum_i L_i,\qquad C_{\mathrm{est}}=\sqrt{J_{\mathrm{grp}}}\sum_i
 L_i\sqrt{D_i}.
\end{equation}

The comparison need retain only the reports protected by execution.
Its estimation error then transfers to the cumulative learning LP.

\begin{theorem}[Learning comparison through protected reports]\label{thm:verified-learning}
Assume finite reward alphabets, fixed caps and depths, and nonbinding
comparison budgets $B_i'\ge NL_i$. Choose prediction depths
$r_{e,i}=r_i$ and a window satisfying
$w+(1+\epsilon)\sum_i r_i\le T_{\mathrm p}<T$, with
$X_2=(T-T_{\mathrm p})/\Lbar\ge1$. Put
\[
 c=c_{X_2,0},\qquad
 \kappa_{\rm pl}=c(1-T_{\mathrm p}/T).
\]
Fix $\xi\in(0,1)$. In the planned comparison, every arm is probed for
$r_i$ plays with probability $\xi$; otherwise the arms follow an
optimal current-episode LP at capacity $T-T_{\mathrm p}$. Let $H_{e-1}$ contain
its reported predicted outcomes, and take any $H_{e-1}$-measurable
estimates $\hat{\mathbf q}_{e,i}$. Define
\[
 \mathcal E_N(\xi)=\sum_{e=1}^N\sum_iL_i\E
 \|\mathbf q_{r_i}(P_i^{(e)})-\hat{\mathbf q}_{e,i}\|_{\rm TV}.
\]
If Algorithm~\ref{alg:promiseless} uses a feasible restricted LP
solution within $\zeta\ge0$ of its optimum, then
\begin{equation}\label{eq:verified-learning}
 V_A\ge\kappa_{\rm pl}(1-\xi)
 \left\{\sum_e\E\LP^{(k)}(P^{(e)})-2\mathcal E_N(\xi)\right\}
 -c\zeta.
\end{equation}
\end{theorem}

\begin{proof}
The comparison learns only from the first $r_i$ outcomes of each arm
in an episode. Its depth-$k_i$ policies therefore form
prediction-limited arm programs. Project them onto the arms as in
Proposition~\ref{prop:relax}(i). Their unpredicted workload fits
$T-T_{\mathrm p}$ in expectation, so they are feasible for
$\LP_N^{[r]}(T-T_{\mathrm p},B')$.

Thinning a current-episode solution from capacity $T$ to $T-T_{\mathrm p}$
loses at most the factor $1-T_{\mathrm p}/T$. Probe rewards are nonnegative.
If $F^*$ is the restricted learning-LP optimum, then
\[
 F^*\ge(1-\xi)(1-T_{\mathrm p}/T)\sum_e\E\LP^{(k)}(\Pb_e),
\]
where $\Pb_e$ conditions on the retained predicted reports. Apply
Corollary~\ref{thm:learning-stability} to this history and sum.
Theorem~\ref{thm:promise-less} gives $V_A\ge cF$ for the chosen value $F$.
Finally $F\ge F^*-\zeta$, which proves the result.
\end{proof}

\begin{corollary}[Stationary learning with episode LPs, proved in Appendix~\ref{app:stationary-comparisons}]\label{cor:episode-lp-rate}
Assume stationary independent arms, finite reward alphabets, report
lag at most $g$, fixed caps and depths, and nonbinding budgets
$B_i\ge NL_i$. Let $\kappa$ be the execution factor in
Theorem~\ref{thm:floor}. Algorithm~\ref{alg:probe}, with
\[
 \xi=\min\left\{\frac12,
       \left(\frac{C_{\rm est}^2}{T^2N}\right)^{1/3}\right\},
\]
satisfies
\begin{equation}\label{eq:episode-lp-rate}
 \E\sum_{e=1}^N\bigl[\kappa\LP^{(k)}(P)-V_e^{\rm probe}\bigr]
 \le5\kappa(TC_{\rm est}^2)^{1/3}N^{2/3}
       +2\kappa(g-1)L_\Sigma.
\end{equation}
Here $P$ is the stationary recurring-prior vector, and expectation
averages under $\Q$ and the policy's coins. Each planning episode
optimizes a single-episode LP over arm programs.
Under the conditions of Theorem~\ref{thm:verified-learning},
the same bound holds for Algorithm~\ref{alg:promiseless}, with $\kappa$
replaced by $\kappa_{\rm pl}$. A planning gap $\zeta$ adds $c\zeta$.
\end{corollary}

The first bound uses short probes and episode LPs; it does not
require a joint Bellman planner. The protected learner transfers the
same reported evidence through its retained states.
Appendix~\ref{app:moving} gives the estimator comparison
(Theorem~\ref{thm:risk}), its specializations to slow drift and
known affine evolution (Theorems~\ref{thm:variation} and~\ref{thm:affine}),
and its transfer through completion and protected reports
(Corollaries~\ref{cor:fullscale} and~\ref{cor:probe-transfer}).
Posterior sampling gives a separate square-root episode-LP comparison
(Theorem~\ref{thm:ps-episode-lp} in Appendix~\ref{app:oracle-price});
its information-price consequence is stated in Section~\ref{sec:cost-content}.

\section{Pricing and corrections}\label{sec:programs}\label{sec:cost}\label{sec:benchmarks}

An observation is worth the decisions it improves, including those
made after further episodes of learning. We first price that benefit
in the learning LP, then use execution certificates to bound its
feasible value. The same certificates distinguish allocation loss
from discrepancies in the model's reported predictions.

\begin{definition}[Information benchmarks]\label{def:benchmarks}
$V_{\mathrm{learn}}=\OPT_N(B)$ is the optimum with the learner's
information. $V_{\mathrm{prior}}$ is the $\Q$-average optimum when the
recurring priors are revealed before episode~$1$.
$V_{\mathrm{state}}$ also reveals each episode's qualities at its
start; future qualities remain hidden. All three use the same
budgets, caps and report calendar. If purchased reports are available,
they use the same menu and subtract fees from reward.
\end{definition}

Extra information can be ignored, so
$V_{\mathrm{state}}\ge V_{\mathrm{prior}}\ge V_{\mathrm{learn}}$.
Thus
\begin{equation}\label{eq:knowledge-gaps}
 V_{\mathrm{state}}-V_{\mathrm{learn}}
 =(V_{\mathrm{state}}-V_{\mathrm{prior}})
 +(V_{\mathrm{prior}}-V_{\mathrm{learn}})
\end{equation}
separates the fresh-draw uncertainty from the price of learning the
recurring priors.

\subsection{Information at capacity prices}

Fix arm $i$ at the start of episode $e$, its remaining budget and
report queue, and the capacity prices $\lambda$. Let $\chi$ be its
complete hidden record of recurring priors, qualities and potential
outcomes, and let $Q$ be the conditional law of that record.
At the model's incoming state, $Q$ is induced by $\Q_{e,i}$ and the
outcome kernels. For a history-based program $\sigma$, let
$g_\sigma^\lambda(\chi)$ be its remaining reward less the capacity
prices of its plays, averaged over its independent coins. Then
\begin{equation}\label{eq:arm-value-envelope}
 U^\lambda(Q)=\sup_\sigma\langle Q,g_\sigma^\lambda\rangle.
\end{equation}
We suppress the fixed physical state and the indices $e,i$.
The programs use the learning LP's information: immediate feedback
for fast arms, nonadaptive play for slow arms, and the specified
boundary calendar. Programs subject to actual within-episode delays
form a smaller class. A program that acquires information over
$N'$ episodes contributes its entire reward and workload vectors;
these costs are already included in \eqref{eq:arm-value-envelope}.

The value is convex in $Q$, as a supremum of linear functions
\citep{smallwood1973finite}. Let $g_*$ attain the supremum at $Q$ and put
$D_{U^\lambda}(Q',Q;g_*)=
U^\lambda(Q')-U^\lambda(Q)-\langle Q'-Q,g_*\rangle\ge0$.
For an instantaneous report experiment $Y$ with posterior $Q_Y$,
\begin{equation}\label{eq:geometric-information-price}
 p_i^\lambda(Y)=\E U^\lambda(Q_Y)-U^\lambda(Q)
 =\E D_{U^\lambda}(Q_Y,Q;g_*).
\end{equation}
The supporting inequality follows by retaining the old program, and
the equality follows from $\E Q_Y=Q$. This is the classical value of
information at fixed prices. Its link to the coupled plan requires
reoptimizing those prices.

\begin{theorem}[Information value and repricing]\label{thm:geometric-price}\label{thm:report-repricing}
Take a finite learning LP at a fixed incoming boundary state. Let
$\LP_0$ be its value without an additional report and $\LP_Y$ its
value when arm $i$ receives the report experiment $Y$ before further
decisions. The report concerns only arm $i$, uses no resources, and
changes no dates. Both LPs impose the same expected episode
capacities. Let $D_0(\lambda)$ and $D_Y(\lambda)$ be their dual
objectives, minimized at $\lambda^*$ and $\lambda_Y^*$, respectively.
Then
\begin{equation}\label{eq:report-price-sandwich}
 p_i^{\lambda_Y^*}(Y)\le\LP_Y-\LP_0\le p_i^{\lambda^*}(Y),
\end{equation}
and
\begin{equation}\label{eq:report-price-correction}
 \LP_Y-\LP_0
 =p_i^{\lambda^*}(Y)
 -\bigl[D_Y(\lambda^*)-D_Y(\lambda_Y^*)\bigr].
\end{equation}
The subtracted correction is nonnegative.
\end{theorem}

\begin{proof}
Independence leaves the other arms' priced values unchanged.
At every fixed $\lambda$,
$D_Y(\lambda)=D_0(\lambda)+p_i^\lambda(Y)$.
Evaluating the informed dual at $\lambda^*$ gives the upper bound.
At $\lambda_Y^*$, use $D_0(\lambda_Y^*)\ge\LP_0$ for the lower bound.
Finally substitute
$D_Y(\lambda^*)=\LP_0+p_i^{\lambda^*}(Y)$ and
$D_Y(\lambda_Y^*)=\LP_Y$.
\end{proof}

Both price vectors are chosen for the ex-ante LPs. In particular,
$\lambda_Y^*$ is one vector for the report experiment, rather than
a different vector for each realized report. The fixed-price value
predicts the gain; repricing corrects it by the improvement in the
informed dual. The comparison includes future play-based learning
in both branches.

\subsection{Purchased reports and prior buyout}\label{sec:buyout}

Allow an arm's finite DAG to include report requests with specified
fees, resource charges, signal laws and delivery dates. Their laws
are arm-local, conditional on the arm's hidden record. Add any
observation-resource ledger to its state. Write $R_e^\Q(\pi)$,
$W_e^\Q(\pi)$ and $C_e^\Q(\pi)$ for reward, active-slot use and
fees.
\begin{lemma}[Priced observations on a finite DAG]\label{lem:paid-dag}
At fixed capacity prices the recursion is
\begin{equation}\label{eq:bellman-paid}
 U_e^{\mathrm{fee}}(S)=\max_\pi\left\{
 R_e^\Q(\pi)-\lambda_eW_e^\Q(\pi)-C_e^\Q(\pi)
 +\sum_{S'}\mathcal T_e^\pi(S,S')U_{e+1}^{\mathrm{fee}}(S')
 \right\},
\end{equation}
with $U_{N+1}^{\mathrm{fee}}=0$. Each local step is an LP with fixed
continuation values.
\end{lemma}
\begin{proof}
The supplied signal likelihoods fix the posterior at each history.
Reward, resource use, fees and boundary transitions are linear in
the history probabilities, as in Proposition~\ref{prop:history}.
\end{proof}
Reports that consume time are priced through their resulting states
and remaining opportunities.

For a feasible purchase decision, let $J_0,J_1$ be the optimal
remaining rewards when the report is declined or requested, before
subtracting its quoted fee. Let $L_b\le J_b\le U_b$ be execution
and LP certificates for each full branch, $b=0,1$. Then
\begin{equation}\label{eq:purchase-interval}
 L_1-U_0\le J_1-J_0\le U_1-L_0.
\end{equation}
A fee below the lower endpoint favors the certified purchase policy
over every decline policy; a fee above the upper endpoint favors
the certified decline policy over every purchase policy.
The correction to the LP's predicted price is
$(U_0-J_0)-(U_1-J_1)$. Each branch includes subsequent learning,
and its certificates must include any fees paid on unfinished paths.
The execution theorems for the original play/report model do not
by themselves bound these additional fees.

\paragraph{The price of the recurring prior.}\label{sec:cost-content}\label{sec:law-draw}
Revealing the recurring priors before episode~$1$ is worth
$p^*=V_{\mathrm{prior}}-V_{\mathrm{learn}}$.
This price can be linear when reports arrive too late to improve
decisions (Example~\ref{ex:paid-knowledge}). Under timely stationary
learning it has the following bound. An exact joint episode planner
serves here as a witness for information value.

\begin{proposition}[Price of the prior by posterior sampling]\label{prop:ps-price}
Assume stationary independent arms, finite reward alphabets of sizes
$d_i$, caps $L_i$, report lag at most $g$, and budgets $B_i\ge NL_i$.
Put $D_i^{\rm tape}=\binom{L_i+d_i-1}{d_i-1}$,
$D_{\rm tape}=\prod_iD_i^{\rm tape}$ and $L_\Sigma=\sum_iL_i$.
Let $\pi^*(P)$ be an optimal delay-respecting episode policy for $P$,
with fixed tie-breaking, in the finite class $\Pi$ of deterministic
feasible episode policies. Then
\begin{equation}\label{eq:prior-price-sqrt}
 p^*\le\min\{T,L_\Sigma\}\sqrt{\tfrac12gD_{\rm tape}NH(\pi^*)}
 \le\min\{T,L_\Sigma\}\sqrt{\tfrac12gD_{\rm tape}N\log|\Pi|}.
\end{equation}
For fixed arms, horizon, caps and lag, the price is $O(\sqrt N)$.
\end{proposition}

\begin{proof}
Sample the recurring priors and execute an optimal joint episode
policy for the sample, using $g$ inference threads.
Lemma~\ref{lem:policy-sampling} bounds its loss against the
revealed-prior optimum. The Bayes-optimal learner does at least as
well, which bounds $p^*$.
\end{proof}

\begin{corollary}[The learner against the revealed prior]\label{cor:prior-transfer}
Under Proposition~\ref{prop:ps-price}'s conditions, suppose
$B_i'\ge NL_i$ and run Algorithm~\ref{alg:learner} with an optimal
solution and the uniform completion certificate of
Theorem~\ref{thm:meta}. With $\bar p$ the bound in
\eqref{eq:prior-price-sqrt} and $1-\omega=C/T$ as in that theorem,
\[
 V_A\ge(1-\omega)(V_{\mathrm{prior}}-\bar p)-1.
\]
\end{corollary}
\begin{proof}
The comparison budgets do not bind, so
$\OPT_N(B')\ge V_{\mathrm{prior}}-\bar p$.
Apply Theorem~\ref{thm:meta}.
\end{proof}

The execution factor is charged separately from the information
price. Example~\ref{ex:paid-knowledge} also distinguishes buying the
recurring law from buying the current draw; both options enter
\eqref{eq:bellman-paid}.

\subsection{Certificate corrections and predictive checks}\label{sec:certificates}

Use the full learning relaxation with zero-reward waiting. For
dual-feasible prices and node values in \eqref{eq:state-dual}, put
\begin{equation}\label{eq:policy-dual-value}
 D(\lambda,v)=T\sum_e\lambda_e+\sum_iv_i(o_i).
\end{equation}
Its nonnegative play and handoff slacks are
\begin{equation}\label{eq:policy-bellman-slacks}
\begin{aligned}
 d_i^{\mathrm{play}}(u)
 &=v_i(u)-r_i(u)+\lambda_{e(u)}-\sum_vp_i(u,v)v_i(v),\\
 d_i^{\mathrm{stop}}(u)
 &=v_i(u)-\sum_vh_i(u,v)v_i(v).
\end{aligned}
\end{equation}
Set the play slack to zero where playing is forbidden.

\begin{theorem}[A common policy certificate]\label{thm:common-certificate}
Every admissible policy $A$ has projected play and handoff
probabilities $w_i^A,z_i^A$ feasible for the full learning LP. Let
$\Work_e(A)=\sum_i\sum_{u:e(u)=e}w_i^A(u)$. Then
\begin{equation}\label{eq:common-certificate}
 D(\lambda,v)-V_A
 =\sum_e\lambda_e\bigl(T-\Work_e(A)\bigr)
 +\sum_{i,u}\bigl[w_i^A(u)d_i^{\mathrm{play}}(u)
                 +z_i^A(u)d_i^{\mathrm{stop}}(u)\bigr].
\end{equation}
Every term is nonnegative.
\end{theorem}
\begin{proof}
Project the policy as in Proposition~\ref{prop:relax}(i).
Actual policy probabilities satisfy flow conservation with equality.
Multiply by $v_i(u)$ and sum; internal transitions cancel, leaving
the initial values. Substitute \eqref{eq:policy-bellman-slacks}
and add the unused-capacity charges.
\end{proof}

This identity evaluates the learning-LP executor and HierTS with the
same certificate. Decisions of slack at least $d>0$ contribute at
least $d$ times their expected number of uses. For the optional safe
action with reward $a$, use $r_i(u)-a$, add $aTN$ to $D$, and add
$a$ times the expected number of zero-reward waiting steps to the
right side. The base model has $a=0$.

\paragraph{Repricing continuation.}
Let $\mathcal B_e$ be the integrable remaining-value certificate
chosen from history $H_{e-1}$, and let
$\widehat{\mathcal B}_{e+1}$ be its continuation at the state
actually reached. After correction, let $\mathcal B_{e+1}$ be the
new certificate at that state. Define
\[
 s_e=\mathcal B_e-\E[V_e+\widehat{\mathcal B}_{e+1}\mid H_{e-1}],
 \qquad \Delta_e=\mathcal B_{e+1}-\widehat{\mathcal B}_{e+1}.
\]
With terminal certificates zero, cancellation gives
\begin{equation}\label{eq:repriced-certificate}
 \E\mathcal B_1-V_A
 =\E\sum_es_e-\E\sum_{e<N}\Delta_e.
\end{equation}
The correction is charged with its sign. In particular, if
continuation is preserved and
$\E[V_e+\alpha_eF_{e+1}\mid H_{e-1}]\ge\alpha_eF_e$, where
$0\le\alpha_e\le1$ is chosen from the incoming history, then
\begin{equation}\label{eq:repriced-composition}
 V_A\ge\E[\alpha_1F_1]
       +\E\sum_{e=2}^N(\alpha_e-\alpha_{e-1})F_e.
\end{equation}
A fixed execution factor is paid once; changing it reprices the
remaining continuation. The completion and protected-report
executors supply the required transition guarantees. Re-solving
from an actual state must also account for $\Delta_e$.

\paragraph{Checking the reported predictions.}
A large Bellman slack can arise under a correct model. For finite
report alphabets, index arriving batches by $j$ and let
$\mathcal F_{j-}$ be the actual information just before batch $Y_j$,
including the plays made and the pending-report queue.
A working model $\widehat{\mathbb P}$ predicts
$\widehat p_j(y)=\widehat{\mathbb P}(Y_j=y\mid\mathcal F_{j-})$.
Choose a bounded score $\psi_j$ before seeing $Y_j$ and form
\begin{equation}\label{eq:certificate-innovation}
 Z_j=\psi_j(Y_j)-\sum_y\widehat p_j(y)\psi_j(y).
\end{equation}
For an undelayed play, suppose $u$ is sufficient for this history and
the tested law supplies the LP transitions. With score
$\psi_j=R+v_i(u^+)$,
\[
 v_i(u)-R+\lambda_{e(u)}-v_i(u^+)
 =d_i^{\mathrm{play}}(u)-Z_j.
\]
The realized shortfall separates the predicted allocation loss from
the prediction error. With delays, use the conditional law of the
reports actually arriving, including conditional expectations of
hidden outcomes in any continuation score.

\begin{proposition}[Predictive checking under adaptive play]\label{thm:predictive-audit}
For any admissible policy, under the tested model
$\widehat{\E}[Z_j\mid\mathcal F_{j-}]=0$, and
\begin{equation}\label{eq:predictive-test}
 E_m^{\rm test}=\prod_{j=1}^m
 \frac{\exp\{\psi_j(Y_j)\}}
 {\sum_y\widehat p_j(y)\exp\{\psi_j(y)\}}
\end{equation}
is a nonnegative martingale starting at one. For $0<q<1$,
$\widehat{\mathbb P}\{\sup_mE_m^{\rm test}\ge1/q\}\le q$.
If the policy changes only upon rejection, followed by admissible
recovery, the expected reward loss under $\widehat{\mathbb P}$ is
at most $qNT$ in the original bounded-reward model.
\end{proposition}
\begin{proof}
Each multiplier has conditional mean one. Ville's inequality gives
the crossing bound \citep{shafer2011test}. Couple the checked and
unchecked policies on the same tapes and coins: they agree until
rejection and their rewards differ by at most $NT$ thereafter.
\end{proof}

A report assigned zero probability is an immediate rejection.
For HierTS, the tested law is the hierarchical posterior predictive
mixture for the selected arm, averaged over its latent parameters.
Under an actual conditional law $p_j^\star$, the innovation has drift
$\E_\star[Z_j\mid\mathcal F_{j-}]
=\sum_y(p_j^\star(y)-\widehat p_j(y))\psi_j(y)$.
Certificate scores therefore test discrepancies relevant to the
remaining program. A fixed mixture of the tests for $\psi_j$ and
$-\psi_j$ checks both directions with the same false-alarm bound.
Detection requires distinguishing observations;
the cap-based indistinguishability of Corollary~\ref{cor:tree}
applies here too.

The false-alarm guarantee is under the specified predictive law,
averaged over a Bayesian meta-prior. Repeated episodes reveal one
recurring law, not independent draws from that meta-prior.
Rejection concerns the predictive specification; attributing it to
the prior requires holding the kernels, stationarity and calendar
fixed. Correction retains the actual clock, budget and queue and
charges the change of certificate in \eqref{eq:repriced-certificate}.
A regret bound under misspecification would further require control
of detection delay and replacement-model errors.

\clearpage
\appendix
\section{Illustrative Examples}\label{app:examples}

The planning examples use the model of Section~\ref{sec:model}.
In examples specifying $a>0$, an idle step means taking the safe
alternative of Section~\ref{sec:weights}.
Example~\ref{rem:within} changes the within-episode reward law to
show why the fixed-law assumption is needed.
The numerical comparisons take
$K_i(\theta,\cdot)=\operatorname{Bernoulli}(\theta)$.
Available-report HierTS draws the current episode qualities from
their posterior given received reports, then chooses the feasible
arm with largest sampled mean, including the idle option.
Feasibility checks caps and remaining budgets. Sampling the recurring
law and then the episode quality gives the same posterior draw.

\begin{example}[Pricing one test]\label{ex:price-test}
Let idle steps earn $a$, and let one arm have cap $1$, immediate
feedback and a nonbinding budget. Its quality is $1$ in every episode
with probability $q<a$ and $0$ in every episode otherwise, so its first
play reveals which law recurs. Assume $q(N-2)(1-a)>a-q$. No single
episode gains from the arm: its one play earns $q<a$ in expectation,
and what it reveals cannot be used before the episode ends. The
learning LP tests in episode~$1$ and, after a success, plays once in
every later episode, gaining
$G_N=q(N-1)(1-a)-(a-q)$ over idling.

Take $e_0=1$. Without a test in episode~$1$, the late phase tests in
episode~$2$. Testing in episode~$1$ instead changes the early value by
$\Delta R=-(a-q)$ and the late value by
$\Delta C=(a-q)+q(1-a)=a(1-q)$: the late phase saves its own test and
gains one informed episode. Hence the test is made in episode~$1$ when
$\eta>\eta_1=(a-q)/\{a(1-q)\}$, and in episode~$2$ when
$0<\eta<\eta_1$, whatever $N$. For late targets between these two
plans, $\eta_1$ is the shadow price in
\eqref{eq:late-target-price}. At $\eta=0$ episode~$1$ does not
test, and re-solving with $\eta=0$ at every episode never tests.

Posterior sampling can escape this episode-planning lock-in:
HierTS samples quality $1$ with probability $q$ at each step until
the arm is tested, giving it a positive probability of testing in
each untested episode.

\emph{Planning solution.}
With one arm and cap $1$, the capacity rows at $C=T$ are redundant.
The cascade of Theorem~\ref{lem:perarm} is the exact Bayesian program,
with value
\[
 \LP_N(T,B)=V_{\mathrm{learn}}=aTN+G_N,\qquad
 V_{\mathrm{prior}}-V_{\mathrm{learn}}=a(1-q).
\]
The test and subsequent plays fit on every path, so execution has
no loss. Revealing the law gives
$V_{\mathrm{prior}}=aTN+qN(1-a)$. The episode planner always uses
the safe action and loses $G_N=\Theta(N)$ for fixed $q,a$.
\end{example}

\begin{example}[Cumulative versus episode planning on a fast arm]\label{ex:spread}
Let idle steps earn $a=1/2$. Two priors recur, both with mean $\mu=1/4$. Under $\Pflat$, the mean
is $\theta=1/4$ in every episode. Under $\Pspread$, it is $\theta=1$
with probability $1/4$ and $\theta=0$ otherwise.

Here the episode planner and every one-step rule stop learning, while
the cumulative planner buys tests. Give the arm cap $4\le L\le T$,
a nonbinding total budget and immediate feedback ($\delta=1$). Under $\Pflat$ the
best episode policy never plays the arm. Under $\Pspread$ it plays once and, after a
success, makes $L-1$ further plays, gaining
\[
 \Delta_{\mathrm{spread}}=\mu(1-a)(L-1)-(a-\mu)=\frac{L-1}8-\frac14
\]
over $aT$.

\emph{One play cannot tell.} Under both priors a single play is
Bernoulli$(1/4)$ (Corollary~\ref{cor:tree}(iii)). A second play after a
success can tell: under $\Pspread$ it is a sure success, so a failure rules $\Pspread$
out.

\emph{Certainty equivalence locks in.} Let $q$ be the posterior
probability of $\Pspread$. If
\[
 q<q_0=\frac{a-\mu}{\Delta_{\mathrm{spread}}+a-\mu},
\]
the episode policy that is optimal under $\Pb_e$ never plays the
arm. Then no data arrive and $q$ never changes,
although under $\Pspread$ every episode could gain $\Delta_{\mathrm{spread}}$.

To verify the threshold, let $\varrho$ be the probability of making
any play. It is the same under both priors because the first play
precedes any observation. Under $\Pflat$ the loss is at least
$\varrho(a-\mu)$. Under $\Pspread$ the gain is at most
$\varrho\Delta_{\mathrm{spread}}$: a bad quality loses at least $a$,
and a good quality gains at most $L(1-a)$. The mixture's gain is at most
$\varrho\{q\Delta_{\mathrm{spread}}-(1-q)(a-\mu)\}$,
which is negative whenever $\varrho>0$.

\emph{The test.} Play once, and after a success play once more. Under $\Pflat$, this costs $(a-\mu)(1+\mu)=5/16$ in expectation.
It rules out $\Pspread$ with probability $\mu(1-\mu)=3/16$ per episode. Under $\Pspread$, a success
followed by a success multiplies the odds of $\Pspread$ by $4$.
Both plays count toward reward, and their outcomes are reported
before the next episode.

\emph{What buys the test.} After one test, $q$ becomes $0$, stays at
$q$, or becomes $4q/(1+3q)$.
While $4q/(1+3q)<q_0$, no outcome of one test changes the next episode policy.
The one-step value of information of the test is then zero, so a rule
based on it never buys the test \citep{radner1984}.
The obstruction is not specific to this test. The data of one episode
multiply the odds of $\Pspread$ by at most $4^{L-1}$, which happens
after $L$ successes. Under the condition $4^{L-1}q/(1-q+4^{L-1}q)<q_0$, no local policy
within one episode can change the next episode policy. Every one-step rule
therefore stops exploring.

The full belief about the recurring law is needed here. Knowing
$\widetilde P=q\Pspread+(1-q)\Pflat$ and being uncertain between
$\Pspread$ and $\Pflat$ give the same current predictive prior.
Only the second belief can learn which law recurs and change its
future episode policy.

\emph{Planning solution.}
The single-arm learning LP is exact here: its capacity rows are
redundant, and every local policy executes within $T$.
The revealed-prior value is $aTN+qN\Delta_{\mathrm{spread}}$.
For fixed $L$ and $q>0$, Proposition~\ref{prop:ps-price} gives
\[
 \LP_N(T,B)=V_{\mathrm{learn}}
 \ge aTN+qN\Delta_{\mathrm{spread}}-O(\sqrt N).
\]
The episode planner remains at $aTN$ under the stated condition on
$q$, so the advantage is linear in $N$. Direct execution realizes
the LP program without loss and preserves every learning transition.
\end{example}

\begin{example}[Probe, wait, and return]\label{ex:delay-motivation}
Let $T=2d$, where $d\ge3$ is an integer. There is one unknown arm with
known recurring prior
\[
 \Pr(\theta^{(e)}=1)=\tfrac14,
 \qquad \Pr(\theta^{(e)}=0)=\tfrac34,
\]
and idle steps earn $a=1/2$. A play at step $t$ is reported before
the decision at step $t+d$. Caps and total budgets do not bind. Then the
Bayes-optimal reward and the available-report HierTS reward in each
episode are
\[
 V_* = \frac{9d}{8}-\frac14,
 \qquad
 V_{\mathrm{TS}}=
 \frac{17d}{16}-\frac9{16}
       +\frac9{16}\left(\frac34\right)^d.
\]
Consequently,
\[
 V_*-V_{\mathrm{TS}}
 =\frac d{16}+\frac5{16}
       -\frac9{16}\left(\frac34\right)^d
 \ge \frac{T}{32}.
\]
Repeating the instance for $N$ episodes gives an expected-reward gap of
at least $NT/32$.

The optimal policy probes at step~1, idles until the report and
returns to the arm only on success. Its reward is
\[
 \frac14+\frac{d-1}{2}
 +d\left(\frac14+\frac34\cdot\frac12\right)
 =\frac{9d}{8}-\frac14.
\]
Extra probes before the first report have mean $1/4<1/2$ and reveal
nothing sooner. A first probe at step $\tau$ gains
$-1/4+\max\{T-\tau-d+1,0\}/8$ over idling, maximized at $\tau=1$.

Before a report, HierTS plays the arm with probability $1/4$ and
earns $7/16$ per step in expectation; afterward it earns $5/8$.
At step $t$, no report has arrived with probability
$(3/4)^{\max\{t-d,0\}}$, independently of the quality. Hence
\[
 V_{\mathrm{TS}}=\frac58(2d)-\frac3{16}
 \left[d+\sum_{k=1}^d(3/4)^k\right].
\]
This gives the displayed value and gap.

\emph{Unknown recurring prior.}
Let the success probability $p$ be $1/8$ or $3/8$, equally likely,
and $T\ge32$ be even. At every episode start,
$q_e=\E[p\mid H_{e-1}]\in[1/8,3/8]$. In the first $d=T/2$ steps
HierTS loses $q_e(1/2-q_e)\ge3/64$ per step relative to idling.
Even granting it the quality for the remaining steps, its expected
reward is at most $T/2-3d/64+(T-d)/8$, since $\E q_e=1/4$.
Probe--wait--return needs no knowledge of $p$ and earns
$T/2-1/4+(T-d)/8$. Its advantage is at least
$3T/128-1/4\ge T/64$ per episode.

\emph{Planning solution.}
In the known-prior instance, the undelayed local LP probes once and,
after success, continues to the cap. Its relaxed value per episode is
$5T/8-3/8$. To execute this plan, protect the first probe and its
report, then truncate the reward tail to the $d$ remaining steps.
This is precisely probe--wait--return, with
\[
 V_{\mathrm{exec}}=\frac{9T}{16}-\frac14=V_*,
 \qquad \frac{\LP_N(T,B)}{N}-V_*=\frac{T-2}{16}.
\]
The first report reveals the episode quality completely; further
plays supply no additional information, including in the unknown-prior
extension. With nonbinding budgets, truncating those plays leaves
the value of later decisions unchanged. Thus this execution is optimal
even at delay $T/2$, while the immediate-feedback relaxation remains
an upper bound.
\end{example}

\begin{example}[A completely slow arm with known prior]\label{ex:slow}
Let the safe action earn $a=1/2$. One unknown arm has cap $T$,
a nonbinding total budget and known recurring prior
$\Pr(\theta^{(e)}=1)=1/4$, $\Pr(\theta^{(e)}=0)=3/4$.
No report arrives within its episode. Available-report Thompson
sampling therefore plays the unknown arm with probability $1/4$
at every step. Each such play has expected reward $1/4$, so its
expected episode reward is
$T[(1/4)(1/4)+(3/4)(1/2)]=7T/16$.

\emph{Planning solution.}
Without within-episode reports, the number of unknown-arm plays is
independent of its fresh quality. Every play replaces safe reward
$1/2$ by expected reward $1/4$. The slow-arm LP therefore makes no
unknown-arm plays and is exactly executable:
$\LP_N(T,B)=V_{\mathrm{prior}}=V_{\mathrm{learn}}=NT/2$.
Thompson sampling loses $T/16$ per episode, or $NT/16$ in total,
although the recurring prior is known.
\end{example}

\begin{example}[Spending a scarce arm too early]\label{ex:budget-motivation}
Let $N$ be even and $T=3m$, with integer $m\ge1$. Idle steps earn~0.
Two unknown arms have known recurring priors concentrated at means
$1$ and $1/2$, respectively. Both arms have per-episode cap $2m$.
Their total budgets are $B_1=Nm$ and $B_2=2Nm$. Feedback is immediate.
HierTS restricted to feasible arms earns $7Nm/4$, whereas the optimal
feasible reward is $2Nm$. The expected-reward gap is $NT/12$.

HierTS always prefers arm~1. It plays it $2m$ times in each of the
first $N/2$ episodes, exhausting its budget. Arm~2 is played $m$ times
in each of those episodes and $2m$ times in every later one; the last
$m$ steps of those later episodes are idle. Total reward is
$Nm+(1/2)(3Nm/2)=7Nm/4$.
Playing arm~1 $m$ times and arm~2 $2m$ times in every episode spends
both budgets and earns $2Nm$. This is optimal because it equals the
sum of the budgets times their means.

\emph{Unknown recurring priors.}
Independent meta-priors supported on arm~1 laws in $[3/4,1]$ and
arm~2 laws in $[1/4,1/2]$ preserve the same sample ordering.
If $\mu_i$ is arm~$i$'s unconditional mean, the two allocations earn
$(\mu_1+3\mu_2/2)Nm$ and $(\mu_1+2\mu_2)Nm$, respectively.
The gap is $\mu_2Nm/2\ge NT/24$.

\emph{Planning solution.}
In the known-prior instance, the deterministic allocation $(m,2m)$
in every episode is an optimal learning-LP solution. Its value
$2Nm=B_1+B_2/2$ attains the budget upper bound, and its workload is
exactly $T$ on every path. Hence $\LP_N(T,B)=\OPT_N(B)=2Nm$ and execution
has no loss. The same fixed allocation realizes the stated improvement
in the unknown-prior extension.
\end{example}

\begin{example}[A one-step risk separation]\label{ex:risk}
For lower-tail CVaR with tail mass $\beta$, write
\[
 \operatorname{C}_\beta(V)=\sup_z
 \left\{z-\frac1\beta\E(z-V)^+\right\}.
\]
Let $T=1$, with immediate feedback and nonbinding budgets. Idling
earns $a=1/2$. The unknown arm has known recurring prior
$\Pr(\theta^{(e)}=1)=3/4$ and $\Pr(\theta^{(e)}=0)=1/4$.
Let $V$ be the episode's realized reward, including the policy's
randomization. Idling is optimal for each of the following
three episode criteria, with value $1/2$:
\[
 \E V-2\operatorname{Var}(V),\qquad
 \operatorname{C}_{1/4}(V),\qquad
 \E\bigl[V-3(1/2-V)^+\bigr].
\]
HierTS attains, respectively,
\[
 \frac{49}{128},\qquad \frac18,\qquad \frac{13}{32}.
\]
The same failures hold for a rule that samples the episode quality and
then maximizes any of these three criteria conditional on that sample.

If a policy chooses the arm with probability $r$, its reward
probabilities at $0$, $1/2$ and $1$ are $r/4$, $1-r$ and $3r/4$.
Consequently,
\[
 \E V-2\operatorname{Var}(V)=\frac12-\frac r4+\frac{r^2}{8},\qquad
 \operatorname{C}_{1/4}(V)=\frac{1-r}{2},\qquad \E[V-3(1/2-V)^+]=\frac12-\frac r8.
\]
Each is maximized at $r=0$. HierTS takes $r=3/4$, giving the stated
values. Conditional on a sampled quality, rewards are deterministic;
all three criteria therefore make the same sampled-mean choice.

\emph{Unknown recurring prior.}
Let $p$ be $7/10$ or $4/5$, equally likely, with
$\theta^{(e)}\mid p\sim\operatorname{Bernoulli}(p)$.
At any available history, $q=\E[p\mid H]\in[7/10,4/5]$.
HierTS's reward masses at $0$, $1/2$ and $1$ are $q(1-q)$,
$1-q$ and $q^2$. Its conditional lower-quarter CVaR is
$1/2-2q(1-q)\le9/50$, whereas idling gives $1/2$.
Thus its loss is at least $8/25$ in every episode.

\emph{Planning solution.}
The one-step formulations of Subsection~\ref{sec:beyond} retain the
displayed outcome probabilities. The CVaR and downside objectives are
linear in $r$ and attain $1/2$ at $r=0$. Mean--variance has the linear
upper certificate
\[
 \frac12-\frac r4+\frac{r^2}{8}
 \le \frac12-\frac r8\le\frac12
 \qquad(0\le r\le1).
\]
Idling attains this bound as well, so all three known-prior optima are
exact. The one-step decisions require no execution correction.
\end{example}

\begin{example}[Learning the law versus observing the draw]\label{ex:paid-knowledge}
Each benchmark gap can be linear in $N$, and a different paid report
closes each. Consider one unknown arm with cap $L\le T$, total budget $B\ge NL$,
and idle steps earning $a\in(0,1)$. The additional perfect reports
below are idealized finite signals with no active-slot or play-budget
charge; their fees are subtracted from reward.
The three benchmarks first use only ordinary reports;
$V_{\mathrm{fee}}$ denotes the optimum after adding the stated paid offer.

\emph{Learning the recurring law.} Let the unknown prior be
$P_+$, the point mass at $u$, with probability $q\in(0,1)$, and $P_-$,
the point mass at $\ell$, otherwise, where $0\le\ell<a<u\le1$. The arm
is slow, and every ordinary batch is scheduled after episode $N$ (so
$g\ge N$). Thus ordinary plays yield no information during the horizon,
and each play has predictive mean $\bar\mu=qu+(1-q)\ell$. Using the arm
pays exactly when $\bar\mu>a$; with the law known, it pays only under
$P_+$. Hence
\[
\begin{gathered}
 V_{\mathrm{learn}}=aTN+NL(\bar\mu-a)^+,\qquad V_{\mathrm{prior}}=aTN+NLq(u-a),\\
 D_N:=V_{\mathrm{prior}}-V_{\mathrm{learn}} =NL\min\{q(u-a),(1-q)(a-\ell)\}.
\end{gathered}
\]
Here $x^+=\max\{x,0\}$. If a single report offered before episode~1
reveals $P$ for fee $c\ge0$, buying it realizes the law-informed value
less $c$, so
\[
 V_{\mathrm{fee}}=V_{\mathrm{learn}}+(D_N-c)^+,
 \qquad V_{\mathrm{prior}}-V_{\mathrm{fee}}=\min\{D_N,c\}.
\]
The report is worth buying when $c<D_N$: its price is supported by
all $N$ episodes even though it supplies no operating reward itself.
Knowing $P$ also fixes every episode mean here, so
$V_{\mathrm{state}}=V_{\mathrm{prior}}$.
Put $D_1=D_N/N$. With weight $1$ on episode~1 and $\eta$ on later
episodes, the report is bought when $c<D_1\{1+\eta(N-1)\}$.
When $D_1<c<D_N$, optimizing only episode~1 rejects
this one-time offer, whereas the cumulative planner buys it and gains
$D_N-c$.

\emph{Observing the current draw.} Now let $P$ be known, with each
episode mean independently equal to $1$ with probability $p>a$
and $0$ otherwise. Ordinary reports arrive only after their episode,
so every episode starts at success probability $p$. Without a current
report the optimal episode policy uses all $L$ allowed plays, each
gaining $p-a$; with the mean revealed it uses them only when the mean
is $1$, each gaining $1-a$. Consequently,
\[
 V_{\mathrm{learn}}=V_{\mathrm{prior}}=N[aT+L(p-a)],
 \qquad V_{\mathrm{state}}-V_{\mathrm{prior}}=NL(1-p)a.
\]
A perfect report at each episode start is worth buying when its fee is
below $L(1-p)a$. For $p=3/4$ and $a=1/2$, the no-report gap is $NL/8$,
despite complete knowledge of the recurring law.

\emph{Planning solution.}
The single-arm program is exact for each of these information menus.
The cap $L\le T$ makes shared capacity redundant, and the instantaneous
reports use no active slots. Adding their request nodes as in
Lemma~\ref{lem:paid-dag} therefore gives exactly the purchase thresholds
$D_N$ for the law and $L(1-p)a$ for the current draw, with no scheduling
loss. With ordinary reports alone it attains the displayed
$V_{\mathrm{learn}}$.
\end{example}

\begin{example}[The fixed-law boundary]\label{rem:within}
Movement within an episode can invalidate the relaxation. Take $n$
arms with caps at least $T/n$, each of mean $1$ in the same window of
$T/n$ steps and $0$ outside it, and let $0\le a<1$. Each arm alone gains
$(T/n)(1-a)$ over $aT$, so the relaxation predicts a gain of
$T(1-a)$ and total value $T$. A global
policy gains at most $(T/n)(1-a)$, because the good windows coincide.

\emph{Scope.}
This construction violates the fixed within-episode law assumed in
Section~\ref{sec:model}. It establishes why that assumption is needed
for the relaxation with one shared capacity row per episode.
\end{example}

\section{Execution proofs}\label{app:execution-proofs}

\subsection{Prefix bounds for buffered execution}\label{app:buffered-prefix}

\begin{proof}[Prefix bounds in Lemma~\ref{lem:buffered-executor}]
Each block supplies at least one new outcome. There are at most $k$
blocks before commitment. Blocks have size one while
$q\le1/\epsilon$, so this initial phase uses at most
$\lfloor1/\epsilon\rfloor+1$ blocks. Afterwards, every block except
possibly the last has $b\ge\epsilon q$. Completed blocks therefore
increase $q$ by a factor of at least $1+\epsilon$.
Since $\ln(1+\epsilon)\ge\epsilon\ln2$, the number of further
blocks is at most $\ln(\epsilon k)/(\epsilon\ln2)+1$ when
$\epsilon k\ge1$. Since $1/\epsilon+2\le3/\epsilon$ and
$1/\ln2<1.45$, this gives the stated $B_\epsilon(k)$.

A block of $b$ plays starting at step $s$ receives its last report
at step $s+b-1+\delta_i$. It occupies $b+\delta_i-1$ own steps.
The plays made before commitment number at most $(1+\epsilon)k+1$.
This proves the commit-policy prefix bound.

For a general policy, replace $k$ in this count by
$\lceil\delta_i/\epsilon\rceil$. Before switching modes there are
at most $(3+1.45\ln(\delta_i+1))/\epsilon$ blocks and
$2\delta_i/\epsilon+3$ plays. The switch adds at most
$\delta_i$ plays. The total prefix is bounded by
$5(\delta_i+1)(1+\ln(\delta_i+1))/\epsilon$.
In delay-free mode, the oldest unused outcome was generated at least
$\delta_i$ own steps earlier and is available. Each remaining play
can occur at the next own step. No physical plays remain after
the cap is reached.
\end{proof}

\subsection{The reserve lower bound}\label{app:reserve-tail}

\begin{proof}[Proof of Theorem~\ref{thm:reserve-lower}]
Each arm's fresh episode mean is $1$ with probability $p=1/(m+1)$
and $0$ otherwise, independently across arms and episodes. A degenerate
meta-prior at these known priors is allowed by
Assumption~\ref{ass:drift}. Let $\pi^\star$ probe once, continue to
the cap after success, and stop after failure. Its workload and reward
are
\[
 X=1+(L-1)G,\quad G\sim\operatorname{Bernoulli}(p),\qquad
 \E X=m,\quad R(\pi^\star)=\frac{m^2}{m+1}.
\]
Thus the collection uses expected capacity exactly $C=nm$ in each
episode. At price $\lambda_e=1/2$, a success makes every remaining
play profitable and a failure makes every remaining play unprofitable.
Starting has strictly positive priced value
\[
 R(\pi^\star)-\tfrac12\E X
 =\frac{m(m-1)}{2(m+1)}>0.
\]
These choices uniquely maximize the priced arm problem on reachable
histories. Priors are known and budgets do not bind, so this holds in
every episode. The collection is primal feasible and its dual value is
\[
 \tfrac12 CN+Nn\left(\frac{m^2}{m+1}-\frac m2\right)
 =Nn\frac{m^2}{m+1},
\]
equal to its reward. Every optimal solution must therefore attain each
priced arm maximum and prescribe the stated plays.

Let $Z_e$ count the good arms in episode $e$. Then the $Z_e$ are
independent $\operatorname{Binomial}(n,p)$ variables and total
standalone work is
\[
 S_e=n+(L-1)Z_e,\qquad S_e-C=(L-1)(Z_e-np).
\]
All prescribed policies can finish if and only if $S_e\le T$:
serial execution suffices, while excess prescribed work cannot fit
under any scheduler. With
\[
 \beta(s)=\Pr\left(Z_e>np+\frac{s}{L-1}\right),
\]
the largest possible full-completion probability is exactly
$(1-\beta(s))^N$. Hence a probability of at least $1-\delta$ requires
\begin{equation}\label{eq:necessary-episode-tail}
 \beta(s)\le1-(1-\delta)^{1/N}
 \le\frac{-\log(1-\delta)}N\le\frac{2\delta}N=2e^{-h}.
\end{equation}

Put $s_\star=(L-1)\sqrt{2vh}=(m-1)\sqrt{2nmh}$.
The binomial estimate below, Lemma~\ref{lem:binomial-moderate}, gives, for every fixed
$\varepsilon_{\rm res}\in(0,1)$,
\[
 \log\beta((1\pm\varepsilon_{\rm res})s_\star)
 =-(1\pm\varepsilon_{\rm res})^2h(1+o(1)).
\]
The assumptions $h\to\infty$ and $h=o(v)$ put the binomial deviations
between $\sqrt v$ and $v$; integer rounding is negligible. For
$s\le(1-\varepsilon_{\rm res})s_\star$, the lower tail estimate contradicts
\eqref{eq:necessary-episode-tail}. For $s\ge(1+\varepsilon_{\rm res})s_\star$, the
upper estimate gives $\beta(s)\le\delta/N$ eventually, so a union
bound gives full completion with probability at least $1-\delta$.
Letting $\varepsilon_{\rm res}$ tend to zero proves \eqref{eq:necessary-reserve}.

Finally, $s_\star^2/C^2=2(m-1)^2h/(nm)=o(1)$, so $T\sim C$ at
the smallest reserve. Since $L=m^2$, substitution gives the last two
statements.
\end{proof}

\begin{lemma}[Binomial moderate deviations]\label{lem:binomial-moderate}
Let $Z\sim\operatorname{Binomial}(n,p)$, with $p$ allowed to vary
along the sequence, and put $v=np(1-p)$. If $v\to\infty$,
$t/\sqrt v\to\infty$ and $t/v\to0$, then
\[
 \log\Pr(Z\ge np+t)=-(1+o(1))\frac{t^2}{2v}.
\]
A noninteger threshold is rounded upward.
\end{lemma}

\begin{proof}
Let $k=\lceil np+t\rceil$ and let $D(\cdot\Vert\cdot)$ be the binary
relative entropy. The Chernoff bound gives the upper estimate
$\Pr_p(Z\ge k)\le\exp\{-nD(k/n\Vert p)\}$. For the lower estimate, put
$r=(k+2\sqrt v)/n$; since $t=o(v)$, $r-p=o(p(1-p))$ and
$nr(1-r)=(1+o(1))v$. Under $\operatorname{Binomial}(n,r)$, the event
$A=\{|Z-nr|\le2\sqrt v\}$ lies above $k$ and has probability at least
$1/2$ by Chebyshev. On $A$ the likelihood ratio
$\Pr_p(Z=z)/\Pr_r(Z=z)=\exp\{-nD(r\Vert p)-(z-nr)\ell\}$, with
$\ell=\log\frac{r(1-p)}{p(1-r)}$, is at least
$\exp\{-nD(r\Vert p)-2\sqrt v\,\ell\}$. Taylor expansion gives
$nD(r\Vert p)$ and $nD(k/n\Vert p)$ both equal to $(1+o(1))t^2/(2v)$,
and $2\sqrt v\,\ell=O(t/\sqrt v+1)=o(t^2/v)$.
\end{proof}

\subsection{Single-episode execution and comparison with GM13}\label{app:cap-execution}

\paragraph{Depth-restricted programs.}
A depth-$k$ commit policy may commit earlier and randomize its tail
length. Committed outcomes do not change that length, but may be
retained for later learning.
Fix a prior $P$, a depth $k$ and the caps $L_i$.
For arm $i$, put
$k_i=\min\{k,L_i\}$ and truncate its posterior graph at depth $k_i$. At
a state $u$ of depth $j$, let $z_u$ be the probability of reaching $u$
before commitment, $x_u$ that of an adaptive play there, $y_u$ that of
committing there, and $c_u$ the expected committed plays. The
constraints are
\[
 z_u=\mathbf1_{\{u=o_i\}}+\sum_vp_i(v,u)x_v,\qquad
 x_u+y_u\le z_u,\qquad 0\le c_u\le(L_i-j)y_u,
\]
with $x_u=0$ at depth $k_i$; unassigned reach mass stops. The objective
is $\sum_{i,u}r_i(u)(x_u+c_u)$ subject to
$\sum_{i,u}(x_u+c_u)\le T$. Playing adaptively with probability
$x_u/z_u$, committing with probability $y_u/z_u$, and randomly rounding
$c_u/y_u$ to a tail length realizes the flows; conditionally i.i.d.\
rewards give each committed play mean $r_i(u)$, so the objective is
exact. For continuation, assign unallocated reach probability to a
zero-length commitment. Index commitment by its tail length
$b$: a prior-free weight $y(h,b)$ contributes
$b\,y(h,b)\sum_{y\in\mathcal Y_i}y\ell_i(P_i;hy)$ to reward and
$b\,y(h,b)\ell_i(P_i;h)$ to workload, and each ordered tail
$t\in\mathcal Y_i^b$ has terminal probability $y(h,b)\ell_i(P_i;ht)$.
For a one-shot arm, commitment after the first outcome uses
$c_u\le(L_i-1)y_u$. Replacing this tail cap by the post-observation
horizon $T-\delta_i$ would exclude feasible nonadaptive plays.

\paragraph{The general guarantee.}
Admission controls expected workload. The cap bounds each arm's
realized workload. We use both facts to bound the reward lost at
the horizon.

For $\nu>0$, let $Z_\nu\sim\operatorname{Poisson}(\nu)$ and define
\[
  H_\nu(x)=\int_0^x\min\{1,(y-\nu+1)^+\}\,dy,\qquad \rho_P(\nu)=1-\E[H_\nu(Z_\nu)]/\nu.
\]
The subscript $P$ denotes the Poisson comparison. For $X\ge1$
and $0\le\gamma<1$, set
\begin{equation}\label{eq:finite-retention}
 c_{X,\gamma}=\max\{(1-\gamma)/2,\rho_P((1-\gamma)X)\}.
\end{equation}
The first term retains GM13's triangle bound. The second gives the
small-cap improvement. Related Poisson supply bounds yield the same
leading constant in sequential posted pricing
\citep{chakraborty2010pricing}.
\begin{lemma}[Single-episode execution guarantee]\label{lem:execution}
Fix feasible local policies $\pi_i$ with
$\sum_i\Work(\pi_i)\le T$ and value
\[
 J=\sum_iR(\pi_i).
\]
Let $\epsilon\in(0,1]$ and $\gamma\in(0,1)$. Suppose
$W_i\le\gamma T$ for every used unknown arm.
Choose a common cap $0<L\le T$ with $L_i\le L$, and put
$\nu=(1-\gamma)T/L$. Admit each transformed policy independently
with probability $\alpha=(1-\gamma)/(1+\epsilon)$.
The priority scheduler of Section~\ref{sec:handoff}, with admission, earns
\[
  V^P(\ALG)\ge\frac{(1-\gamma)\rho}{1+\epsilon}\,J,\qquad
 \rho=\max\{(1-\gamma)/2,\rho_P(\nu)\}.
\]
A $(1+\epsilon')$-approximate solution of $\LP(P)$ therefore earns at least
\[
 \frac{(1-\gamma)\rho}{(1+\epsilon)(1+\epsilon')}
 \LP(P).
\]
For a depth-restricted solution, replace $\LP(P)$ by $\LP^{(k)}(P)$.
Use the prefix bound for the supplied policy class.
For one-shot arms, depth $1$ gives
$\LP^{(1)}(P)=\LP(P)\ge\OPT_\delta(P)$ without a commitment loss.
Direct undelayed execution uses the original policies. The same
guarantee then holds with $\epsilon=\gamma=0$ and $\alpha=1$.
\end{lemma}

Order the transformed policies $\pi_i^{\mathrm x}$ by decreasing
$R(\pi_i^{\mathrm x})/\Work(\pi_i^{\mathrm x})$, breaking ties in
favor of smaller delays, and run the priority scheduler of
Section~\ref{sec:handoff} in this order. Admission replaces a local
policy by a null policy with probability $1-\alpha$.

\begin{proof}[Proof of Lemma~\ref{lem:execution}]
Let $\pi_i^{\mathrm x}=A_\epsilon(\pi_i)$ and let
$\pi_i^{\mathrm{adm}}$ be its admitted version. Index the arms in
the scheduler's order. Write $\mathsf N_i$ for the number of plays
made by $\pi_i^{\mathrm{adm}}$ when run alone. Put
\[
 \tau_i=\E\mathsf N_i,\qquad
 R_i=R(\pi_i^{\mathrm{adm}})=\alpha R(\pi_i^{\mathrm x}).
\]
The arm tapes and admission coins are independent. Hence the
$\mathsf N_i$ are independent. Lemma~\ref{lem:buffered-executor} gives
\[
 \mathsf N_i\le L,\qquad
 \sum_i\tau_i\le\alpha(1+\epsilon)\sum_i\Work(\pi_i)
 \le(1-\gamma)T=\nu L.
\]
Discard policies with $\tau_i=0$. If none remain, idling fills
the horizon and proves the claim.

\emph{Pathwise truncation.}
Pre-sample every arm's episode quality, tape and coins. Let
$S_i=\sum_{j<i}\mathsf N_j$. Condition on the tapes and coins
of higher-priority arms. Their actual plays do not depend on arm
$i$ and number at most $S_i$.

Let $M_i$ be the plays of arm $i$ made by the scheduler. Its waiting
steps advance in calendar time. A higher-priority play blocks it
only when it requests a play. After at most $W_i\le\gamma T$ own prefix
steps, every unblocked step is a play until the policy ends. Thus
\[
  M_i\ge\min\{\mathsf N_i,((1-\gamma)T-S_i)^+\} \ge\phi_i\mathsf N_i,\qquad
 \phi_i=\min\{1,(\nu L-S_i)^+/L\}.
\]
If the prefix does not fit, the first lower bound is zero.
The factor $\phi_i$ is fixed by the higher-priority tapes and is
independent of arm $i$. Lemma~\ref{lem:truncation} therefore gives
expected reward at least $R_i\E\phi_i$.

\emph{The cap-independent bound.}
Since $L\le T$, we have
$\phi_i\ge1-\gamma-S_i/T$. Consequently,
\[
 \E\phi_i\ge1-\gamma-\frac{\sum_{j<i}\tau_j}{T}.
\]
The densities $R_i/\tau_i$ are nonincreasing, and
$\sum_i\tau_i\le(1-\gamma)T$.
The triangle argument of \citet[Lemma~3]{GM13} yields
\[
 \sum_iR_i\E\phi_i\ge\frac{1-\gamma}{2}\sum_iR_i.
\]

\emph{The Poisson bound.}
Define
$h(s)=\min\{1,(s-(\nu-1)L)^+/L\}$, so that
$1-\phi_i=h(S_i)$. Put $S=\sum_i\mathsf N_i$.
The function $h$ is nondecreasing and
$S_{i+1}=S_i+\mathsf N_i$. Hence
\[
 \sum_i\mathsf N_ih(S_i)\le\int_0^Sh(y)\,dy=:H(S).
\]
The function $H$ is convex, and $H(Ly)=LH_\nu(y)$.

Every convex function lies below its chord on $[0,L]$.
It follows that $\mathsf N_i$ is dominated in convex order by
$L$ times a Bernoulli variable of mean $\tau_i/L$.
A Bernoulli variable of mean $q$ is dominated by a Poisson variable
of mean $q$. To see this, take a convex function $f$.
For every integer $m\ge0$,
\[
 f(m)\ge f(0)+m(f(1)-f(0)).
\]
Averaging over the Poisson law gives at least
$(1-q)f(0)+qf(1)$, the Bernoulli expectation.
Convex order is preserved under independent sums
\citep[Chapter~3]{shaked2007}. Equivalently, condition on the other
summands and replace one at a time. Thus $S$ is dominated by
$LZ_{\nu_0}$, where
$\nu_0=\sum_i\tau_i/L\le\nu$.

Independence of $\mathsf N_i$ and $S_i$ now gives
\[
  \sum_i\tau_i\E h(S_i) =\E\sum_i\mathsf N_ih(S_i) \le\E H(S)\le L\E H_\nu(Z_{\nu_0}).
\]
The function $u\mapsto\E H_\nu(Z_u)$ vanishes at zero and is
convex. Its second derivative is
\[
 \E\big[H_\nu(Z_u+2)-2H_\nu(Z_u+1)+H_\nu(Z_u)\big]\ge0.
\]
Therefore
\[
 \E H_\nu(Z_{\nu_0})
 \le\frac{\nu_0}{\nu}\E H_\nu(Z_\nu),
\]
which implies
\[
 \sum_i\tau_i\E h(S_i)
 \le(1-\rho_P(\nu))\sum_i\tau_i.
\]
Both $R_i/\tau_i$ and $1-\E h(S_i)$ are nonincreasing in $i$.
Weighted Chebyshev's inequality gives
\[
  \sum_iR_i(1-\E h(S_i)) \ge\left(\sum_iR_i\right) \frac{\sum_i\tau_i(1-\E
 h(S_i))}{\sum_i\tau_i} \ge\rho_P(\nu)\sum_iR_i.
\]

\emph{Combining the bounds.}
The unknown arms earn at least $\rho\sum_iR_i$.
Lemma~\ref{lem:buffered-executor} also gives
\[
 \sum_iR_i=\alpha\sum_iR(\pi_i^{\mathrm x})
 \ge\alpha\sum_iR(\pi_i).
\]
Consequently,
\[
  V^P(\ALG)\ge\rho\alpha\sum_iR(\pi_i)=\rho\alpha J.
\]
The approximate-solution guarantee follows by substituting
$J\ge\LP(P)/(1+\epsilon')$, or the restricted LP value for a
restricted policy class.

For direct undelayed execution, take
$\pi_i^{\mathrm x}=\pi_i$, $W_i=0$ and $\alpha=1$.
The proof then applies verbatim with $\epsilon=\gamma=0$.
\end{proof}
\begin{proof}[Proof of Theorem~\ref{thm:budget} and Corollary~\ref{cor:delayed}]
For direct undelayed execution, apply Lemma~\ref{lem:execution}
with $L=T/X$, $\epsilon=\gamma=0$ and $\alpha=1$. This gives
$c_{X,0}J$. For delayed execution, use the same common cap and the
stated admission probability. Then $\nu=(1-\gamma)X$ and
$\rho=c_{X,\gamma}$, which gives Corollary~\ref{cor:delayed}.

For integer $X$,
\[
  \E H_X(Z_X)=\tfrac12\Pr(Z_X\ge X)+\E(Z_X-X)^+,\qquad \E(Z_X-X)^+=X\Pr(Z_X=X).
\]
Hence $c_{X,0}$ equals the stated $c_X$.
Continuity of $\rho_P$ gives $c_{X,\gamma}\to c_X$ at fixed $X$.
As $\epsilon,\gamma\to0$, the delayed factor therefore tends
to the undelayed factor. Stirling's formula gives
\[
 \Pr(Z_X=X)=(2\pi X)^{-1/2}+O(X^{-3/2}),
\]
and the remaining tail term is $O(X^{-1})$. Taking the reciprocal
gives the expansion of $1/c_X$.
For $X=4,8,16$, the factors are below $1.363$, $1.211$ and $1.132$.

\end{proof}

\paragraph{Integrality gaps.}
\begin{proof}[Proof of Proposition~\ref{prop:gap}]
Let $L=T/X$, take $p=T^{-1/2}$, and set
$n=\lfloor T/(1+p(L-1))\rfloor$.
Each arm has $\theta=1$ with probability $p$ and zero otherwise.
The local policy probes the arm and continues to its cap after success.
The collection is feasible for the relaxation. Its expected work is
$n(1+p(L-1))\le T$ and its reward is
$npL=(1-o(1))T$. Also, $np\to X$.

Let $B_{\mathrm{good}}\sim\operatorname{Binomial}(n,p)$ count
the good arms.

Even a clairvoyant earns at most
$L\min\{B_{\mathrm{good}},X\}$. The bounded variable
$\min\{B_{\mathrm{good}},X\}$ converges in mean to
$\min\{Z_X,X\}$ for $Z_X\sim\operatorname{Poisson}(X)$.
Since
\[
 \E Z_X=X,\qquad \E(Z_X-X)^+=X\Pr(Z_X=X),
\]
division by the relaxation's feasible value proves the claim.
\end{proof}

\begin{proof}[Proof of Proposition~\ref{prop:gap-two}]
Take $n=T-1$ arms, each with $\theta=1$ with probability
$p=1/(T-1)^2$ and $\theta=0$ otherwise. Probing every arm and
continuing to the cap after success uses $n(1+p(T-1))=T$ plays in
expectation, so $\LP(P)\ge npT$. A policy earns nothing before its
first play of a good arm. If that play is at step $\tau$, it earns at
most $T-\tau+1$. The arms are i.i.d.\ and one play reveals each, so $\tau$ is
at least the position of the first good arm in the order in which the
policy probes arms, a geometric variable with parameter $p$ truncated
at $n$. Hence
\[
 \OPT(P)\le\sum_{k=1}^{n}p(1-p)^{k-1}(T-k+1)\le p\sum_{j=2}^{T}j
 =\frac{p(T-1)(T+2)}2,
\]
and $\OPT(P)/\LP(P)\le(T+2)/(2T)$.
\end{proof}

In the instances of
Proposition~\ref{prop:gap-two}, add the safe alternative of
Section~\ref{sec:weights} with reward $a$ just below the common
density of the probing policies. Then
$\OPT(P)/\LP(P)\to1$, but the optimal LP solution is still the search,
and executing it earns about $\LP(P)/2$.

In the future-utilization problem of \citet[Section~7]{GM13}, the
budget is spent before a terminal choice of $L$ arms, and the factor of
\citet{guha2026future} tends to $1-\Pr(Z_L=L)$ in its non-vanishing
regime. With caps $T/X$, cumulative reward loses the same Poisson term
$\Pr(Z_X=X)$, plus $\Pr(Z_X\ge X)/(2X)$ for the local policy cut at the
horizon. At general caps, exploiting a find uses the clock that finding
it used, and Proposition~\ref{prop:gap-two} makes $1/2$ tight; a
terminal choice takes no time, so future utilization keeps $1-e^{-1}$.

\paragraph{Greedy allocation after protected reports.}
\begin{proposition}[Greedy second phase]\label{prop:greedy-phase2}
Suppose phase~2 of Algorithm~\ref{alg:promiseless} plays the remaining
unpredicted plays in decreasing order of their conditional means given
the reports available at the start of phase~2 (for a slow arm, its
predictive mean at the start of the episode). Put
$X_2=(T-T_{\mathrm p})/\Lbar>0$. The reward bounds of
Theorem~\ref{thm:promise-less} and
Corollaries~\ref{cor:oneshot} and~\ref{cor:probe-transfer} hold with
$c=\rho_{X_2}$, where
\[
 \rho_{X_2}=1-\frac{\E(Z-X_2)^+}{X_2},\qquad Z\sim\operatorname{Poisson}(X_2).
\]
These greedy bounds hold for every $X_2>0$, and
$\rho_{X_2}=1-\Pr(Z=\lfloor X_2\rfloor)$.
The function $\nu\mapsto\rho_\nu$ is continuous and nondecreasing
on $(0,\infty)$. For $X_2\ge1$, $\rho_{X_2}\ge c_{X_2,0}$.
\end{proposition}

\begin{proof}[Proof of Proposition~\ref{prop:greedy-phase2}]
Fix $e$ and put $C_2=T-T_{\mathrm p}$. Let $\mathcal R_i$ be arm $i$'s local
record: its incoming retained state, its coins, and the outcomes
supplied to its policy in phase~1.
It determines the number $M_i\le\Lbar$ of remaining unpredicted plays.
Put $p_i=\E[\mu_i(\theta_i)\mid\mathcal R_i]$. By condition~\ref{item:executor-interleaving} these pairs do
not depend on the interleaving, so they are independent across arms,
and $\E\sum_iM_i\le C_2$. First allocate phase~2 greedily by the $p_i$.
Let $S_t=\sum_iM_i\mathbf 1\{p_i>t\}$. The greedy order makes up to $C_2$
plays with the largest $p_i$, so phase~2 earns
$\int_0^1\min\{C_2,S_t\}\,dt$ in conditional expectation. As in
the proof of Lemma~\ref{lem:execution}, $S_t$ is dominated in convex
order by $\Lbar Z_{\nu_t}$, where $\nu_t=\E S_t/\Lbar\le X_2$. Since
$u\mapsto\E(Z_u-X_2)^+$ is convex and vanishes at $0$,
\[
 \E(S_t-C_2)^+\le\Lbar\,\E(Z_{\nu_t}-X_2)^+
 \le\Lbar\,\frac{\nu_t}{X_2}\,\E(Z_{X_2}-X_2)^+
 =(1-\rho_{X_2})\,\E S_t .
\]
Hence $\E\min\{C_2,S_t\}\ge\rho_{X_2}\E S_t$. Integrating over $t$,
phase~2 earns at least $\rho_{X_2}\E\sum_iM_ip_i$. Supplied
outcomes have been reported, and the learner knows its coins and the
incoming retained states, so the information available at the start of
phase~2 determines every $\mathcal R_i$. Greedy allocation by the
conditional means given this information maximizes the conditional
expected value of phase~2 over allocations determined by it, which
include the allocation above; its expected value is therefore at least
as large. The steps of the phase-2 plays that are not made
are also idle. The
accounting in the proof of Theorem~\ref{thm:promise-less} then gives
\[
 \E V_e\ \ge\ J_e-(1-\rho_{X_2})\,\E\sum_iM_ip_i,
\]
and $\E\sum_iM_ip_i\le\sum_iR^{\mathrm u}$. This is the theorem's
bound with $c$ replaced by $\rho_{X_2}$; the corollaries follow as
before. The argument holds for every $X_2>0$.
For $Z_\nu\sim\operatorname{Poisson}(\nu)$ and $m=\lfloor\nu\rfloor$,
\[
 \E(Z_\nu-\nu)^+=\nu\Pr(Z_\nu=m),\qquad
 \rho_\nu=1-e^{-\nu}\nu^m/m!.
\]
On $(m,m+1)$ its derivative is
$\Pr(Z_\nu=m)(1-m/\nu)\ge0$, and the expressions agree at integers.
Thus $\rho_\nu$ is continuous and nondecreasing.
Finally, $H_X(x)\ge(x-X)^+$ gives $\rho_P(X)\le\rho_X$, and
$\E(Z_X-X)^+=\E|Z_X-X|/2\le\sqrt X/2$ gives $\rho_X\ge1/2$ for $X\ge1$;
hence $\rho_{X_2}\ge c_{X_2,0}$ for $X_2\ge1$.
\end{proof}

\begin{corollary}[Sharp factor with short probes]\label{cor:short-probe-gap}
Fix an integer $X\ge1$. For undelayed one-shot single-episode
instances with caps $L_i\le T/X$, budgets $B_i\ge L_i$ and $n=o(T)$, the greedy variant of
Algorithm~\ref{alg:promiseless}, with $\epsilon=w=0$, $T_{\mathrm p}=n$
and an optimal restricted LP solution, earns
\[
 \E V\ge(\rho_X-o(1))\LP(P).
\]
The factor is worst-case tight relative to $\LP(P)$.
\end{corollary}

\begin{proof}
Put $X_2=(T-n)/\Lbar\ge X(1-n/T)$. Monotonicity and continuity give
$\rho_{X_2}\ge\rho_X-o(1)$. The proof of
Corollary~\ref{cor:oneshot}, using the greedy bound for every
$X_2>0$, gives
$\E V\ge\rho_{X_2}(1-n/T)\LP(P)$, proving the guarantee.
For the instances of Proposition~\ref{prop:gap},
$n=O(X\sqrt T)=o(T)$ and
$\OPT(P)\le(\rho_X+o(1))\LP(P)$. Thus the worst-case ratio
$\OPT/\LP$ in this class is $\rho_X+o(1)$.
\end{proof}

\paragraph{The short-probe condition.}
The restriction to short probes is necessary. Let $T$ be a multiple of
$4$, and take $n=T/4$ one-shot arms with $L_i=L=T/2$, $p=3/(L-1)$ and
$a=0$. Probing every arm and continuing to the cap after success fits in
expectation, $n(1+p(L-1))=T$, and collects every good arm's reward, so
$\LP(P)=npL$. The optimum exhausts each good arm before probing further.
If the second good arm is the $j$th probed, the first contributes $L$
and the second $L-j+2$, so
\[
 \OPT(P)=L\bigl(1-(1-p)^n\bigr)
 +\sum_{j=2}^{n}(L-j+2)(j-1)p^2(1-p)^{j-2},
 \qquad
 \frac{\OPT(P)}{\LP(P)}\to\frac89-\frac{13}{18}e^{-3/2}\approx0.7277,
\]
below $\rho_2=1-2e^{-2}\approx0.7293$. Here probes take a quarter of the
horizon: the time spent finding good arms adds to their competition for
time.

\paragraph{Delays and comparison with GM13.}\label{app:delay-cost}
Table~\ref{tab:gm13} compares the two guarantees.

\begin{table}[htbp]
\centering
\begin{tabular}{@{}lcc@{}}
\toprule
Setting & GM13 & This paper\\
\midrule
Undelayed, any caps & $2$ & $2$\\
Undelayed, caps at most $T/4$ & $2$ & $<1.363$\\
Undelayed, caps at most $T/8$ & $2$ & $<1.211$\\
Undelayed, caps at most $T/16$ & $2$ & $<1.132$\\
Undelayed, caps at most $T/X$, $X\to\infty$ & $2$ & $1+(2\pi X)^{-1/2}+O(X^{-1})$\\
Delay $\delta=\sqrt T$, $T\to\infty$ & $\to66$ & $\to1/c_X$\\
Delay $\delta\log\delta=o(T)$, any caps & $2+32(y+y^2)$ & $\to2$\\
\bottomrule
\end{tabular}
\caption{Approximation factors relative to the single-episode LP; in the
delayed rows, GM13's factor is relative to its delay-respecting
relaxation and ours to the undelayed $\LP(P)$, and $y=\delta/\sqrt T$.
Proposition~\ref{prop:gap-two} shows that the first and last rows are
tight, and Proposition~\ref{prop:gap} that the leading term in the fifth
row is necessary for this relaxation.}
\label{tab:gm13}
\end{table}

Fix $X\ge1$ and use the general-policy prefix bound. Put
\[
 D_\delta=(\delta+1)(1+\ln(\delta+1)),\qquad
 \epsilon=\gamma=\sqrt{5D_\delta/T}<1.
\]
Then $W_i\le5D_\delta/\epsilon=\gamma T$, so
Corollary~\ref{cor:delayed} applies. The function $\rho_P$ is
locally Lipschitz on $(0,\infty)$: on a compact parameter interval,
$H_\nu(z)$ is Lipschitz in $\nu$ with a summable Poisson envelope.
Hence $c_{X,\gamma}=c_{X,0}+O(\gamma)$ at fixed $X$. Since
$c_{X,0}\ge1/2$, the approximation factor is
\[
 \frac{1+\epsilon}{(1-\gamma)c_{X,\gamma}}
 =\frac1{c_{X,0}}+O\bigl(\sqrt{D_\delta/T}\bigr).
\]
An LP approximation with $\epsilon'=O(\epsilon+\gamma)$ has the
same order of additional loss. At $\delta=\sqrt T$, this delay cost
is $O(T^{-1/4}\sqrt{\log T})$ and tends to zero.

GM13's delayed factor is $2(1+\epsilon)+32(y+y^2)$, where
$y=\delta/\sqrt T$ \citep[Theorem~16]{GM13}, relative to its
delay-respecting relaxation; at $\delta=\sqrt T$ it tends to $66$ as
$\epsilon\to0$.

\begin{corollary}[The delay cost vanishes]\label{cor:delay-regret}
Let $\delta\ge1$ be the largest delay of a fast arm, put
$\varepsilon_\delta=\sqrt{5D_\delta/T}$, and let $\varepsilon_\delta<1$.
\textup{(i)} Suppose every unknown arm is fast. The policy of
Corollary~\ref{cor:delayed} with $\epsilon=\gamma=\varepsilon_\delta$, applied to an
optimal solution, earns
\[
 \E V\ \ge\ \frac{1-\varepsilon_\delta}{1+\varepsilon_\delta}\,c_{X,\varepsilon_\delta}\,\LP(P)
 \ \ge\ \frac{(1-\varepsilon_\delta)^2}{2(1+\varepsilon_\delta)}\,\LP(P),
\]
the second bound for arbitrary caps and general bounded rewards. The
factor is $2(1+\varepsilon_\delta)/(1-\varepsilon_\delta)^2=2+O(\sqrt{D_\delta/T})$, and its limit
$2$ is tight by Proposition~\ref{prop:gap-two}.
\textup{(ii)} Run Algorithm~\ref{alg:learner} with
$\epsilon=\varepsilon_\delta\le1$, $w=5D_\delta/\epsilon$ and the uniform
reserve $s<T-w$ of Theorem~\ref{thm:meta}. Then
\[
 \E\sum_eV_e\ \ge\ \LP_N(T,B')-N\Bigl(\sqrt{2\Lbar T\Lambda_{\mathrm{exec}}}
 +\tfrac23\Lbar\Lambda_{\mathrm{exec}}+2\sqrt{5D_\delta T}\Bigr)-1,
\]
where $\LP_N(T,B')\ge\OPT_N(B')$. For $N=1$ and $B_i'\ge L_i$,
$\LP_N(T,B')\ge\OPT_\delta(P)$. It equals $\LP(P)$ when every unknown
arm is fast; with slow arms it is the single-episode LP that restricts
them to nonadaptive local policies.
\end{corollary}

\begin{proof}
(i) Corollary~\ref{cor:delayed} applies because
$W_i\le5D_\delta/\varepsilon_\delta=\varepsilon_\delta T$. It gives the first bound, and
$c_{X,\varepsilon_\delta}\ge(1-\varepsilon_\delta)/2$ gives the second.
(ii) Lemma~\ref{lem:buffered-executor} bounds every prefix by $w$. For an optimal
solution, Proposition~\ref{prop:relax}(ii) gives
$F\ge(C/T)\LP_N(T,B')$, and $C/T\ge1-(w+s)/T-\epsilon$. Since
$\LP_N(T,B')\le TN$,
Theorem~\ref{thm:meta} gives
$\E\sum_eV_e\ge\LP_N(T,B')-N(w+s+\epsilon T)-1$. The choice of
$\epsilon$ makes $w+\epsilon T=2\sqrt{5D_\delta T}$.
\end{proof}

\subsection{Several plays per step}\label{app:kplays}

We prove Theorem~\ref{thm:k-lanes} and its dependence on $K$.
Each arm remains serial, while different arms share $K$ lanes.

\begin{proof}[Proof of Theorem~\ref{thm:k-lanes}]
If $W=0$, then $J=0$ and the claim is immediate. Otherwise discard
policies with zero workload and index the rest in priority
order. Let $N_j$ be the plays of policy $j$ when run alone,
$\tau_j=\E N_j$, $R_j=R(\pi_j)$ and $S_i=\sum_{j<i}N_j$.

\emph{Blocking.} Higher-priority policies do not depend on policy $i$.
At a step where policy $i$ has not stopped but is not played, $K$
higher-priority policies are played; such steps use $K$ of at most
$S_i$ higher-priority plays, so there are at most $S_i/K$ of them.
Policy $i$ is played at every other step until it stops. Its executed
plays therefore satisfy
\[
 M_i\ge\min\{N_i,(T-S_i/K)^+\}\ge\phi_iN_i,\qquad
 \phi_i=\min\{1,(T-S_i/K)^+/L\}.
\]
The factor $\phi_i$ is determined by the tapes and coins of policies
$j<i$, which are independent of policy $i$, so
Lemma~\ref{lem:truncation} gives expected reward at least
$R_i\E\phi_i$.

\emph{Triangle bound.} Since $L\le T$, $\phi_i\ge1-S_i/(KT)$ and
$\E\phi_i\ge1-\sum_{j<i}\tau_j/(KT)$. The densities $R_i/\tau_i$ are
nonincreasing and $\sum_{j<i}\tau_j$ is nondecreasing in $i$, so
Chebyshev's sum inequality with weights $\tau_i$ gives
\[
 \sum_iR_i\sum_{j<i}\tau_j
 \le\frac JW\sum_i\tau_i\sum_{j<i}\tau_j\le\frac{JW}2 .
\]
Hence $\sum_iR_i\E\phi_i\ge(1-W/(2KT))J$.

\emph{Poisson bound.} Write $1-\phi_i=h(S_i/K)$ with
$h(s)=\min\{1,(s-T+L)^+/L\}$, and put $H(x)=\int_0^xh$ and
$S=\sum_iN_i$. Since $h$ is nondecreasing and $S_{i+1}=S_i+N_i$,
\[
 \sum_iN_ih(S_i/K)\le K\sum_i\int_{S_i/K}^{S_{i+1}/K}h(y)\,dy
 =KH(S/K).
\]
By independence of $N_i$ and $S_i$,
$\sum_i\tau_i\E h(S_i/K)\le K\E H(S/K)$. As in the proof of
Lemma~\ref{lem:execution}, $S\le_{\rm cx}LZ_{\nu_0}$ with
$\nu_0=W/L\le K\nu$. The map $x\mapsto H(x/K)$ is convex and
$H(Lx)=LH_\nu(x)$, so $\E H(S/K)\le L\,\psi(\nu_0)$ with
$\psi(u)=\E H_\nu(Z_u/K)$. The second derivative of $\psi$ is the
expected second difference of the convex function $z\mapsto H_\nu(z/K)$
at $Z_u$, so $\psi$ is convex with $\psi(0)=0$, and
$\psi(\nu_0)\le(\nu_0/K\nu)\,\psi(K\nu)$. Therefore
\[
 \sum_i\tau_i\E h(S_i/K)\le\frac{L\nu_0}{\nu}\,\psi(K\nu)
 =\bigl(1-\rho_{P,K}(\nu)\bigr)W .
\]
Both $R_i/\tau_i$ and $1-\E h(S_i/K)$ are nonincreasing in $i$, and
Chebyshev's sum inequality gives
$\sum_iR_i\bigl(1-\E h(S_i/K)\bigr)\ge\rho_{P,K}(\nu)J$.
\end{proof}

\paragraph{Dependence on $K$.}
Let $\nu\ge1$.
\textup{(i)} $\rho_{P,K}(\nu)$ is nondecreasing in $K$.
\textup{(ii)} For every $K$,
\[
 1-\frac1{2\nu}-\frac{\E(Z_{K\nu}-K\nu)^+}{K\nu}
 \ \le\ \rho_{P,K}(\nu)\ \le\ 1-\frac1{2\nu}.
\]
The subtracted term is at most $(4K\nu)^{-1/2}$, and equals
$\Pr(Z_{K\nu}=K\nu)\le(2\pi K\nu)^{-1/2}$ when $K\nu$ is an integer.
In particular $\rho_{P,K}(\nu)\to1-1/(2\nu)=1-L/(2T)$ as $K\to\infty$.

Let $X_1,X_2,\ldots$ be i.i.d.\ $\operatorname{Poisson}(\nu)$ and
$\bar X_K=K^{-1}\sum_{k\le K}X_k$; then $Z_{K\nu}/K$ has the law of
$\bar X_K$.
(i) For $k\le K+1$ let $\bar X_K^{(k)}$ be the mean of
$X_1,\ldots,X_{K+1}$ without $X_k$. By exchangeability the conditional
expectations $\E[\bar X_K^{(k)}\mid\bar X_{K+1}]$ coincide, and their
average is $\bar X_{K+1}$. Conditional Jensen gives
$\E f(\bar X_{K+1})\le\E f(\bar X_K^{(1)})=\E f(\bar X_K)$ for convex
$f$; take $f=H_\nu$.
(ii) Since $\nu\ge1$, $H_\nu(\nu)=\int_{\nu-1}^\nu(y-\nu+1)\,dy=\frac12$,
and Jensen gives $\E H_\nu(\bar X_K)\ge\frac12$, the upper bound. Also
$H_\nu(x)\le\frac12+(x-\nu)^+$: for $x\le\nu$ because $H_\nu$ is
nondecreasing, and for $x\ge\nu$ with equality. Hence
$\E H_\nu(\bar X_K)\le\frac12+\E(Z_{K\nu}-K\nu)^+/K$, the lower bound.
Finally $\E(Z_m-m)^+=\frac12\E|Z_m-m|\le\frac12\sqrt m$, and for integer
$m$, $\E(Z_m-m)^+=m\Pr(Z_m=m)\le m(2\pi m)^{-1/2}$.

\begin{example}[Finding time with $K$ plays per step]\label{ex:k-finding-time}
For every $K\ge1$ and integers $1\le L\le T$ with $KT\ge2$ there are undelayed one-shot
instances with caps $L$ and
\[
 \frac{\OPT_K(P)}{\LP_K(P)}
 \le\Bigl(1-\frac{L-1}{2T}\Bigr)\frac{KT}{KT-1}.
\]

Take $n=KT-1$ arms whose episode mean is $1$ with probability $p\in(0,1)$ and
$0$ otherwise, where $p(L-1)\le1/(KT-1)$. Probing every arm and
continuing to the cap after a success uses $n(1+p(L-1))\le KT$ plays in
expectation, so $\LP_K(P)\ge npL$. Fix an admissible policy and let
$t_i$ be the step of arm $i$'s first play ($t_i=\infty$ if none). That
decision uses no outcome of arm $i$, so arm $i$'s mean is independent
of $t_i$, and arm $i$ earns at most its mean times
$\min\{L,(T-t_i+1)^+\}$. At most $K$ arms are first played at any step,
and $\min\{L,T-t+1\}$ is nonincreasing in $t$, so
\[
 \OPT_K(P)\le pK\sum_{t=1}^T\min\{L,T-t+1\}
 =pKL\Bigl(T-\frac{L-1}2\Bigr).
\]
Dividing by $npL$ proves the claim.
\end{example}

As $T\to\infty$ with $L/T\to x\in(0,1]$, this bound tends to $1-x/2$,
matching the many-lane limit above. The uniform factor of
Theorem~\ref{thm:k-lanes} is therefore tight relative to $\LP_K$
in the limit of many lanes. For each fixed $K$, taking $L=T\to\infty$
makes the displayed upper bound tend to $1/2$; the theorem attains
that uniform factor for every $K$. For $K=1$ and $L=T$, this recovers
the limiting factor $1/2$ of Proposition~\ref{prop:gap-two}. The common value
$1-x/2=\int_0^1\min\{1,(1-s)/x\}\,ds$ is the fraction a job of length
$xT$ keeps when it starts at a uniform time: GM13's triangle, with the
job's own length in place of the work ahead of it.

\section{Statistical bounds from reported experiments}\label{sec:drift}
\label{app:moving}

We treat two drift models: a known affine path between unknown endpoint
priors (Theorem~\ref{thm:affine}), and general drift that is slow in
the laws of short probes.
Throughout this appendix, rewards have finite alphabets $\mathcal Y_i$
of sizes $d_i$. Fix caps $1\le L_i\le\Lbar$ and adaptation
depths $k_i$ across episodes. Slow arms have $k_i=0$.
Write $k=\max_i k_i$ and use $\LP^{(k)}$ for these arm-specific
classes. Put $r_i=\min\{L_i,k_i+1\}$. The comparison budgets do not
bind: $B_i\ge NL_i$. The drift law is known, as in
Assumption~\ref{ass:drift}. The realized priors are learned from
reports. Earlier reports arrive only at episode starts, with lag at
most $g$.

The supplied single-episode
planner has a uniform completion certificate $\kappa\LP^{(k)}(P)$ for every
product prior $P$. Theorem~\ref{thm:budget} and
Corollary~\ref{cor:delayed} supply such certificates. Any valid
single-episode factor can be used. Restricting a
slow arm to nonadaptive local policies is essential to that
certificate. The oracle knows the current priors, not the realized
qualities $\theta_i^{(e)}$.

\subsection{Stationary comparisons}\label{app:stationary-comparisons}

\begin{proof}[Proof of Corollary~\ref{cor:episode-lp-rate}]

Let $\hat{\mathbf q}_{e,i}$ be the empirical law of arm $i$'s probe
histograms from episodes $1,\ldots,(e-g)^+$, using a fixed probability
vector if there are none. All these reports are available.
Conditional on $P$, arm $i$'s probe histograms are independent with
law $\mathbf q_{r_i}(P_i)$. For $m>0$ samples, the coordinate variances
sum to at most $1/m$, giving expected total-variation error at most
$\tfrac12\sqrt{D_i/m}$. Among $h$ earlier episodes whose reports have
arrived, the sample count is $M\sim\operatorname{Binomial}(h,p)$.
Including $M=0$ and using
$\E[1/(M+1)]\le1/[p(h+1)]$ bounds the error by
$\sqrt{D_i/[p(h+1)]}$. Summing over episodes, with
$\sum_{h=0}^{N-1}(h+1)^{-1/2}\le2\sqrt N$, gives
\[
 \sum_{e,i}L_i\E
 \|\mathbf q_{r_i}(P_i)-\hat{\mathbf q}_{e,i}\|_{\rm TV}
 \le2C_{\rm est}\sqrt{N/\xi}+(g-1)L_\Sigma.
\]
The planning branch has the conditional floor
$(1-\xi)\kappa\LP^{(k)}(\Pb_e)$.
Apply Corollary~\ref{thm:learning-stability}, and charge at most
$\kappa\xi NT$ to restore the factor $\kappa$. The resulting bound is
\[
 \kappa\left[\xi NT+4C_{\rm est}\sqrt{N/\xi}
                   +2(g-1)L_\Sigma\right].
\]
The stated choice of $\xi$ gives \eqref{eq:episode-lp-rate} when
unclipped. If clipping occurs, $(TC_{\rm est}^2)^{1/3}N^{2/3}>NT/2$,
so the trivial bound $\kappa NT$ suffices.

For the protected learner, use the same calculation with one probe
group, $C_{\rm est}=\sum_iL_i\sqrt{D_i}$.
Theorem~\ref{thm:verified-learning} gives
\begin{equation}\label{eq:stationary-verified-rate}
 \kappa_{\rm pl}\sum_e\E\LP^{(k)}(P)-V_A
 \le5\kappa_{\rm pl}(TC_{\rm est}^2)^{1/3}N^{2/3}
       +2\kappa_{\rm pl}(g-1)L_\Sigma.
\end{equation}
Use the empirical probe estimates above, bound $1-\xi$ by one in the
error term, charge $\kappa_{\rm pl}\xi NT$ to restore the benchmark,
and optimize $\xi$ as before. A planning gap $\zeta$ adds $c\zeta$.
\end{proof}

\subsection{Reported probes and estimator comparisons}

Use Algorithm~\ref{alg:probe}, with probe probability
$p=\xi/J_{\rm grp}$ and constants from
\eqref{eq:estimation-coefficient}. The following bounds specialize
Corollary~\ref{thm:learning-stability} to reported probe data.

For general drift, define
\begin{equation}\label{eq:variation-budget}
 \mathcal V=\E\sum_{e=1}^{N-1}\sum_i L_i\,
 d_{r_i}\bigl(P_i^{(e+1)},P_i^{(e)}\bigr).
\end{equation}
We call the drift \emph{slow} when
$\mathcal V=o(N)$, with the other parameters fixed. This definition
allows occasional jumps. Changes that preserve the $r_i$-play probe law contribute zero.
For a slow arm ($\delta_i\ge T$), $r_i=1$ and the one-reward law
controls its current reward and workload. Longer reported batches can still reveal features of the
recurring law that matter to later decisions.

Write $\Delta_{i,u}=d_{r_i}(P_i^{(u+1)},P_i^{(u)})$, so
$\mathcal V=\sum_iL_i\sum_{u<N}\E\Delta_{i,u}$. The proofs below use
the history factorization and conditional predictive prior from
Section~\ref{sec:learn}. The experiment law and its planning-value bound are
\eqref{eq:experiment-distance} and Corollary~\ref{thm:learning-stability}.

Put $\kappa_\xi=(1-\xi)\kappa$ and define the scaled comparison loss
\[
 \Reg_N=\E\sum_{e=1}^N
 \bigl(\kappa_\xi\LP^{(k)}(P^{(e)})-V_e\bigr).
\]
This is an approximation regret in the sense of \citet{chen2013cmab}.
All expectations average under the meta-prior and the policy's coins.

\begin{theorem}[Floor and comparison with any estimator]\label{thm:risk}
Algorithm~\ref{alg:probe} respects every cap and budget, and in every episode
$\E[V_e\mid H_{e-1}]\ge\kappa_\xi\LP^{(k)}(\Pb_e)$. For any
history-measurable vectors $\hat{\mathbf q}_{e,i}\in\mathbb R^{D_i}$,
\[
 \Reg_N\le2\kappa_\xi\sum_iL_i\sum_{e=1}^N
 \E\big\|\mathbf q_{r_i}(P_i^{(e)})-\hat{\mathbf q}_{e,i}\big\|_{\rm TV}.
\]
The estimators need not correspond to realizable priors.
\end{theorem}
\begin{proof}
Both branches respect the caps and report calendar, and comparison
budgets do not bind. The planning branch and Theorem~\ref{thm:floor}
give the conditional floor. Apply
Corollary~\ref{thm:learning-stability} to the available history, multiply
by $\kappa_\xi$ and sum over episodes.
\end{proof}

\subsection{Estimating moving experiment laws}

\begin{lemma}[Windowed estimate]\label{lem:window}
Fix an integer $1\le\tau\le N$.
In episode $e$, use the window
$I_e=\{\max\{1,e-g-\tau+1\},\dots,e-g\}$ of
$h_e=|I_e|=\min\{\tau,(e-g)^+\}$ episodes. The window is empty
when $e\le g$. Let
$M_{e,i}$ be the number of probes of arm $i$ in $I_e$. Let
$\hat{\mathbf q}_{e,i}$ be the empirical distribution of their histograms,
or a fixed probability vector if $M_{e,i}=0$. These quantities are
history-measurable. For $e\ge g$,
\[
 \E\big\|\mathbf q_{r_i}(P_i^{(e)})-\hat{\mathbf q}_{e,i}\big\|_{\rm TV}
 \le\sqrt{\frac{D_i}{p\,(h_e+1)}}+b_{e,i},
\qquad
 b_{e,i}=\E\max_{t\in I_e}d_{r_i}(P_i^{(t)},P_i^{(e)}).
\]
Here $b_{e,i}=0$ when $I_e$ is empty, and
$b_{e,i}\le\E\sum_{u=\max\{1,e-g-\tau+1\}}^{e-1}\Delta_{i,u}$.
\end{lemma}
\begin{proof}
Condition on the true prior sequences and probe times. The episode
qualities are conditionally independent, and the probes are nonadaptive.
The probe histograms are therefore independent with categorical laws
$\mathbf q_{r_i}(P_i^{(t)})$.
Suppose $M_{e,i}=m>0$. The coordinate variances of the empirical
distribution sum to at most $1/m$, so by Cauchy--Schwarz it differs from
the average of these laws by at most $\frac12\sqrt{D_i/m}$ in
expected total variation. The average differs from
$\mathbf q_{r_i}(P_i^{(e)})$ by at most
$\max_{t\in I_e}d_{r_i}(P_i^{(t)},P_i^{(e)})$.
Including $m=0$, where the error is at most $1$, the statistical term is
at most $\sqrt{D_i/(m+1)}$. Each episode probes arm $i$ with probability $p$, independently
of the sequences. Thus $M_{e,i}\sim\mathrm{Binomial}(h_e,p)$.
The first bound follows from Jensen's inequality and the identity
$\E\frac1{M_{e,i}+1}=\frac{1-(1-p)^{h_e+1}}{p(h_e+1)}\le\frac1{p(h_e+1)}$. For $t\in I_e$ the triangle inequality gives
$d_{r_i}(P_i^{(t)},P_i^{(e)})\le\sum_{u=t}^{e-1}\Delta_{i,u}$, and
$t\ge e-g-\tau+1$.
\end{proof}

\begin{theorem}[Tracking with a reward floor]\label{thm:variation}\label{thm:rates-summary}
Under the conditions above, use the execution factor $\kappa$
and the scaled comparison $\Reg_N$ defined above.
Algorithm~\ref{alg:probe} satisfies
\[
 \E[V_e\mid H_{e-1}]\ge\kappa_\xi\LP^{(k)}(\Pb_e)
\]
in every episode. For every integer $1\le\tau\le N$,
\[
 \Reg_N\le\kappa_\xi c_\xi
 \left[\frac{N}{\sqrt\tau}+\tau\mathcal V
 +(g-1)(1+\mathcal V)\right],
\]
where
$c_\xi=2\max\{2C_{\mathrm{est}}/\sqrt\xi,L_\Sigma,1\}$.
Thus $c_\xi=O(\xi^{-1/2})$. Its other dependence is explicit; it
does not depend on $N$, $\mathcal V$, $g$ or $\tau$.
\end{theorem}

\begin{proof}[Proof of Theorem~\ref{thm:variation}]
The floor is Theorem~\ref{thm:risk}. Use $\hat{\mathbf q}_{e,i}$ of
Lemma~\ref{lem:window} in Theorem~\ref{thm:risk}. The first $g-1$
episodes contribute at most $1$ each to the sum of total variations,
hence $(g-1)L_\Sigma$ after weighting. For the others put $e'=e-g+1$, so
that $h_e=\min\{\tau,(e'-1)^+\}$ and
\[
 \sum_{e'=1}^N\frac1{\sqrt{h_e+1}}\le2\sqrt\tau+\frac{N-\tau}{\sqrt{\tau+1}}
 \le\frac{2N}{\sqrt\tau},
\]
which with $p=\xi/J_{\mathrm{grp}}$ bounds the statistical terms by
$2NC_{\mathrm{est}}/\sqrt{\xi\tau}$ after weighting by $L_i$. For the bias, each
increment $\Delta_{i,u}$ appears in $b_{e,i}$ for at most $\tau+g-1$
values of $e$, so the weighted bias sums to at most $(\tau+g-1)\mathcal V$.
Multiplying by $2\kappa_\xi$ gives the explicit window bound
\begin{equation}\label{eq:window-rate}
  \Reg_N\le2\kappa_\xi\Bigl[ \frac{2N C_{\mathrm{est}}}{\sqrt{\xi\tau}}+\tau\mathcal V
 +(g-1)(L_\Sigma+\mathcal V)\Bigr].
\end{equation}
Taking
$c_\xi=2\max\{2C_{\mathrm{est}}/\sqrt\xi,L_\Sigma,1\}$ proves
Theorem~\ref{thm:variation} in its stated form.

Optimizing over $\tau$, with
$\tau\asymp(NC_{\mathrm{est}}/(\sqrt\xi\,\mathcal V))^{2/3}$ rounded and
clipped to $[1,N]$, and $\tau=N$ when $\mathcal V=0$, gives
\[
 \Reg_N=O\Bigl(\kappa_\xi\Bigl[C_{\mathrm{est}}\sqrt{N/\xi}
 +\bigl(N C_{\mathrm{est}}/\sqrt\xi\bigr)^{2/3}\mathcal V^{1/3}
 +(g-1)(L_\Sigma+\mathcal V)\Bigr]\Bigr).
\]
The lower endpoint never binds, since
$\mathcal V\le(N-1)L_\Sigma\le NC_{\mathrm{est}}/\sqrt\xi$; when the
upper endpoint binds, the square-root term covers the bound with
$\tau=N$.
\end{proof}

The window appears only in the proof; the learner uses its full
posterior and needs neither a window nor a variation budget. For fixed
parameters and $\xi$, the bound is $O(\sqrt N)$ under stationarity,
$O(N^{2/3})$ for bounded $\mathcal V$, and sublinear whenever
$\mathcal V=o(N)$.

\paragraph{Known affine evolution.}

Local averaging incurs first-order error even when drift has a
stable direction. A local linear estimate removes this error. We use
that estimate only in the analysis. Algorithm~\ref{alg:probe} continues to use its
full posterior under the known drift law. The realized endpoints
or trajectory remain unknown. Take $g=1$ and
$2\le\tau\le N$.
For $h=\min\{\tau,e-1\}\ge2$ give the last $h$ episodes, of ages
$\ell=1,\dots,h$, the weights for linear extrapolation
$\upsilon_{h,\ell}=(4h+2-6\ell)/(h(h-1))$, which
satisfy
\[
 \sum_\ell \upsilon_{h,\ell}=1,\qquad\sum_\ell\ell\,\upsilon_{h,\ell}=0,\qquad
 \sum_\ell \upsilon_{h,\ell}^2=\frac{2(2h+1)}{h(h-1)}\le\frac{10}h .
\]
For $h<2$, use a fixed probability vector. Let $A_{t,i}$ indicate a
probe of arm $i$ in episode $t$ and $S_{t,i}$ its outcome histogram,
set to an arbitrary fixed histogram when there is no probe. For $h\ge2$, set
$\hat{\mathbf q}_{e,i}=\sum_{\ell=1}^h\upsilon_{h,\ell}\,(A_{e-\ell,i}/p)\,\mathbf e_{S_{e-\ell,i}}$,
with $\mathbf e_s$ a coordinate vector. The estimate need not be a probability vector.
Theorem~\ref{thm:risk} permits such centers.
Given the sequence of priors, its mean is
$\sum_\ell \upsilon_{h,\ell}\mathbf q_{r_i}(P_i^{(e-\ell)})$. One
inverse-probability weighted categorical vector has summed coordinate
variances at most $1/p$, and the episodes are independent given the
sequence. Hence
\[
 \E\big\|\hat{\mathbf q}_{e,i}-\E[\hat{\mathbf q}_{e,i}\mid(P_i^{(t)})_t]\big\|_{\rm TV}
 \le\tfrac12\sqrt{\tfrac{D_i}p\textstyle\sum_\ell \upsilon_{h,\ell}^2}.
\]

\begin{theorem}[Unknown affine mixtures]\label{thm:affine}
Suppose that for each arm there are unknown random endpoint priors
$P_i^{\mathrm L},P_i^{\mathrm R}$ such that
$P_i^{(e)}=(1-e/N)P_i^{\mathrm L}+(e/N)P_i^{\mathrm R}$, with a known
meta-prior on the endpoint pairs, independent across arms. For $g=1$ and
$N\ge2$,
\[
 \Reg_N\le4\kappa_\xi L_\Sigma+2\sqrt{10}\,\kappa_\xi\sqrt{N/\xi}\;C_{\mathrm{est}} .
\]
In particular, for $J_{\mathrm{grp}}=1$ this is $O(\Lbar n\sqrt{D_{\max}N/\xi}+\Lbar n)$, where $D_{\max}=\max_iD_i$.
\end{theorem}
\begin{proof}
By linearity,
$\mathbf q_{r_i}(P_i^{(e)})=(1-e/N)\mathbf q_{r_i}(P_i^{\mathrm L})+(e/N)\mathbf q_{r_i}(P_i^{\mathrm R})$,
and the first two identities for the weights make the estimator unbiased
for the current outcome law, whatever the endpoint priors. Take
$\tau=N$, bound the first two episodes by total variation one each, and
use the variance bound above. For $e\ge3$, $h=e-1$ and
$\sum_{h=2}^{N-1}h^{-1/2}\le2\sqrt N$, which gives a statistical sum of
at most $\sqrt{10D_iN/p}$ per arm. Multiply by
$2\kappa_\xi L_i$ as in Theorem~\ref{thm:risk} and sum over arms.
\end{proof}

The priors may travel a nonvanishing distance; the rate comes from the
known linear form of the probe laws, not from a vanishing total
distance.

\subsection{Transfer bounds}

Algorithm~\ref{alg:learner} plans for cumulative reward. Its Bayes-optimal comparison
value is at least Algorithm~\ref{alg:probe}'s reward for every $\xi$. To restore
the benchmark factor from $\kappa_\xi$ to $\kappa$, charge at most
$\xi NT$ for exploration. Optimize this charge together with the
estimation and drift costs.

\begin{corollary}[The learning-LP learner under slow drift]\label{cor:fullscale}
Run Algorithm~\ref{alg:learner} with an optimal solution. Suppose the uniform
completion certificate of Theorem~\ref{thm:meta} holds,
$B_i'\ge NL_i$, and
$1-\omega=C/T>0$ as in Theorem~\ref{thm:meta}. Its expected cumulative regret against
$(1-\omega)\kappa\LP^{(k)}(P^{(e)})$ satisfies
\[
 \mathcal R_N^{(1)}\le(1-\omega)\kappa c_{\mathrm{inst}}
 \left[N^{2/3}+\mathcal V^{1/4}N^{3/4}\right]+1.
\]
Constant $c_{\mathrm{inst}}$ depends on horizon, arms, caps,
depths and report lag. It does not depend on $N$ or $\mathcal V$.
For the affine model with $g=1$, the bound improves to
$O(N^{2/3})+1$ for fixed parameters.
\end{corollary}

For every fixed $\xi$ the probing learner is admissible for the budgets
$B'$, so $\OPT_N(B')\ge\E\sum_eV_e^{\text{probe}}$, and with
Theorem~\ref{thm:meta} a regret bound $R_N$ for it gives
\[
 \E\sum_e\Big((1-\omega)\kappa_\xi\LP^{(k)}(P^{(e)})-V_e^{\text{Alg.\,\ref{alg:learner}}}\Big)
 \le(1-\omega)R_N+1 .
\]
Algorithm~\ref{alg:learner} is not modified to use probes or windows.
With
$\mathcal R_N^{(1)}=\E\sum_e\bigl((1-\omega)\kappa\LP^{(k)}(P^{(e)})-V_e\bigr)$,
the exploration/window tradeoff underlying Corollary~\ref{cor:fullscale} is
\begin{equation}\label{eq:fullscale-bound}
  \mathcal R_N^{(1)}\le(1-\omega)\kappa \inf_{\substack{0<\xi<1\\1\le\tau\le N}}
 \Bigl\{\xi NT+2(g-1)L_\Sigma +\frac{4NC_{\mathrm{est}}}{\sqrt{\xi\tau}}
 +2(\tau+g-1)\mathcal V\Bigr\}+1,
\end{equation}
where $\tau$ is an integer. The coefficients remain explicit here; the corollary absorbs them into constants to display the rates.

\begin{proof}[Proof of Corollary~\ref{cor:fullscale}]
The relaxation satisfies $\LP^{(k)}(P)\le T$, since expected reward is
at most expected plays. Restoring
the benchmark from $\kappa_\xi\LP^{(k)}$ to $\kappa\LP^{(k)}$ therefore
costs at most $\kappa\xi NT$. Add this charge to
\eqref{eq:window-rate}, bound $1-\xi$ by one in its nonnegative upper
bound, and transfer through Theorem~\ref{thm:meta}. This proves
\eqref{eq:fullscale-bound}. Every fixed $\xi$ defines an admissible
comparison learner, so the infimum changes only the analysis.

The term $\xi NT$ charges for exploration, and
$N C_{\mathrm{est}}/\sqrt{\xi\tau}$ measures sampling error.
Drift over the window costs $\tau\mathcal V$. The remaining
terms account for report lag. Optimizing the window as above leaves,
apart from lag terms,
\[
 O\left(\xi NT+C_{\mathrm{est}}\sqrt{N/\xi}
 +N^{2/3}C_{\mathrm{est}}^{2/3}\mathcal V^{1/3}\xi^{-1/3}\right).
\]
Define the two exploration scales
\[
 \xi_{\mathrm{stat}}=\left(\frac{C_{\mathrm{est}}^2}{T^2N}\right)^{1/3},
 \qquad
 \xi_{\mathrm{drift}}=\left(\frac{C_{\mathrm{est}}^2\mathcal V}{T^3N}\right)^{1/4},
\]
and take $\xi=\min\{1/2,\max\{\xi_{\mathrm{stat}},\xi_{\mathrm{drift}}\}\}$.
Without clipping, each sampling or drift term is at most $\xi NT$.
The cost is therefore
$O((T C_{\mathrm{est}}^2)^{1/3}N^{2/3}
+(T C_{\mathrm{est}}^2)^{1/4}\mathcal V^{1/4}N^{3/4})$.
If clipping is active, this expression is at least $NT/2$, so the
trivial regret bound suffices.
Consequently, for a universal constant $c$,
\begin{equation}\label{eq:fullscale-parameters}
  \mathcal R_N^{(1)}\le(1-\omega)\kappa c\Bigl[ (T C_{\mathrm{est}}^2)^{1/3}N^{2/3} +(T
 C_{\mathrm{est}}^2)^{1/4}\mathcal V^{1/4}N^{3/4} +(g-1)(L_\Sigma+\mathcal V)\Bigr]+1.
\end{equation}

To obtain the stated form, use $0\le\mathcal V\le(N-1)L_\Sigma$ and
$N\ge1$, so that $(g-1)L_\Sigma\le(g-1)L_\Sigma N^{2/3}$ and
$(g-1)\mathcal V\le(g-1)L_\Sigma^{3/4}\mathcal V^{1/4}N^{3/4}$. All
coefficients in \eqref{eq:fullscale-parameters} are then absorbed into
$c_{\mathrm{inst}}$, which depends only on the fixed horizon, arms,
caps, depths and lag. For a feasible solution with objective $F$, the
feasible-solution form of Theorem~\ref{thm:meta} loses at most
$\LP_N(C,B')-F$ relative to this comparison.
\end{proof}

\paragraph{Affine transfer.}
Under Theorem~\ref{thm:affine}'s conditions and the completion
certificate of Theorem~\ref{thm:meta}, suppose $B_i'\ge NL_i$ and
$\omega<1$. Comparing the LP learner with Algorithm~\ref{alg:probe} gives
\[
 \mathcal R_N^{(1)}\le(1-\omega)\kappa
 \inf_{0<\xi<1}\left\{
 \xi NT+4L_\Sigma+2\sqrt{10}\,C_{\mathrm{est}}\sqrt{N/\xi}
 \right\}+1.
\]
Taking $\xi=\min\{1/2,(C_{\mathrm{est}}^2/(T^2N))^{1/3}\}$ gives
$O((1-\omega)\kappa[(TC_{\mathrm{est}}^2)^{1/3}N^{2/3}+L_\Sigma])+1$; if
clipping is active, the displayed rate is of order $NT$ and the trivial
bound suffices. The $N^{2/3}$ term pays for exploration while restoring
$\kappa$; it is an upper bound for this comparison, not a lower bound.

\paragraph{Protected-report execution.}

Algorithm~\ref{alg:promiseless} needs no reserve. Its comparison runs
through a planned version of the probing policy that learns only from
the first $r_i$ plays of each arm in each episode.

\begin{corollary}[Rates without a reserve]\label{cor:probe-transfer}
Assume the conditions of this appendix, with $B_i'\ge NL_i$. Take
prediction depths $r_{e,i}=r_i$ and a window with
$w+(1+\epsilon)\sum_ir_i\le T_{\mathrm p}<T$ and $X_2=(T-T_{\mathrm p})/\Lbar\ge1$. Suppose
Algorithm~\ref{alg:promiseless} uses an optimal solution. Put
\[
 \kappa_{\mathrm{pl}}
 =c_{X_2,0}\Bigl(1-\frac{T_{\mathrm p}}T\Bigr)>0.
\]
Then, for every $N$,
\[
 \E\sum_e\Bigl(\kappa_{\mathrm{pl}}\LP^{(k)}(P^{(e)})-V_e\Bigr)
 \le\kappa_{\mathrm{pl}}\inf_{\substack{0<\xi<1\\1\le\tau\le N}}
 \Bigl\{\xi NT+2(g-1)L_\Sigma+\frac{4NC_{\mathrm{est}}}{\sqrt{\xi\tau}}
 +2(\tau+g-1)\mathcal V\Bigr\},
\]
with $C_{\mathrm{est}}$ from \eqref{eq:estimation-coefficient} at
$J_{\mathrm{grp}}=1$. This is \eqref{eq:fullscale-bound} with
$(1-\omega)\kappa$ replaced by $\kappa_{\mathrm{pl}}$ and without the
additive $1$. The rates of Corollary~\ref{cor:fullscale}, including the
affine case, therefore hold for every $N$, with $\kappa_{\mathrm{pl}}$
in place of $(1-\omega)\kappa$.
\end{corollary}

\begin{proof}
Use Theorem~\ref{thm:verified-learning} with $\zeta=0$.
The window estimates use probe outcomes only, which are predicted
plays of the planned comparison. With $J_{\rm grp}=1$, the calculation
of \eqref{eq:window-rate} gives
\[
 2\mathcal E_N(\xi)\le
 2(g-1)L_\Sigma+\frac{4NC_{\rm est}}{\sqrt{\xi\tau}}
 +2(\tau+g-1)\mathcal V.
\]
Charge at most $\xi NT$ for restoring the factor $1-\xi$ in
\eqref{eq:verified-learning} and take the infimum over $\xi,\tau$.
The affine estimate gives the corresponding affine bound in the same way.
\end{proof}

\section{Proofs for the price of prior uncertainty}\label{app:oracle-price}

The information-price bounds use exact episode optimization to compare
with $V_{\rm prior}$. The posterior-sampling argument below also
applies to a fixed approximate planner and executor. It compares with
that reference policy's value, using a quadratic information bound
and a telescoping entropy budget. Whole-tape histogram dimensions
control the coefficient.

\begin{lemma}[Sampling a reference episode policy]\label{lem:policy-sampling}
Assume the stationary conditions of Proposition~\ref{prop:ps-price}.
Let $\mathcal A(P,U)$ be a fixed map to deterministic feasible
episode policies, with $U$ independent of $P$, and put
$\pi^{\mathcal A}=\mathcal A(P,U_0)$ for an independent reference
seed. Use $g$ threads, each drawing $\widetilde P$ from its own
posterior and executing $\mathcal A(\widetilde P,U_e)$ with fresh
coins. With $D_{\rm tape}=\prod_i\binom{L_i+d_i-1}{d_i-1}$, its reward satisfies
\[
 N\,\E V^P(\pi^{\mathcal A})-V_{\mathcal A}^{\rm PS}
 \le\min\{T,L_\Sigma\}\sqrt{\tfrac12\,gD_{\rm tape}N\,H(\pi^{\mathcal A})}.
\]
\end{lemma}

\begin{proof}
Let $b_0=\min\{T,L_\Sigma\}$. Every episode reward lies in
$[0,b_0]$.
The complete reported trajectory determines that reward.
Given $P$, the arms' tape histograms are independent with laws
$\mathbf q_{L_i}(P_i)$, and their orders given those histograms are
uniform. Thus every fixed episode policy has mean reward
\[
 V^P(\pi)=\langle f_\pi,\mathbf q(P)\rangle,\qquad
 \mathbf q(P)=\bigotimes_i\mathbf q_{L_i}(P_i)\in\mathbb R^{D_{\rm tape}},
\]
where $f_\pi$ averages its reward over the uniform tape orders.

Fix a thread's history $h$ and write
$p_\alpha=\Pr\{\pi^{\mathcal A}=\alpha\mid h\}$ for $\alpha\in\Pi$.
The selected policy $\pi_e$ has these probabilities and is
conditionally independent of $P$ and $U_0$. Posterior matching gives
\[
 \Delta_h:=\E[V^P(\pi^{\mathcal A})-V^P(\pi_e)\mid h]
 =\E[V^{\widetilde P}(\pi_e)-V^P(\pi_e)\mid h].
\]
On policies of positive conditional probability, form the matrix
\[
 M_{\alpha\beta}=\sqrt{p_\alpha p_\beta}
 \left(\E[V^P(\alpha)\mid\pi^{\mathcal A}=\beta,h]
             -\E[V^P(\alpha)\mid h]\right).
\]
Then $\Delta_h=\operatorname{tr}M$. The linear representation gives
$\operatorname{rank}M\le D_{\rm tape}$. If $Y_e$ denotes the episode's complete
reports, Pinsker's inequality and the reward range give
\[
 I(\pi^{\mathcal A};\pi_e,Y_e\mid h)
 \ge\frac{2}{b_0^2}\|M\|_F^2,\qquad
 \Delta_h^2\le D_{\rm tape}\|M\|_F^2
 \le\frac{b_0^2D_{\rm tape}}{2}I(\pi^{\mathcal A};\pi_e,Y_e\mid h).
\]
This is the matrix argument of
\citet[Proposition~5]{russo2016info}, valid for an arbitrary reference
target \citep[Section~6]{russo2022satisficing}; $\Delta_h$ may be signed.

Split episodes by their residue modulo $g$. Each thread uses only
its own earlier reports, all of which have arrived before its next
episode. Conditional on $P$, fresh outcome tapes are independent
across episodes. The information chain rule bounds its total expected information gain
about $\pi^{\mathcal A}$ by $H(\pi^{\mathcal A})$. Cauchy--Schwarz
therefore bounds a thread of $n_r$ episodes by
$b_0\sqrt{D_{\rm tape}n_rH(\pi^{\mathcal A})/2}$. Summing and using
$\sum_r\sqrt{n_r}\le\sqrt{gN}$ proves the result.
\end{proof}

\subsection{Posterior sampling with episode LPs}

Posterior sampling gives a second comparison without prescribed
probing episodes. Sample the recurring priors, solve their episode
LP, and execute it with the known-prior guarantee. Fix the planner,
its tie-breaking and its executor as a map $\mathcal A(P,U)$, where
$U$ supplies all policy coins. Fixing $U$ gives a deterministic
episode policy, feasible under every recurring prior. Its expected
reward guarantee is required only under the supplied prior $P$.

\begin{theorem}[Posterior sampling with episode LPs]\label{thm:ps-episode-lp}
Assume the stationary conditions of
Corollary~\ref{cor:episode-lp-rate}. Let $V^P(\pi)$ be the expected
episode reward of policy $\pi$ under $P$, and suppose
\[
 \E_U V^P(\mathcal A(P,U))\ge\kappa\LP^{(k)}(P)
 \qquad\text{for every }P.
\]
Use $g$ inference threads, each updating only from its own episodes.
In each episode, draw $\widetilde P$ from the thread's posterior and
execute $\mathcal A(\widetilde P,U_e)$ with fresh coins.
Put
\[
 D_{\rm tape}=\prod_i\binom{L_i+d_i-1}{d_i-1},\qquad
 \pi^{\mathcal A}=\mathcal A(P,U_0),
\]
where $U_0$ is an independent reference seed, unused by the learner.
Then this policy's expected total reward $V_{\rm PS\text{-}LP}$ satisfies
\begin{equation}\label{eq:ps-episode-lp}
 \kappa N\,\E_\Q\LP^{(k)}(P)-V_{\rm PS\text{-}LP}
 \le\min\{T,L_\Sigma\}
       \sqrt{\tfrac12\,gD_{\rm tape}N\,H(\pi^{\mathcal A})}.
\end{equation}
The entropy is under $\Q$ and $U_0$, with natural logarithms.
It is at most $\log|\Pi|$, where $\Pi$ is the finite set of
deterministic feasible episode policies.
\end{theorem}

\begin{proof}
The selected policy has the conditional distribution of
$\pi^{\mathcal A}$ given the thread's history. Apply
Lemma~\ref{lem:policy-sampling} to this fixed reference map, then use
its known-prior guarantee. The lemma uses whole-tape histograms to
represent every episode policy's expected reward linearly; it does
not require the reference policy to be optimal.
\end{proof}

Under the completion conditions of Theorem~\ref{thm:meta} and
nonbinding comparison budgets $B_i'\ge NL_i$, the posterior-sampled
comparison gives
\[
 (1-\omega)\kappa N\,\E_\Q\LP^{(k)}(P)-V_A
 \le(1-\omega)\bar R_{\rm PS}+1,
\]
where $\bar R_{\rm PS}$ is the right side of
\eqref{eq:ps-episode-lp}. Indeed, the learning LP at budgets $B'$
upper-bounds this admissible comparison, and
Theorem~\ref{thm:meta} transfers its value. This argument uses full completion.

\begin{proof}[Proof of Proposition~\ref{prop:ps-price}]
Apply Lemma~\ref{lem:policy-sampling} with
$\mathcal A(P,U)=\pi^*(P)$. Nonbinding budgets give
$V_{\rm prior}=N\,\E_\Q V^P(\pi^*(P))$, and $V_{\rm learn}$ is at
least the value of this admissible posterior-sampling policy.
Finally $H(\pi^*)\le\log|\Pi|$.
\end{proof}

The information-price bounds also allow the specified safe
alternative of Section~\ref{sec:weights}. Fill otherwise unused
steps with that action. Episode rewards then lie in
$[a(T-b_0),aT+(1-a)b_0]$, still an interval of length $b_0$,
and remain linear in the whole-tape law. The same range and
information arguments apply to these policies.

The exact reference planner bounds the information price; the LP
planner and executor in Theorem~\ref{thm:ps-episode-lp} give a scaled
LP comparison by the same argument. Both coefficients include the
product of whole-tape histogram dimensions. Even for one unknown
arm, this dimension, the report lag and the reference-policy entropy
remain. If $\Q$ has finite support, then
$H(\pi^*)\le\log|\operatorname{supp}\Q|$.

\begin{samepage}
\subsubsection*{AI Use Statement}
We used multiple generative AI tools for
improving writing and searching related work. We have reviewed all AI-assisted work.

\end{samepage}

\bibliographystyle{plainnat}
\bibliography{metabandits_refs}

\begin{thebibliography}{44}
\providecommand{\natexlab}[1]{#1}
\providecommand{\url}[1]{\texttt{#1}}
\expandafter\ifx\csname urlstyle\endcsname\relax
  \providecommand{\doi}[1]{doi: #1}\else
  \providecommand{\doi}{doi: \begingroup \urlstyle{rm}\Url}\fi

\bibitem[Adelman and Mersereau(2008)]{adelman2008}
Daniel Adelman and Adam~J. Mersereau.
\newblock Relaxations of weakly coupled stochastic dynamic programs.
\newblock \emph{Operations Research}, 56\penalty0 (3):\penalty0 712--727, 2008.
\newblock \doi{10.1287/opre.1070.0445}.

\bibitem[Agarwal and Duchi(2011)]{agarwal2011distributed}
Alekh Agarwal and John~C. Duchi.
\newblock Distributed delayed stochastic optimization.
\newblock In \emph{Advances in Neural Information Processing Systems 24 (NIPS
  2011)}, volume~24, pages 873--881. Curran Associates, Inc., 2011.

\bibitem[Anderson(1964)]{anderson1964}
T.~W. Anderson.
\newblock Sequential analysis with delayed observations.
\newblock \emph{Journal of the American Statistical Association}, 59\penalty0
  (308):\penalty0 1006--1015, 1964.
\newblock \doi{10.1080/01621459.1964.10480746}.

\bibitem[Austin and Panchenko(2014)]{austin2014hierarchical}
Tim Austin and Dmitry Panchenko.
\newblock A hierarchical version of the de {Finetti} and {Aldous--Hoover}
  representations.
\newblock \emph{Probability Theory and Related Fields}, 159\penalty0
  (3):\penalty0 809--823, 2014.
\newblock \doi{10.1007/s00440-013-0521-0}.

\bibitem[Azizi et~al.(2023)Azizi, Kveton, Ghavamzadeh, and
  Katariya]{azizi2023meta}
Mohammadjavad Azizi, Branislav Kveton, Mohammad Ghavamzadeh, and Sumeet
  Katariya.
\newblock Meta-learning for simple regret minimization.
\newblock In \emph{Proceedings of the 37th AAAI Conference on Artificial
  Intelligence (AAAI)}, volume~37, pages 6709--6717, 2023.
\newblock \doi{10.1609/aaai.v37i6.25823}.

\bibitem[Bastani et~al.(2022)Bastani, Simchi-Levi, and Zhu]{bastani2022meta}
Hamsa Bastani, David Simchi-Levi, and Ruihao Zhu.
\newblock Meta dynamic pricing: Transfer learning across experiments.
\newblock \emph{Management Science}, 68\penalty0 (3):\penalty0 1865--1881,
  2022.
\newblock \doi{10.1287/mnsc.2021.4071}.

\bibitem[Basu et~al.(2021)Basu, Kveton, Zaheer, and Szepesv{\'a}ri]{basu2021}
Soumya Basu, Branislav Kveton, Manzil Zaheer, and Csaba Szepesv{\'a}ri.
\newblock No regrets for learning the prior in bandits.
\newblock In \emph{Advances in Neural Information Processing Systems 34
  (NeurIPS)}, 2021.

\bibitem[Baudry et~al.(2021)Baudry, Gautron, Kaufmann, and
  Maillard]{baudry2021cvarts}
Dorian Baudry, Romain Gautron, Emilie Kaufmann, and Odalric Maillard.
\newblock Optimal {Thompson} sampling strategies for support-aware {CVaR}
  bandits.
\newblock In \emph{Proceedings of the 38th International Conference on Machine
  Learning}, volume 139 of \emph{Proceedings of Machine Learning Research},
  pages 716--726. PMLR, 2021.
\newblock URL \url{https://proceedings.mlr.press/v139/baudry21a.html}.

\bibitem[Caro and Yoo(2010)]{caro2010}
Felipe Caro and Onesun~Steve Yoo.
\newblock Indexability of bandit problems with response delays.
\newblock \emph{Probability in the Engineering and Informational Sciences},
  24\penalty0 (3):\penalty0 349--374, 2010.
\newblock \doi{10.1017/S0269964810000021}.

\bibitem[Chakraborty et~al.(2010)Chakraborty, Even-Dar, Guha, Mansour, and
  Muthukrishnan]{chakraborty2010pricing}
Tanmoy Chakraborty, Eyal Even-Dar, Sudipto Guha, Yishay Mansour, and
  S.~Muthukrishnan.
\newblock Approximation schemes for sequential posted pricing in multi-unit
  auctions.
\newblock In \emph{Internet and Network Economics (WINE 2010)}, volume 6484 of
  \emph{Lecture Notes in Computer Science}, pages 158--169. Springer, 2010.
\newblock \doi{10.1007/978-3-642-17572-5_13}.
\newblock URL \url{https://arxiv.org/abs/1008.1616}.

\bibitem[Chen et~al.(2013)Chen, Wang, and Yuan]{chen2013cmab}
Wei Chen, Yajun Wang, and Yang Yuan.
\newblock Combinatorial multi-armed bandit: General framework and applications.
\newblock In \emph{Proceedings of the 30th International Conference on Machine
  Learning (ICML)}, volume~28 of \emph{Proceedings of Machine Learning
  Research}, pages 151--159, 2013.

\bibitem[Eick(1988)]{eick1988}
Stephen~G. Eick.
\newblock The two-armed bandit with delayed responses.
\newblock \emph{The Annals of Statistics}, 16\penalty0 (1):\penalty0 254--264,
  1988.

\bibitem[Ellis et~al.(2021)Ellis, Wong, Nye, Sabl{\'e}-Meyer, Morales, Hewitt,
  Cary, Solar-Lezama, and Tenenbaum]{ellis2021dreamcoder}
Kevin Ellis, Catherine Wong, Maxwell Nye, Mathias Sabl{\'e}-Meyer, Lucas
  Morales, Luke Hewitt, Luc Cary, Armando Solar-Lezama, and Joshua~B.
  Tenenbaum.
\newblock {DreamCoder}: Bootstrapping inductive program synthesis with
  wake-sleep library learning.
\newblock In \emph{Proceedings of the 42nd ACM SIGPLAN International Conference
  on Programming Language Design and Implementation}, PLDI '21, pages 835--850.
  Association for Computing Machinery, 2021.
\newblock \doi{10.1145/3453483.3454080}.

\bibitem[Gittins and Jones(1974)]{gittins1974dynamic}
John~C. Gittins and David~M. Jones.
\newblock A dynamic allocation index for the sequential design of experiments.
\newblock In \emph{Progress in Statistics}, pages 241--266. North-Holland,
  1974.

\bibitem[Guha(2026)]{guha2026future}
Sudipto Guha.
\newblock Budgeted {Bayesian} bandits for future utilization.
\newblock Preprint, 2026.
\newblock arXiv:2610.24156.

\bibitem[Guha and Munagala(2007)]{GM07}
Sudipto Guha and Kamesh Munagala.
\newblock Approximation algorithms for budgeted learning problems.
\newblock In \emph{Proceedings of the 39th Annual ACM Symposium on Theory of
  Computing (STOC)}, pages 104--113, 2007.
\newblock \doi{10.1145/1250790.1250807}.

\bibitem[Guha and Munagala(2013)]{GM13}
Sudipto Guha and Kamesh Munagala.
\newblock Approximation algorithms for {Bayesian} multi-armed bandit problems.
\newblock arXiv:1306.3525, 2013.
\newblock URL \url{https://arxiv.org/abs/1306.3525}.

\bibitem[Guha et~al.(2010)Guha, Munagala, and Shi]{GMS10}
Sudipto Guha, Kamesh Munagala, and Peng Shi.
\newblock Approximation algorithms for restless bandit problems.
\newblock \emph{Journal of the ACM}, 58\penalty0 (1), 2010.
\newblock \doi{10.1145/1870103.1870106}.
\newblock Article 3. Preliminary version in SODA 2009; arXiv:0711.3861.

\bibitem[Hong et~al.(2022)Hong, Kveton, Zaheer, and Ghavamzadeh]{hong2022}
Joey Hong, Branislav Kveton, Manzil Zaheer, and Mohammad Ghavamzadeh.
\newblock Hierarchical {Bayesian} bandits.
\newblock In \emph{Proceedings of the 25th International Conference on
  Artificial Intelligence and Statistics (AISTATS)}, volume 151 of
  \emph{Proceedings of Machine Learning Research}, 2022.

\bibitem[Jiang et~al.(2023)Jiang, Jiang, Li, Lin, Wang, and
  Zhou]{jiang2023restless}
Bowen Jiang, Bo~Jiang, Jian Li, Tao Lin, Xinbing Wang, and Chenghu Zhou.
\newblock Online restless bandits with unobserved states.
\newblock In \emph{Proceedings of the 40th International Conference on Machine
  Learning (ICML)}, volume 202 of \emph{Proceedings of Machine Learning
  Research}, pages 15041--15066, 2023.

\bibitem[Joulani et~al.(2013)Joulani, Gy{\"o}rgy, and
  Szepesv{\'a}ri]{joulani2013}
Pooria Joulani, Andr{\'a}s Gy{\"o}rgy, and Csaba Szepesv{\'a}ri.
\newblock Online learning under delayed feedback.
\newblock In \emph{Proceedings of the 30th International Conference on Machine
  Learning (ICML)}, 2013.

\bibitem[Jung and Tewari(2019)]{jung2019}
Young~Hun Jung and Ambuj Tewari.
\newblock Regret bounds for {Thompson} sampling in episodic restless bandit
  problems.
\newblock In \emph{Advances in Neural Information Processing Systems 32
  (NeurIPS)}, 2019.

\bibitem[Kalkanli and {\"O}zg{\"u}r(2021)]{kalkanli2021}
Cem Kalkanli and Ayfer {\"O}zg{\"u}r.
\newblock Batched {Thompson} sampling.
\newblock In \emph{Advances in Neural Information Processing Systems 34
  (NeurIPS)}, 2021.

\bibitem[Klenske and Hennig(2016)]{klenske2016dual}
Edgar~D. Klenske and Philipp Hennig.
\newblock Dual control for approximate {Bayesian} reinforcement learning.
\newblock \emph{Journal of Machine Learning Research}, 17\penalty0
  (127):\penalty0 1--30, 2016.
\newblock URL \url{https://www.jmlr.org/papers/v17/15-162.html}.

\bibitem[Kveton et~al.(2021)Kveton, Konobeev, Zaheer, Hsu, Mladenov, Boutilier,
  and Szepesv{\'a}ri]{kveton2021meta}
Branislav Kveton, Mikhail Konobeev, Manzil Zaheer, Chih-Wei Hsu, Martin
  Mladenov, Craig Boutilier, and Csaba Szepesv{\'a}ri.
\newblock Meta-{Thompson} sampling.
\newblock In \emph{Proceedings of the 38th International Conference on Machine
  Learning (ICML)}, volume 139 of \emph{Proceedings of Machine Learning
  Research}, pages 5884--5893, 2021.

\bibitem[Liu and Zhao(2010)]{liu2010indexability}
Keqin Liu and Qing Zhao.
\newblock Indexability of restless bandit problems and optimality of {Whittle}
  index for dynamic multichannel access.
\newblock \emph{IEEE Transactions on Information Theory}, 56\penalty0
  (11):\penalty0 5547--5567, 2010.
\newblock \doi{10.1109/TIT.2010.2068950}.
\newblock URL \url{https://arxiv.org/abs/0810.4658}.

\bibitem[Mannor and Tsitsiklis(2011)]{mannor2011meanvariance}
Shie Mannor and John~N. Tsitsiklis.
\newblock Mean-variance optimization in {Markov} decision processes.
\newblock In \emph{Proceedings of the 28th International Conference on Machine
  Learning}, 2011.
\newblock URL \url{https://icml.cc/Conferences/2011/papers/156_icmlpaper.pdf}.

\bibitem[McDonald et~al.(2023)McDonald, Maystre, Lalmas, Russo, and
  Ciosek]{mcdonald2023}
Thomas~M. McDonald, Lucas Maystre, Mounia Lalmas, Daniel Russo, and Kamil
  Ciosek.
\newblock Impatient bandits: Optimizing recommendations for the long-term
  without delay.
\newblock In \emph{Proceedings of the 29th ACM SIGKDD Conference on Knowledge
  Discovery and Data Mining (KDD)}, 2023.
\newblock \doi{10.1145/3580305.3599386}.

\bibitem[Mizuno et~al.(1993)Mizuno, Todd, and Ye]{mizuno1993adaptive}
Shinji Mizuno, Michael~J. Todd, and Yinyu Ye.
\newblock On adaptive-step primal-dual interior-point algorithms for linear
  programming.
\newblock \emph{Mathematics of Operations Research}, 18\penalty0 (4):\penalty0
  964--981, 1993.
\newblock \doi{10.1287/moor.18.4.964}.

\bibitem[Papadimitriou and Tsitsiklis(1999)]{papadimitriou1999}
Christos~H. Papadimitriou and John~N. Tsitsiklis.
\newblock The complexity of optimal queuing network control.
\newblock \emph{Mathematics of Operations Research}, 24\penalty0 (2):\penalty0
  293--305, 1999.
\newblock \doi{10.1287/moor.24.2.293}.

\bibitem[Qiu et~al.(2024)Qiu, Mao, Patke, Cui, Wang, Franke, Kalbarczyk,
  Ba{\c{s}}ar, and Iyer]{qiu2024flash}
Haoran Qiu, Weichao Mao, Archit Patke, Shengkun Cui, Chen Wang, Hubertus
  Franke, Zbigniew~T. Kalbarczyk, Tamer Ba{\c{s}}ar, and Ravishankar~K. Iyer.
\newblock {FLASH}: Fast model adaptation in {ML}-centric cloud platforms.
\newblock In \emph{Proceedings of Machine Learning and Systems}, volume~6,
  2024.

\bibitem[Radner and Stiglitz(1984)]{radner1984}
Roy Radner and Joseph~E. Stiglitz.
\newblock A nonconcavity in the value of information.
\newblock In Marcel Boyer and Richard~E. Kihlstrom, editors, \emph{Bayesian
  Models in Economic Theory}, pages 33--52. North-Holland, 1984.

\bibitem[Rockafellar and Uryasev(2000)]{rockafellar2000cvar}
R.~Tyrrell Rockafellar and Stanislav Uryasev.
\newblock Optimization of conditional value-at-risk.
\newblock \emph{Journal of Risk}, 2\penalty0 (3):\penalty0 21--41, 2000.
\newblock \doi{10.21314/JOR.2000.038}.

\bibitem[Russo and Van~Roy(2016)]{russo2016info}
Daniel Russo and Benjamin Van~Roy.
\newblock An information-theoretic analysis of {Thompson} sampling.
\newblock \emph{Journal of Machine Learning Research}, 17\penalty0
  (68):\penalty0 1--30, 2016.

\bibitem[Russo and Van~Roy(2022)]{russo2022satisficing}
Daniel Russo and Benjamin Van~Roy.
\newblock Satisficing in time-sensitive bandit learning.
\newblock \emph{Mathematics of Operations Research}, 47\penalty0 (4), 2022.

\bibitem[Shafer et~al.(2011)Shafer, Shen, Vereshchagin, and
  Vovk]{shafer2011test}
Glenn Shafer, Alexander Shen, Nikolai Vereshchagin, and Vladimir Vovk.
\newblock Test martingales, {Bayes} factors and $p$-values.
\newblock \emph{Statistical Science}, 26\penalty0 (1):\penalty0 84--101, 2011.
\newblock \doi{10.1214/10-STS347}.

\bibitem[Shaked and Shanthikumar(2007)]{shaked2007}
Moshe Shaked and J.~George Shanthikumar.
\newblock \emph{Stochastic Orders}.
\newblock Springer, 2007.

\bibitem[Simchowitz et~al.(2021)Simchowitz, Tosh, Krishnamurthy, Hsu, Lykouris,
  Dud{\'\i}k, and Schapire]{simchowitz2021}
Max Simchowitz, Christopher Tosh, Akshay Krishnamurthy, Daniel Hsu, Thodoris
  Lykouris, Miroslav Dud{\'\i}k, and Robert~E. Schapire.
\newblock {Bayesian} decision-making under misspecified priors with
  applications to meta-learning.
\newblock In \emph{Advances in Neural Information Processing Systems 34
  (NeurIPS)}, 2021.

\bibitem[Smallwood and Sondik(1973)]{smallwood1973finite}
Richard~D. Smallwood and Edward~J. Sondik.
\newblock The optimal control of partially observable markov processes over a
  finite horizon.
\newblock \emph{Operations Research}, 21\penalty0 (5):\penalty0 1071--1088,
  1973.
\newblock \doi{10.1287/opre.21.5.1071}.

\bibitem[Whittle(1988)]{whittle1988}
Peter Whittle.
\newblock Restless bandits: Activity allocation in a changing world.
\newblock \emph{Journal of Applied Probability}, 25\penalty0 (A):\penalty0
  287--298, 1988.

\bibitem[Zeng et~al.(2026)Zeng, Tang, and Ding]{zeng2026metarouter}
Jiahao Zeng, Ming Tang, and Ningning Ding.
\newblock Learning to route llms from implicit cost-performance preferences via
  meta-learning.
\newblock In \emph{Proceedings of the Thirty-Fifth International Joint
  Conference on Artificial Intelligence}, pages 6045--6053, 2026.
\newblock \doi{10.24963/ijcai.2026/673}.

\bibitem[Zhang et~al.(2026)Zhang, Baldwin-McDonald, Ciosek, Maystre, and
  Russo]{zhang2025}
Kelly~W. Zhang, Thomas Baldwin-McDonald, Kamil Ciosek, Lucas Maystre, and
  Daniel Russo.
\newblock Impatient bandits: Optimizing for the long-term without delay.
\newblock \emph{Journal of Machine Learning Research}, 27\penalty0
  (156):\penalty0 1--58, 2026.

\bibitem[Zhang et~al.(2021)Zhang, Wu, Chang, Jin, Tan, Li, Zhang, and
  Cui]{zhang2021restune}
Xinyi Zhang, Hong Wu, Zhuo Chang, Shuowei Jin, Jian Tan, Feifei Li, Tieying
  Zhang, and Bin Cui.
\newblock {ResTune}: Resource oriented tuning boosted by meta-learning for
  cloud databases.
\newblock In \emph{Proceedings of the 2021 International Conference on
  Management of Data}, SIGMOD '21, pages 2102--2114. Association for Computing
  Machinery, 2021.
\newblock \doi{10.1145/3448016.3457291}.

\bibitem[Zhu and Tan(2020)]{zhu2020mvts}
Qiuyu Zhu and Vincent Tan.
\newblock {Thompson} sampling algorithms for mean-variance bandits.
\newblock In \emph{Proceedings of the 37th International Conference on Machine
  Learning}, volume 119 of \emph{Proceedings of Machine Learning Research},
  pages 11599--11608. PMLR, 2020.
\newblock URL \url{https://proceedings.mlr.press/v119/zhu20d.html}.

\end{thebibliography}

\end{document}